\documentclass{article}

\PassOptionsToPackage{numbers, square, sort}{natbib}

\usepackage{microtype}
\usepackage{graphicx}
\usepackage{subfigure}
\usepackage{booktabs} 

\usepackage[dvipsnames]{xcolor}
\usepackage[citecolor=NavyBlue, linkcolor=NavyBlue, colorlinks=True]{hyperref} 

\usepackage[preprint]{neurips_2024}

\usepackage{amsmath}
\usepackage{amssymb}
\usepackage{mathtools}
\usepackage{amsthm}
\usepackage{mathmacros}

\newcommand{\y}{\m{y}}
\newcommand{\z}{\m{z}}

\newcommand{\Q}{\m{Q}}
\newcommand{\J}{\m{J}}

\newcommand{\al}{\alpha}
\newcommand{\bt}{\beta}
\newcommand{\dl}{\delta}
\newcommand{\gm}{\gamma}
\renewcommand{\th}{\sm{\theta}}
\newcommand{\Lo}{\mathcal{L}}
\renewcommand{\P}{P}
\newcommand{\knorm}{\mathcal{K}}
\newcommand{\D}{D}
\newcommand{\lr}{\eta}

\renewcommand{\H}{\m{H}}
\newcommand{\g}{\m{g}}

\renewcommand{\v}{\m{v}}
\renewcommand{\u}{\m{u}}
\newcommand{\V}{\m{V}}

\newcommand{\ntk}{\sm{\hat{\Theta}}}

\newcommand{\B}{B}

\newcommand{\Id}{\m{I}}

\newcommand{\tl}{\tilde}

\newcommand{\w}{\m{w}}

\newcommand{\eps}{\epsilon}

\newcommand{\lam}{\lambda}

\newcommand{\diag}{{\rm diag}}

\newcommand{\bfr}{\beta}
\newcommand{\pmat}{\m{P}}

\newcommand{\rad}{\rho}

\newcommand{\pvec}{\m{p}}
\newcommand{\tpvec}{\tilde{\pvec}}
\newcommand{\lmat}{\sm{\Lambda}}
\newcommand{\dmat}{\m{D}}

\renewcommand{\S}{\m{S}}
\newcommand{\cmat}{\m{C}}
\newcommand{\bmat}{\m{B}}
\newcommand{\amat}{\m{A}}

\newcommand{\Svec}{\m{S}}
\newcommand{\mom}{\mu}
\newcommand{\momd}{\alpha}
\newcommand{\Om}{\Omega}
\newcommand{\f}{\m{f}}
\newcommand{\shat}{\hat{\sigma}}
\newcommand{\lhat}{\hat{\lambda}}
\renewcommand{\tpose}{\top}
\newcommand{\tlr}{\tl{\lr}}
\newcommand{\linop}{\m{T}}
\newcommand{\C}{C}
\newcommand{\modcov}{\tl{\sm{\Sigma}}}
\newcommand{\modntk}{\tl{\sm{\Theta}}}

\usepackage[capitalize,noabbrev]{cleveref}

\theoremstyle{plain}
\newtheorem{theorem}{Theorem}[section]

\newtheorem{lemma}[theorem]{Lemma}

\theoremstyle{definition}

\theoremstyle{remark}

\usepackage[textsize=tiny]{todonotes}

\newif\ifcomments
\commentsfalse 

\ifcomments
\newcommand{\jp}[1]{{\color{blue}[JP: #1]}}
\newcommand{\aga}[1]{{\color{red}[AA: #1]}}
\else
\newcommand{\jp}[1]{}
\newcommand{\aga}[1]{}
\fi

\title{High dimensional analysis reveals conservative sharpening and a stochastic edge of stability}

\author{%
  Atish Agarwala \\
  Google DeepMind\\
  \texttt{thetish@google.com} \\
   \And
   Jeffrey Pennington \\
  Google DeepMind \\
  \texttt{jpennin@google.com} \\
  }

\begin{document}

\maketitle

\begin{abstract}
Recent empirical and theoretical work has shown that the dynamics of the large eigenvalues of the training loss Hessian have some remarkably robust features across models and datasets in the full batch regime. There is often an early period of \emph{progressive sharpening} where the large eigenvalues increase, followed by stabilization at a predictable value known as the \emph{edge of stability}. Previous work showed that in the stochastic setting, the eigenvalues increase more slowly - a phenomenon we call \emph{conservative sharpening}. We provide a theoretical analysis of a simple high-dimensional model which shows the origin of this slowdown. We also show that there is an alternative \emph{stochastic edge of stability} which arises at small batch size that is sensitive to the trace of the Neural Tangent Kernel rather than the large Hessian eigenvalues. We conduct an experimental study which highlights the qualitative differences from the full batch phenomenology, and suggests that controlling the stochastic edge of stability can help optimization.
\end{abstract}

\section{Introduction \aga{comments on}}

Despite rapid advances in the capabilities of machine learning systems,
a large open question about training remains: what makes stochastic gradient descent work in deep learning? Much recent work has focused on understanding learning dynamics through the lens of the loss landscape geometry. The Hessian of the training loss with respect to the parameters
changes significantly over training, and its statistics are intimately linked to optimization
choices \citep{ghorbani_investigation_2019, gilmer_loss_2022}.

In the full batch setting, is a robust observation about the eigenvalues of the
loss Hessian: the large eigenvalues tend to increase at early times (\emph{progressive sharpening}), until the maximum eigenvalue $\lam_{max}$ stabilizes at the \emph{edge of stability} (EOS) - the maximum value consistent with convergence in the convex setting
\citep{cohen_gradient_2022, cohen_adaptive_2022}. This phenomenology can be explained
via positive alignment and negative feedback between $\lam_{max}$ and the parameter changes in the largest eigendirection of the Hessian \citep{damian_selfstabilization_2022, agarwala_secondorder_2022}.

The phenomenology is more complicated in the minibatch setting (SGD). For one, progressive
sharpening decreases in strength as batch size decreases \citep{jastrzebski_three_2018, cohen_gradient_2022} - a phenomenon which we dub
\emph{conservative sharpening}. In addition,
there is theoretical and experimental evidence that the stochastic nature of the gradients suggests that quantities like the \emph{trace} of the Hessian, are important for long-time convergence and stability \citep{jastrzebski_breakeven_2020, wu_alignment_2022}. This observation has lead to attempts to define a \emph{stochastic edge of stability} (S-EOS) to
understand loss landscape dynamics 
in the SGD setting \citep{wu_implicit_2023a, mulayoff_exact_2023}.

In parallel, there has been progress in understanding aspects of SGD in simple but
high-dimensional models. The theory of infinitely-wide neural networks has shown that
in the appropriate limit, model training resembles gradient-based training of kernel
methods \citep{jacot_neural_2018, lee_wide_2019, adlam_neural_2020}. More recent work has
studied the dynamics of SGD in convex models where the number of datapoints and the number of parameters scale to infinity at the same rate \citep{mei_meanfield_2019, paquette_sgd_2021, paquette_homogenization_2022, paquette_implicit_2022, benarous_highdimensional_2022, arnaboldi_highdimensional_2023}.
These theoretical works have found tight stability/convergence conditions in this
high-dimensional regime -- a regime that is increasingly important in the current landscape of increasing model and dataset sizes.

In this work, we present evidence that a stochastic instability phenomenon
is useful
for understanding neural network training dynamics. We use theoretical analysis
to show the following:
\vspace{-2mm}
\begin{itemize}
\setlength\itemsep{0.25mm}
    \item There is a \emph{stochastic edge of stability} (S-EOS) which in the
    MSE setting is controlled by a scalar $\knorm$ which we call the \emph{noise kernel norm}.
    \item Conservative sharpening depends on the statistics of both the Jacobian and
    its gradient, and provides stronger suppression on larger eigenvalues.
\end{itemize}
\vspace{-2mm}
The theory suggests that S-EOS effects can become important in practical regimes.
We then demonstrate the following experimentally:
\vspace{-2mm}
\begin{itemize}
\setlength\itemsep{0.25mm}
    \item $\knorm$ self-stabilizes near the critical value $1$,
    giving us an S-EOS stabilization which is qualitatively distinct from stabilization of
    $\lam_{max}$ in the original EOS.
    \item For small batch size the behavior of $\knorm$ is a slowly varying
    function of $\lr/\B$.
    \item $\knorm$ is predictive of training outcomes across a variety of model sizes,
    and with additional effects like momentum and learning rate schedules.
\end{itemize}
\vspace{-2mm}
We conclude with a discussion of the utility of $\knorm$ in understanding SGD dynamics more generally.

\begin{figure}[t]
\centering
\begin{tabular}{ccc}
\includegraphics[width=0.32\linewidth]{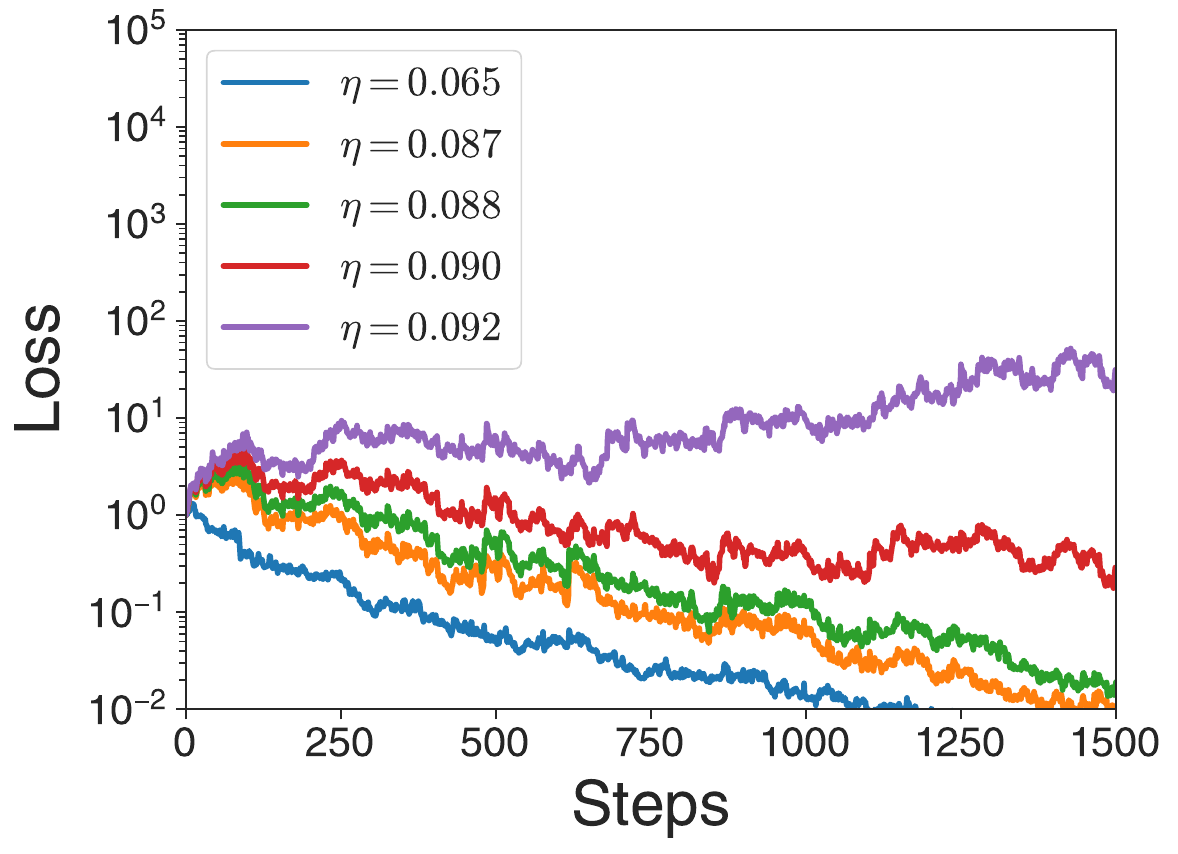}
\includegraphics[width=0.32\linewidth]{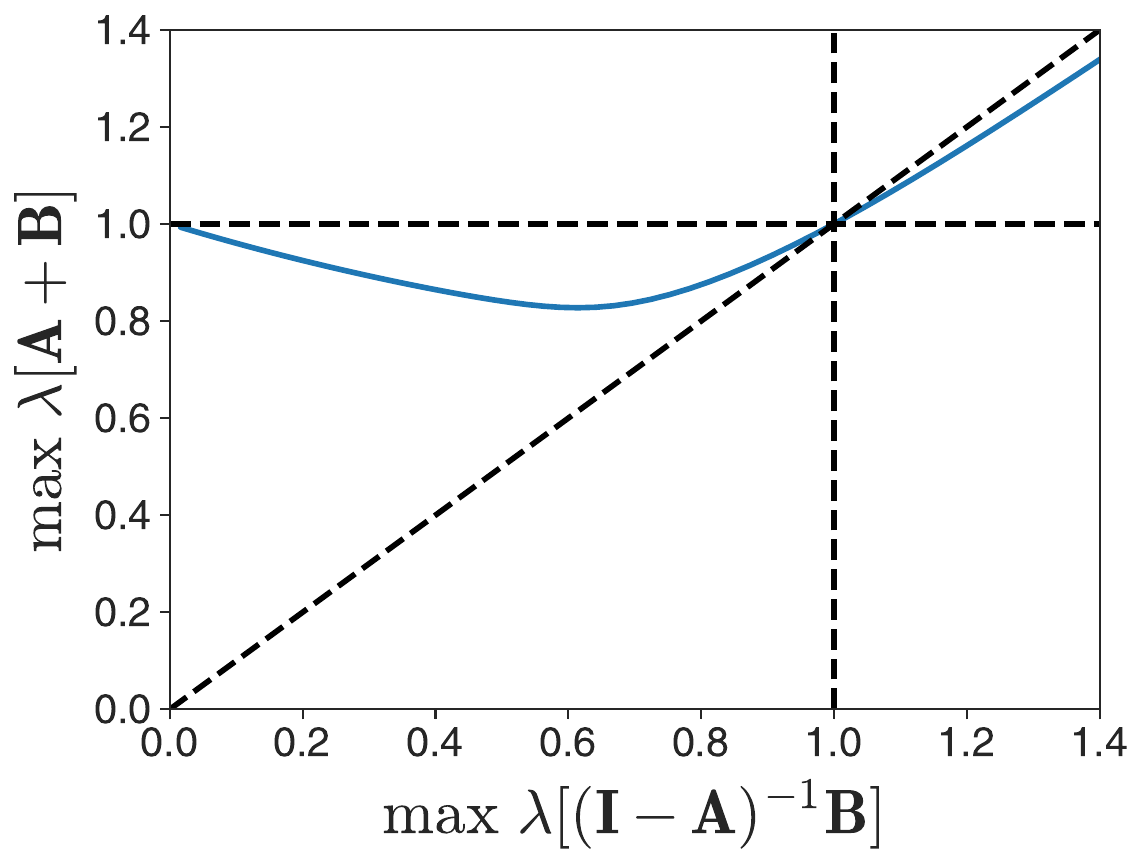} &
\includegraphics[width=0.32\linewidth]{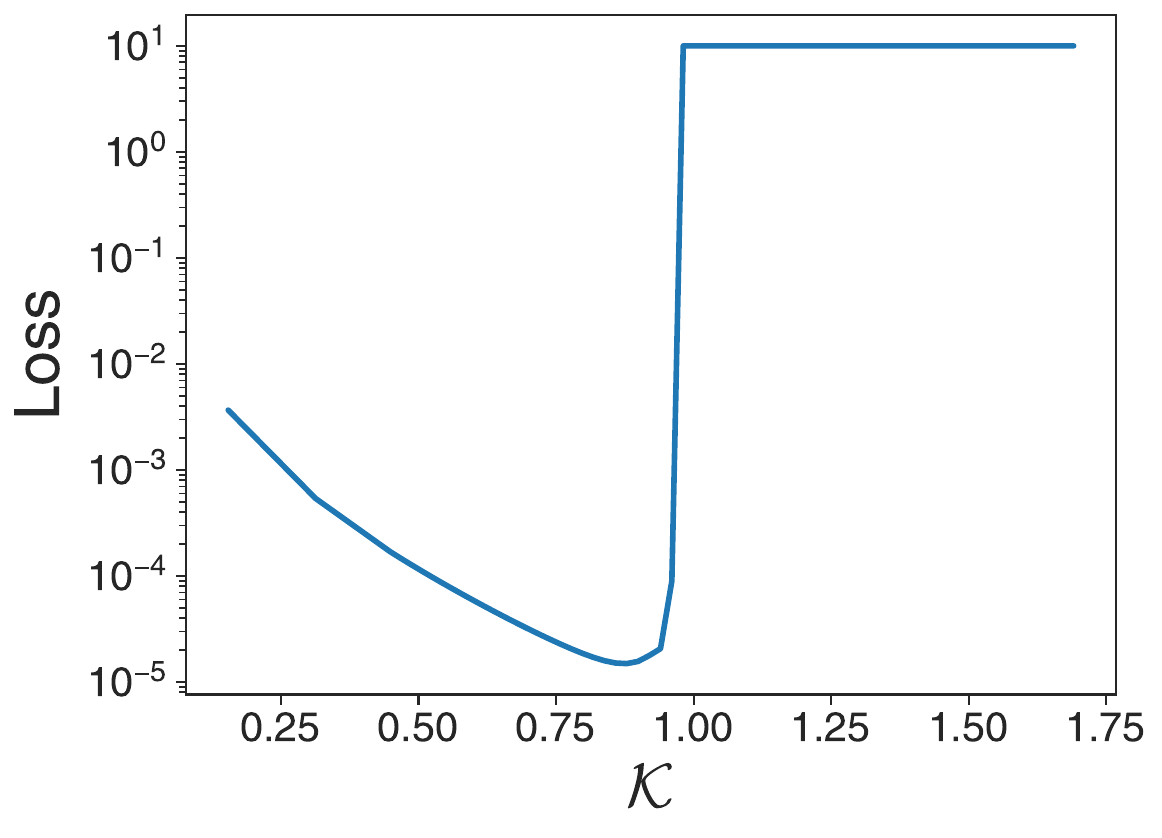}
\end{tabular}
\caption{SGD trajectories for linear regression show divergence due to stochastic effects as
$\lr$ is increased (left, $\B = 5$, $\D = 100$, $\P = 120$, i.i.d. Gaussian $\J$). $\knorm$ interpolates
from $0$ at small learning rate, to value $1$ precisely when $\lam_{max}[\amat+\bmat] = 1$
(middle). Loss after $10^{4}$ steps diverges for $\knorm >1$ (right, plot saturated $10^{1}$ 
for convenience).}
\label{fig:knorm_stability_basics}
\end{figure}

\section{The stochastic edge of stability}

\label{sec:s_eos_theory}

In the deterministic setting, the edge of stability (EOS) is derived by performing
stability analysis of the loss under full batch (GD) dynamics
about a minimum on a convex model. In this section, we derive a stability condition for
SGD in an analogous fashion. In the stochastic setting, we will focus on the long-time
behavior of the \emph{second moments} of the network outputs - where the averages are taken
over the sampling of the minibatches.
A local, weight space analysis of the second moment was studied previously in
\citep{ma_linear_2021, mulayoff_exact_2023}.

Instead, we will use a function space analysis to define a \emph{noise kernel norm} $\knorm$ which characterizes the global stability of
of the residuals $\z_{t}$ under SGD noise. The resulting measure will
range from $0$ in the full batch SGD case to $1$ at the stability threshhold - analogous to
the role the normalized eigenvalue $\lr\lam_{max}$ plays in the full batch case.
This approach most similar to \citet{paquette_sgd_2021}, which focused on a specific, 
high-dimensional, rotationally invariant limit; the majority of our analysis will not make such
assumptions.

\aga{New structure}

\subsection{Linearized model and deterministic EOS}

We first define the basic model of study. Consider a $\P$-dimensional parameter
vector $\th$ and a $\D$-dimensional output function $\f(\th)$.
\aga{next line helpful or not?}
We will generally interpret the $\D$ outputs as coming from $\D$ inputs with $1$-dimensional outputs; however, our analysis
naturally covers the case of $C$-dimensional outputs on $\D/C$ datapoints.

We focus on the case of MSE loss. Given training targets $\y_{tr}$, the full loss is
given by
\begin{equation}
\Lo(\th) = \frac{1}{2\D}||\z||^{2},~\z\equiv \f(\th)-\y_{tr}.
\end{equation}
We will consider training with minibatch SGD with batch size $\B$, which can
be described as follows. Let $\pmat_{t}$ be a sequence of random, i.i.d.
diagonal matrices with exactly $\B$ random $1$s on the diagonal, and $0$s everywhere else.
Then the loss for minibatch $t$ is given by
\begin{equation}
\Lo_{mb,t}(\th) = \frac{1}{2\B}\z^{\tpose}\pmat_{t}\z.
\end{equation}
Like the case of full batch EOS, we will construct
a convex approximation to the training setup.
Consider linearizing $\f$ around a point $\th_{0}$:
\begin{equation}
\f(\th) \approx \f(\th_{0})+\J[\th-\th_{0}]
\end{equation}
where we have ignored higher order terms of $O(||\th-\th_{0}||^{2})$. Here
$\J\equiv \frac{\partial \f}{\partial\th}(\th_{0})$ is the $\D\times\P$-dimensional Jacobian at $\th_{0}$. For convenience we assume, WLOG, that $\th_{0} = 0$.
The update rule for minibatch gradient descent on the linearized model with MSE loss is
\begin{equation}
\th_{t+1}-\th_{t} = -\frac{\lr}{\B} \J^{\tpose}\pmat_{t}\z_{t}.
\end{equation}
To understand the dynamics in function space we can write
the updates for $\z_{t}$:
\begin{equation}
\z_{t+1}-\z_{t}  = -\frac{\lr}{\B}\J \J^{\top}\pmat_{t}\z_{t}.
\label{eq:lin_reg_sgd}
\end{equation}
We can get a basic understanding of the behavior of this system by averaging $\z$ with respect to
the minibatch sampling $\pmat$. The first moment evolves as:
\begin{equation}
\expect_{\pmat}[\z_{t+1}-\z_{t}|\z_{t}, \J_{t}] = -\lr\ntk\expect_{\pmat}[\z_{t}]
\label{eq:z_ave_quad}
\end{equation}
where we define the (empirical) \emph{neural tangent kernel} (NTK, \citep{jacot_neural_2018})  as $\ntk \equiv \frac{1}{\D}\J\J^{\top}$.

This gives us a linear recurrence equation for $\expect[\z_{t}]$, which
converges to $0$ if and only if $\lr\lam_{max} < 2$ for the largest eigenvalue
$\lam_{\max}$ of $\ntk$. This is exactly the full-batch (deterministic) EOS condition. Therefore we
can interpret the ``standard'' EOS as a stability
condition on the first moment of $\z_{t}$.

\aga{end new structure}

\subsection{Second moment stability defines stochastic EOS}

We now describe a method to find \emph{noise-driven} instabilities in the dynamics
of Equation \ref{eq:lin_reg_sgd} which have no full-batch analogue.
These instabilities are found by analyzing the long-time behavior of the
\emph{second moments} of $\z$. We will find a stability condition in terms of $\ntk$, $\lr$,
and $\B$ which we will call the \emph{stochastic EOS} (S-EOS). The covariance of the residuals evolves as:
\begin{equation}
\begin{split}
\expect_{\pmat}[\z_{t+1}\z_{t+1}^{\top}|\z_{t}]  = \z_{t}\z_{t}^{\top}-\lr \left(\ntk\z_{t}\z_{t}^{\top}+\z_{t}\z_{t}^{\top}\ntk\right) \\
+\tl{\bfr}\bfr^{-1}\lr^{2}\ntk\z_{t}\z_{t}^{\top}\ntk
+ (\bfr^{-1}-\tl{\bfr}\bfr^{-1})\lr^{2} \ntk\diag\left[\z_{t}\z_{t}^{\top}\right]\ntk
\end{split}
\label{eq:cov_dyn_lin}
\end{equation}
where $\bfr\equiv \B/\D$ is the batch fraction, and $\tl{\bfr}\equiv (\B-1)/(\D-1)$. Inspecting Equation \ref{eq:cov_dyn_lin}, we see that the covariance evolves as a linear dynamical system, whose corresponding linear operator we denote will denote as $\linop$ (see Appendix \ref{app:second_mom_dyn} for a full expression).
The stability of the dynamics is controlled by $\max||\lam[\linop]||$, the largest eigenvalue
of $\linop$. If $\max||\lam[\linop]||<1$, the dynamics are stable
($\lim_{t\to\infty}\expect_{\pmat}[\z_{t}\z^{\top}_{t}] = 0$). If $\max||\lam[\linop]||<1$, then the
dynamics diverge ($\lim_{t\to\infty}\expect_{\pmat}[\z_{t}\z^{\top}_{t}] = \infty$). Note that
$\expect_{\pmat}[\z_{t}^{\top}\z_{t}]$ is the expected loss.

We say a system is at the \emph{stochastic edge of stability} (S-EOS) if both
$\lr\max\lam[\ntk]<2$ and $\max||\lam[\linop]|| = 1$. This is impossible in the full
batch setting $\bfr = 1$, but for SGD the
last term
in Equation \ref{eq:cov_dyn_lin} contributes to $\max||\lam[\linop]||$, and there are systems which are unstable due to the effects of SGD noise (Figure \ref{fig:knorm_stability_basics}, left).

\subsection{Noise kernel norm}

In general, $\linop$ is a $\D^{2}\times\D^{2}$ matrix, whose entries are derived from $\P$-dimensional
inner products. This can quickly become intractable for
large $\D$ and $\P$. Addtionally, $\max||\lam[\linop]||$ does \emph{not} distinguish between noise-driven and
deterministically-driven instabilities. We will use a $\D\times\D$ dimensional approximation
to the dynamics to define the \emph{noise kernel norm} $\knorm$ - an interpretable measure
of the influence of noise in the optimization dynamics and a good predictor of the S-EOS.

Consider the rotated covariance
$\S_{t}\equiv \V^{\tpose}\expect_{\pmat}[\z_{t}\z_{t}^{\tpose}]\V$, where $\V$
comes from the eigendecomposition
$\ntk = \V\lmat\V^{\tpose}$. We define the normalized diagonal $\tpvec \equiv \lmat^{+}\diag(\S)$,
where $\diag(\S)$ is the vector obtained from the diagonal of $\S$. Consider the dynamics of
$\tpvec$ under the linear operator $\linop$, restricted to $\tpvec$. That is, we ignore
any contributions to the dynamics from  
terms like $\expect_{\pmat}[(\v\cdot\z_{t})(\v'\cdot\z_{t})]$ for distinct eigenvectors
$\v$ and $\v'$ of $\ntk$. We have (Appendix \ref{app:second_mom_dyn}):
\begin{equation}
\begin{split}
\tpvec_{t+1} & = (\amat+\bmat)^{t}\tpvec_{0},~ \amat  \equiv(\Id-\lr\lmat)^2+(\tl{\bfr}\bfr^{-1}-1)\lmat^{2}\\
& ~\bmat\equiv (\bfr^{-1}-\tl{\bfr}\bfr^{-1})\lr^2\lmat\cmat\lmat.
\end{split}
\label{eq:pvec_volt}
\end{equation}
Here $\amat$ (the deterministic contribution)
and $\bmat$ (the stochastic contribution) are both PSD matrices, and
$\cmat_{\bt\mu} \equiv \sum_{\al}\V_{\al\bt}^2\V_{\al\mu}^2$ gives the noise-induced
coupling between the eigenmodes of $\ntk$. The largest eigenvalue of this linear system
gives us an approximation of $\max||\lam[\linop]||$.

Instead of computing $\max\lam[\amat+\bmat]$ directly, we define the
\emph{noise kernel norm} $\knorm$, which interpolates from $0$ for $\bfr = 1$
(no noise) to $\knorm = 1$ at the S-EOS. In Appendix \ref{app:s_eos_def} we 
prove the following:
\begin{theorem}
If the diagonal of $\S$ is governed by Equation \ref{eq:pvec_volt}, then $\lim_{t\to\infty}\expect_{\pmat}[\z_{t}\z_{t}^{\top}] = 0$ for any initialization $\z_{t}$
if and only if $||\amat||_{op} <1$ and $\knorm < 1$ where
\begin{equation}
\knorm \equiv \max\lam\left[(\Id-\amat)^{-1}\bmat\right]
\label{eq:knorm_def}
\end{equation}
for the PSD matrices $\amat$ and $\bmat$ defined above. $\knorm$ is always non-negative.
\label{thm:s_eos}
\end{theorem}

$\knorm$ is a normalized measure of the SGD-induced noise in the dynamics. For $\bfr = 1$
(full-batch training), $\knorm = 0$ -  there is no noise.
This is in contrast to $\max\lam[\amat+\bmat]$, which is often close to $1$ even in the
deterministic setting (Figure \ref{fig:knorm_stability_basics}, middle),
where it is given by $(1-\lr\lam_{min})^{2}$ for the minimum eigenvalue $\lam_{min}$ of $\ntk$.
Even though $\knorm$ is derived from an approximation of $\linop$,
the S-EOS is often well-predicted by $\knorm = 1$- even for small systems
(Figure \ref{fig:knorm_stability_basics}, right, $\D = 100$).
As we will show later, these properties of $\knorm$ make it suitable for analysis of the effects
of SGD in non-convex settings.

\subsection{Approximations of $\knorm$}

\label{sec:knorm_scaled}

A key difference between the S-EOS and the deterministic EOS is that the S-EOS depends
on the whole spectrum of $\ntk$. We can show this directly by computing approximations to
$\knorm$. These will have the additional benefit of being easy to compute, especially on
real neural network setups.
In the high-dimensional limit,
\citet{paquette_sgd_2021} showed that $\tl{\bfr} \approx \bfr$ and $\cmat \approx \frac{1}{\D}\m{1}\m{1}^{\tpose}$,
and we arrive at
\begin{equation}
\knorm \approx \hat{\knorm}_{HD} \equiv  \frac{\lr}{\B}\sum_{\al=1}^{\D} \frac{\lam_{\al}}{2-\lr\lam_{\al}}
\end{equation}
where the $\lam_{\al}$ are the eigenvalues of $\ntk$. The key features are the dependence
on the ratio $\lr/\B$, and the fact that eigenvalues close to the deterministic EOS
$\lr\lam = 2$ have higher weight. We can immediately see that the S-EOS condition is not vacuous;
if the largest $\B$ eigenvalues have $\lr\lam = 1$, then $\knorm\geq 1$ while
$\lr\lam_{max} < 2$. If $\lr\lam_{\al}\ll 2$ for all eigenvalues,
we have the approximation
\begin{equation}
\knorm \approx \hat{\knorm}_{tr}\equiv \frac{\lr}{2\B} \tr\left(\ntk\right)
\label{eq:knorm_approx}
\end{equation}

Equation \ref{eq:knorm_approx} gives us an intuitive understanding of SGD noise. $\knorm$
depends on the ratio $\lr/\B$ which controls the scale of the noise in SDE-based analyses
of SGD \cite{smith_don_2017, jastrzebski_three_2018}. The dependence on the trace of the empirical NTK shows
that the noise depends on many eigendirections. It is interesting to note that some
popular regularization techniques implicitly or explicitly regularize a similar quantity
\cite{wen_how_2023, dauphin_neglected_2024}.



The approximations of $\knorm$ underestimate the noise level; we have
$\hat{\knorm}_{tr} \leq \hat{\knorm}_{HD} \leq \knorm$. In general $\hat{\knorm}_{tr}$ becomes
a poor predictor of $\knorm$ when there are eigenvalues close to $2/\lr$. $\hat{\knorm}_{HD}$
loses accuracy when there is a large spread of eigenvalues. Both become inaccurate when
the eigenvectors $\V$ of $\ntk$ are no longer delocalized with respect to the coordinate
basis of $\z$.
See Appendix \ref{app:knorm_highd_approx} for more details.

Though our exact analysis is restricted to MSE loss, any model can be locally linearized. The relevant quantity then
becomes the trace of the Gram matrix of the Gauss-Newton matrix (Appendix \ref{app:non_mse_loss}).
In that setting, the analysis breaks down if the linearization changes over training timescales.

Nevertheless, $\knorm$ and its approximations are accurate enough to
estimate the effect of noise on optimization trajectories in many linear regression settings.
In Section \ref{sec:experiments} we provide
experimental evidence that $\knorm$ and the S-EOS are useful for understanding aspects of
non-linear settings as well - particularly, training deep neural networks.



\begin{figure}[ht]
    \centering
    \begin{tabular}{cc}
    \includegraphics[width=0.45\linewidth]{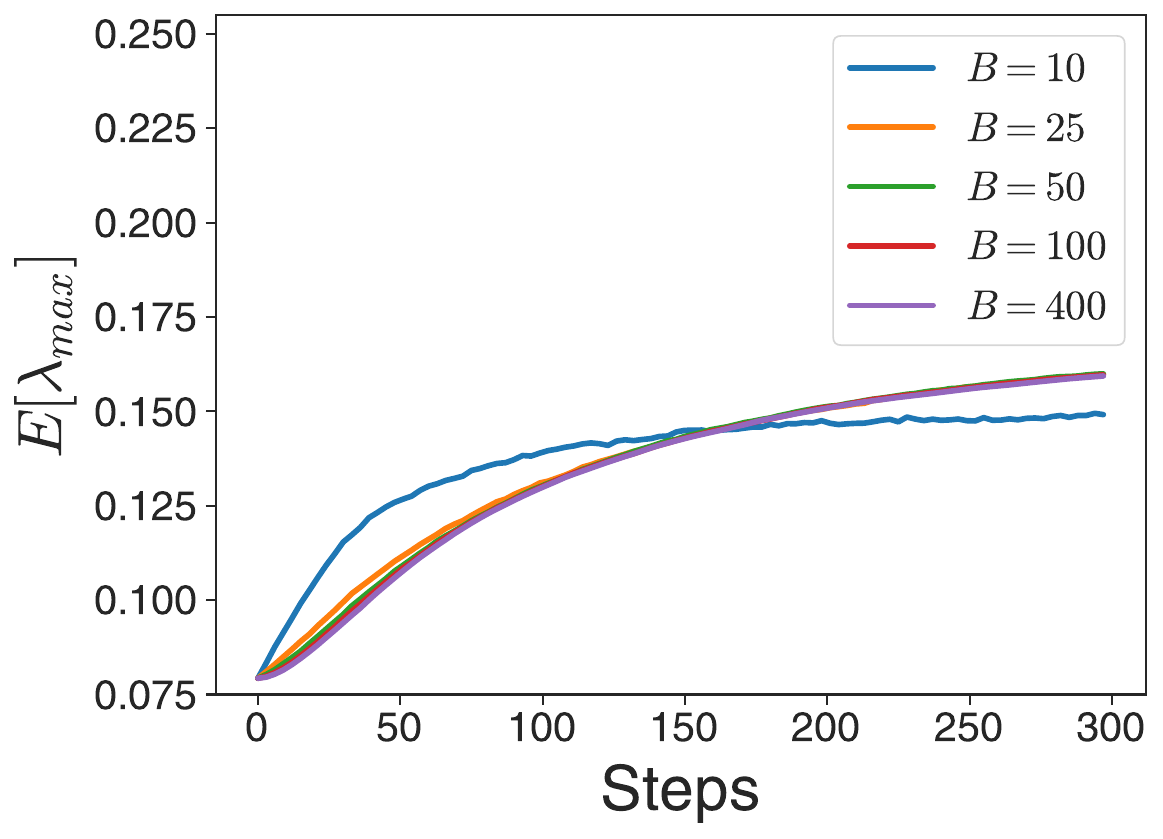} & \includegraphics[width=0.45\linewidth]{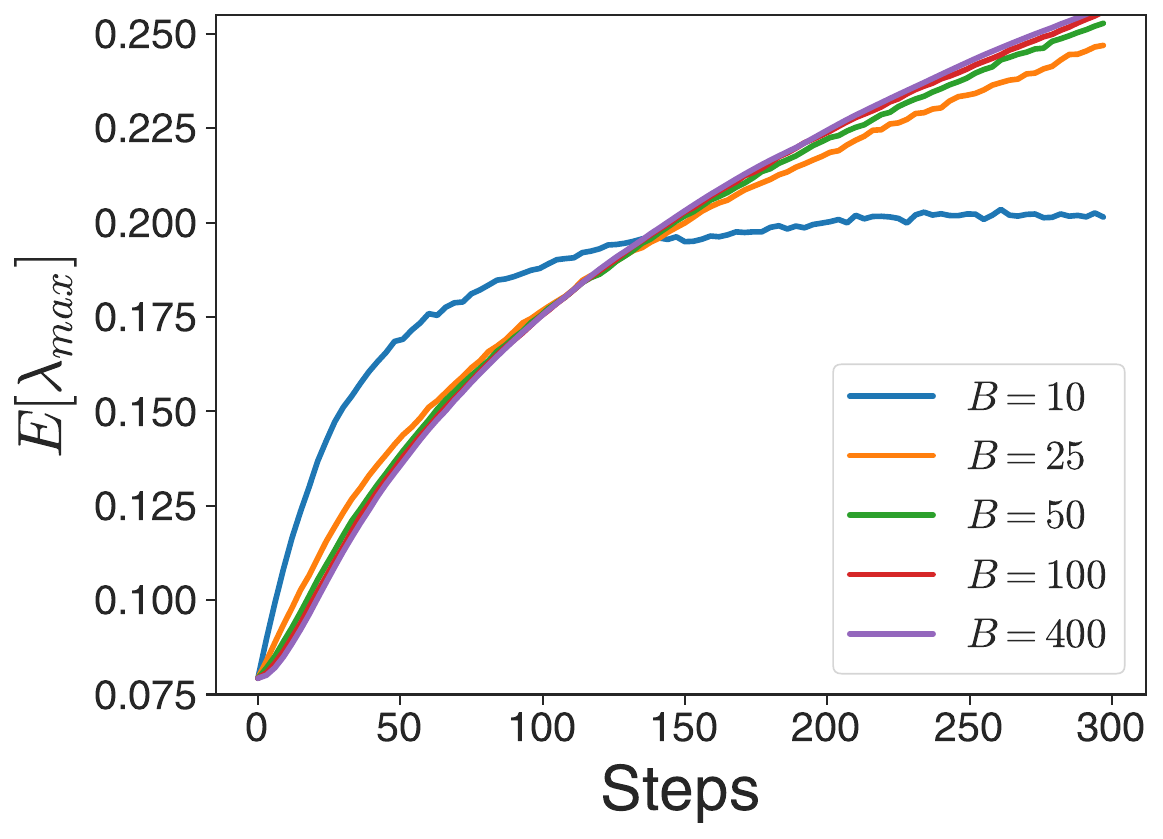}
    \end{tabular}
    \caption{Dynamics of largest Hessian eigenvalue in randomly initialized
    quadratic regression model for fixed learning rate, various batch sizes (averaged over $100$ seeds.
    Small batch size leads to increased initial sharpening, but faster
    saturation (left, $V(\sigma) = 1$). Batch size differences are amplified when $\Q$ is more
    heavily weighted in larger eigenmodes (right, $V(\sigma) = \sigma$). 
    }
    \label{fig:cons_sharpening_example}
\end{figure}

\section{Conservative sharpening}

\label{sec:con_sharp}

In this section, we analyze \emph{conservative sharpening} - the suppression of Hessian
eigenvalue increase with decreasing batch size. We will provide theoretical evidence that
SGD noise suppresses larger eigenvalues more than smaller ones. This phenomenology can
help explain conditions under which the S-EOS can be reached in non-convex settings.

\subsection{Quadratic regression model dynamics}

The most basic model of curvature dynamics requires non-linear models. The simplest such model
is the \emph{quadratic regression model} \cite{agarwala_secondorder_2022, zhu_quadratic_2022}.
The model can be derived by a second order Taylor expansion of $\f(\th)$. Under MSE loss,
it can be shown (Appendix \ref{app:quad_reg_model})
that the SGD dynamics can be written in terms of the residuals $\z_{t}$ and the
(time-varying) Jacobian $\J_{t}$ as
\begin{align}
\z_{t+1}-\z_{t} & = -\frac{\lr}{\B}\J_{t} \J_{t}^{\top}\pmat_{t}\z_{t}  +\frac{\lr^{2}}{2\B^{2}} \Q(\J_{t}^{\top}\pmat_{t}\z_{t},\J_{t}^{\top}\pmat_{t}\z_{t})\notag\\
\J_{t+1} -\J_{t} & = -\frac{\lr}{\B} \Q(\J_{t}^{\top}\pmat_{t}\z_{t}, \cdot)\,.
\label{eq:SGD_general}
\end{align}
Here $\Q$ is the $\D\times\P\times\P$ dimensional \emph{model curvature} tensor
$\frac{\partial^{2}\f}{\partial\th\partial\th'}$, taken as a fixed value at some
point $\th_{0}$. Equation \ref{eq:SGD_general} lets us understand the joint dynamics
of the loss and geometry directly.

We study the dynamics of the singular values of $\J$ (and therefore the eigenvalues of
$\ntk$)
at early times in the quadratic regression model of Equation \ref{eq:SGD_general}. We will
model $\z$ at initialization as
i.i.d. random and independent of $\J$ and $\Q$.
It has been previously observed that the model curvature tensor $\Q$ has more ``weight'' in
directions corresponding to the large NTK eigenvalues \citep{agarwala_secondorder_2022}.
Therefore we will model $\Q$ using a tensor product decomposition. Let $\w_{\al}$
be the left singular vector of $\J_{0}$ associated with singular value $\sigma_{\al}$.
Then we will decompose $\Q$ as:
\begin{equation}
\Q = \sum_{\al} \w_{\al}\otimes\m{M}_{\al}
\label{eq:q_def}
\end{equation}
where each $\m{M}_{\al}$ is a random $\P\times\P$
symmetric matrix with i.i.d. elements with mean $0$ and variance $V(\sigma_{\al})$,
for some non-decreasing function $V$. We use random matrices to model $\m{M}_{\al}$
to study the eigenvalue dynamics under some minimal high-dimensional structure.
Note that $V(\sigma) = 1$ is equivalent to
an i.i.d. initialization of each element of $\Q$.

\subsection{Estimating eigenvalue dynamics under SGD}

In order to understand the eigenvalue dynamics, we will assume that the eigenvectors of the NTK
change relatively slowly. This has been shown empirically for the large eigendirection of the Hessian
\citep{bao_hessian_2023}, which correlate with the large NTK eigendirections (which are of
particular interest here).
Consider the following estimators.
Let $\{(\w_{\al}, \v_{\al}, \sigma_{\al})\}$ be the set of triples that consists of
a pair of the left and right singular vectors of $\J_{0}$ associated with singular value
$\sigma_{\al}$. We define the equivalent
approximate singular value $\shat_{\al, t}$ and NTK eigenvalue $\lhat_{\al, t}$
as
\begin{equation}
\shat_{\al, t} \equiv \w_{\al}^{\tpose}\J_{t}\v_{\al},~\lhat_{\al, t} \equiv \w_{\al}^{\tpose}\J_{t}\J_{t}^{\tpose}\w_{\al}
\label{eq:sig_hat}
\end{equation}
Note that $\shat_{\al, 0}^{2} = \lhat_{\al, 0} = \sigma_{\al}^{2}$. If the singular
vectors change slowly, then this lets us approximate the eigenvalues.
We will also compute the \emph{discrete time derivatives}; for any timeseries
$\{x_{t}\}$ we write
\begin{equation}
\Delta_{1}x_{t} \equiv x_{t+1}-x_{t},~\Delta_{2}x_{t} \equiv x_{t+2}-2x_{t+1}+x_{t}.
\label{eq:discrete_deriv}
\end{equation}
We will show that the discrete first derivative increases with batch size while the
discrete second derivative decreases with batch size, dependent on $\sigma_{\al}$
and $V(\sigma_{\al})$.
Concretely, we prove the following theorem (Appendix \ref{app:cons_sharp_proofs}):
\begin{theorem}
Let $\{(\w_{\al}, \v_{\al}, \sigma_{\al})\}$ be the triple of left and right singular vectors
of $\J_{0}$ with the associated singular value. Let $\Q$ be a random tensor with the
decomposition given by Equation \ref{eq:q_def}. Let $\z_{0}$ have i.i.d. elements with
mean $0$ and variance $V_{z}$. If $\z$, $\J$, and $\Q$ are statistically
independent, we can compute the following average discrete time
derivatives (Equation \ref{eq:discrete_deriv})
of the estimators $\shat_{0}$ and $\lhat_{0}$ (Equation \ref{eq:sig_hat}):
\begin{equation}
\expect_{\pmat,\Q,\z}[\Delta_{1}\lhat_{\al,0}] =
\B^{-1}\P V_{z}\tr\left[\ntk_{t}\right]\lr^{2} V(\sigma_{\al})+O(\D^{-1})
\label{eq:first_deriv}
\end{equation}
\vspace{-2.5mm}
\begin{equation}
\expect_{\pmat, \Q, \z}[\Delta_{2}\shat_{\al,0}]  = d_{2}(\lr)
-\B^{-1}\D^{-2}\lr^{3}\sigma_{\al, t}^{3}V(\sigma_{\al}) \P V_{z} +O(\lr^4)
\label{eq:second deriv}
\end{equation}
where $d_{2}(\tilde{\lr}) = \expect_{\pmat, \Q, \z}[\Delta_{2}\shat_{\al,0}]$ for $\bfr = 1$
and $\lr = \tilde{\lr}$.
\label{thm:cons_sharp}
\end{theorem}
For small batch size $\B$, the first derivative is positive. This depends
on the projection $V(\sigma_{\al})$, but the average eigenvalue of $\ntk$. In contrast,
the second derivative is smaller for smaller $\B$ (and can even become negative), and
also shows sensitivity to the particular singular value
$\sigma_{\al}^{3}$. This suggests that the deviations
due to SGD are more pronounced for eigenmodes with larger model curvature $\Q$, but also
that conservative sharpening is stronger for larger eigenmodes.

We can see this in numerical simulations of randomly initialized $\{\z, \J, \Q\}$ as well.
For a ``flat'' weighting $V(\sigma) = 1$, at small batch sizes the largest eigenvalue
increases more quickly than the full batch case, but its growth slows down quicker
(Figure \ref{fig:cons_sharpening_example}, left). This effect is even stronger for the
correlated weighting $V(\sigma) = \sigma$
(Figure \ref{fig:cons_sharpening_example}, right). This supports the claim that
conservative sharpening depends on not just batch size, but the spectrum of $\Q$ as well.
Our results suggest that conservative sharpening can suppress the large eigenvalues
more than the smaller ones - preventing small batch size models from reaching
the deterministic EOS while leaving the S-EOS attainable.

\section{Experiments on neural networks}

\label{sec:experiments}

We conducted experimental studies on neural networks to understand how the noise kernel
norm $\knorm$ behaves in the convex setting.
We will show that for small batch sizes,
$\knorm$ is a more informative object to study than $\lam_{max}$, the key measurement in the
full batch setting. We show that the best training outcomes come from settings where $\knorm$
is \emph{below} the S-EOS, unlike the full batch case where best training happens \emph{at} the EOS.

\subsection{Fully connected network, vanilla SGD}

\label{sec:fcn_experiments}

We begin by training a fully connected network on $2500$ examples of MNIST with MSE loss.
The details of the setup can be found in Appendix \ref{app:mnist_details}. In this setting
we can compute $\knorm$ exactly and efficiently. We trained with a variety of batch sizes
$\B$ and learning rates $\lr$ to probe the dependence of learning dynamics on each of these
hyperparameters.

Plotting training loss trajectories for fixed, small $\B$ and varying $\lr$ elucidates some of the
key phenomenology (Figure \ref{fig:mnist_loss_kern_dyn}, left, for $\B = 1$). For very small
$\lr$, the loss decreases smoothly but slowly. For larger $\lr$, the optimization is more
efficient, and similar over a range of learning rates. Finally, for larger learning rates,
the loss decreases slowly, until for the largest learning rates the loss diverges.

\begin{figure}[t]
    \centering
    \begin{tabular}{cc}
    \includegraphics[width=0.4\linewidth]{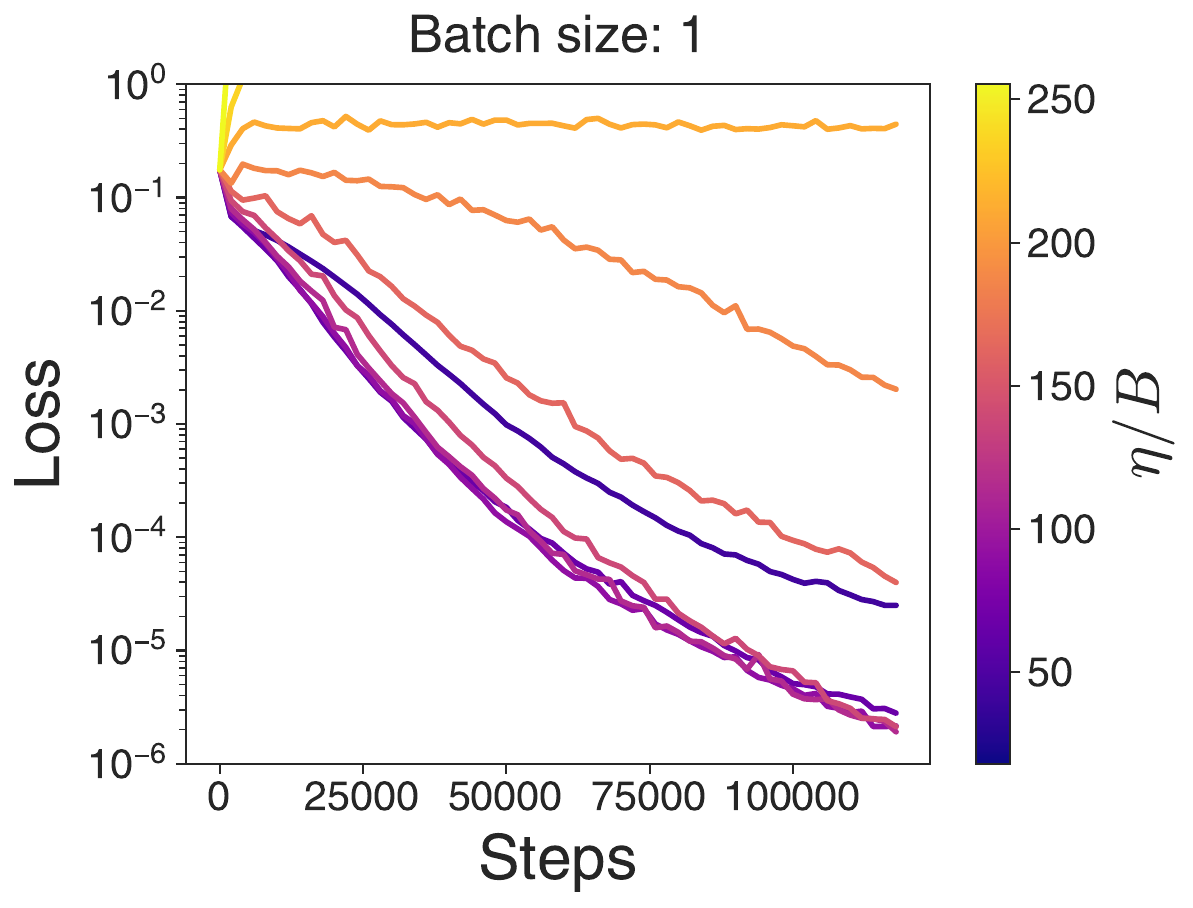} & \includegraphics[width=0.4\linewidth]{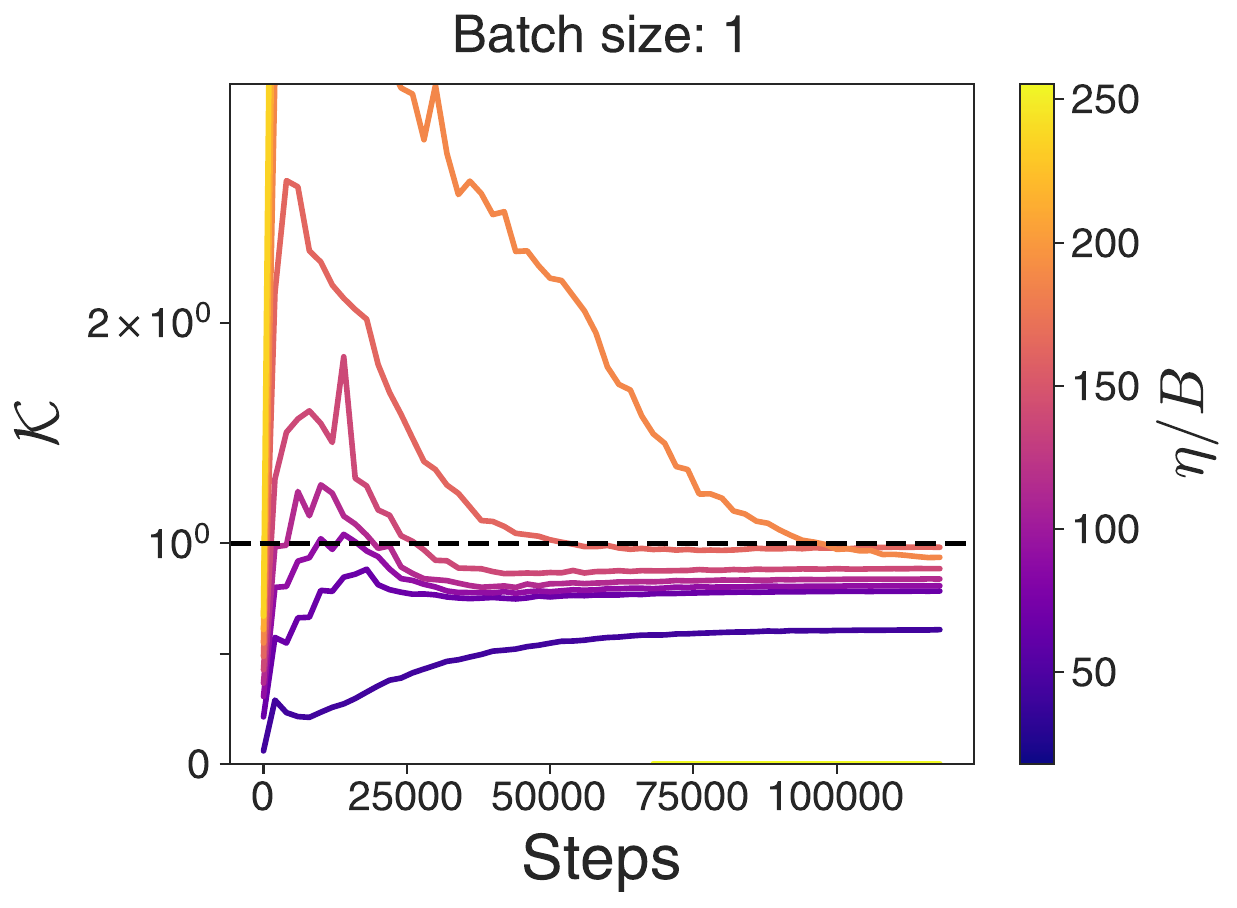}
    \end{tabular}
    \caption{Dynamics of loss (left) and noise kernel norm $\knorm$ (right) for a FCN trained on MNIST, various learning rates, batch size $1$. For small
    learning rates, loss decrease is slow and kernel norm is well below $1$. For intermediate learning rates, $\knorm$ is larger than the
    critical value of $1$, but then decreases and stabilizes below $1$ and
    loss decreases quickly. For larger learning rates, $\knorm$ stays
    above $1$ for a long period and loss decreases slowly.
    }
    \label{fig:mnist_loss_kern_dyn}
\end{figure}

These different regimes are reflected in the dynamics of $\knorm$ as well (Figure
\ref{fig:mnist_loss_kern_dyn}, right). At small $\lr$,
$\knorm$ is small. This corresponds to a low noise regime where the steps are being taken
conservatively. As $\lr$ increases, we begin to see the emergence of S-EOS stabilization -
$\knorm$ is initially increasing, attains values above the S-EOS boundary $\knorm = 1$,
but eventually stabilizes below $1$. For the poorly optimizing trajectories at large $\lr$,
$\knorm$ stays above $1$ for a longer time.

These experiments suggest that there is a
negative feedback effect which prevents the runaway growth of $\knorm$ at intermediate
$\lr$, and eventually drives it below the critical threshold. Unlike the
deterministic EOS, the S-EOS involves only a single, multistep return to the critical
value - unlike
the period $2$ quasi-stable oscillations around the boundary
which characterize the deterministic EOS phase
\citep{agarwala_secondorder_2022, damian_selfstabilization_2022}.

\begin{figure}[ht]
    \centering
    \begin{tabular}{cc}
    \includegraphics[height=0.34\linewidth]{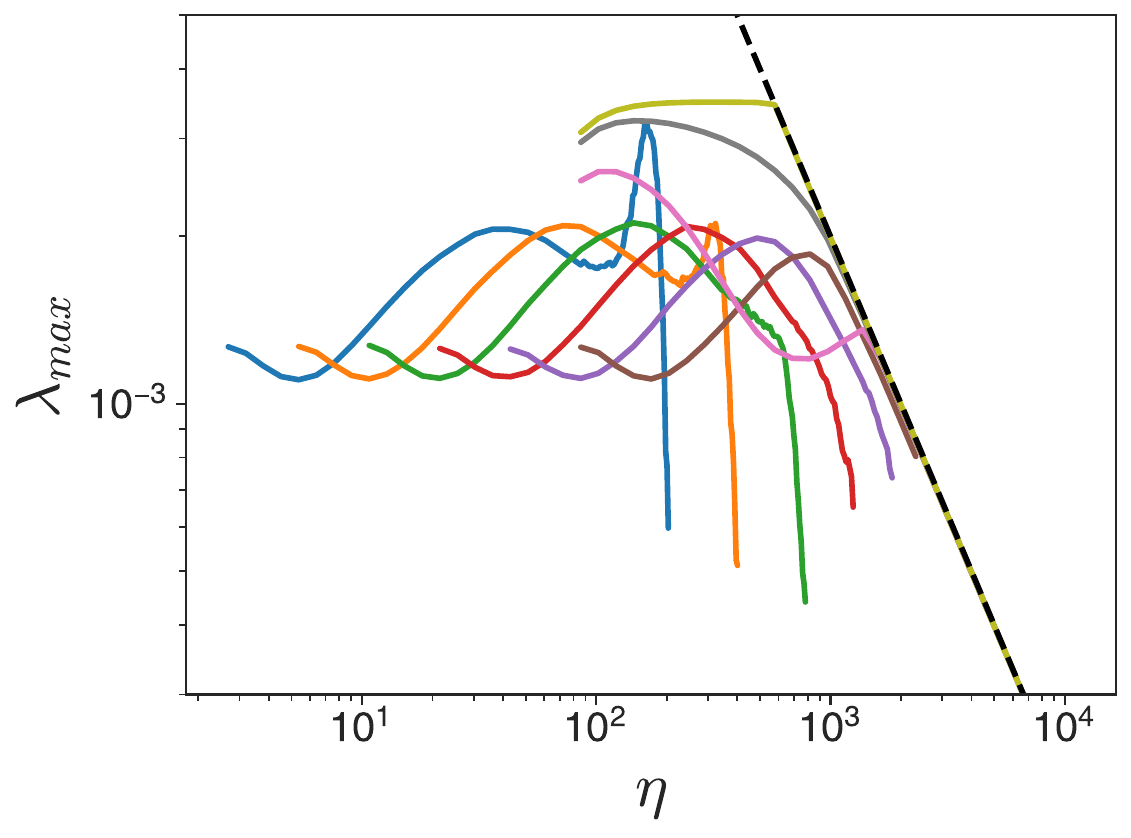} & \includegraphics[height=0.34\linewidth]{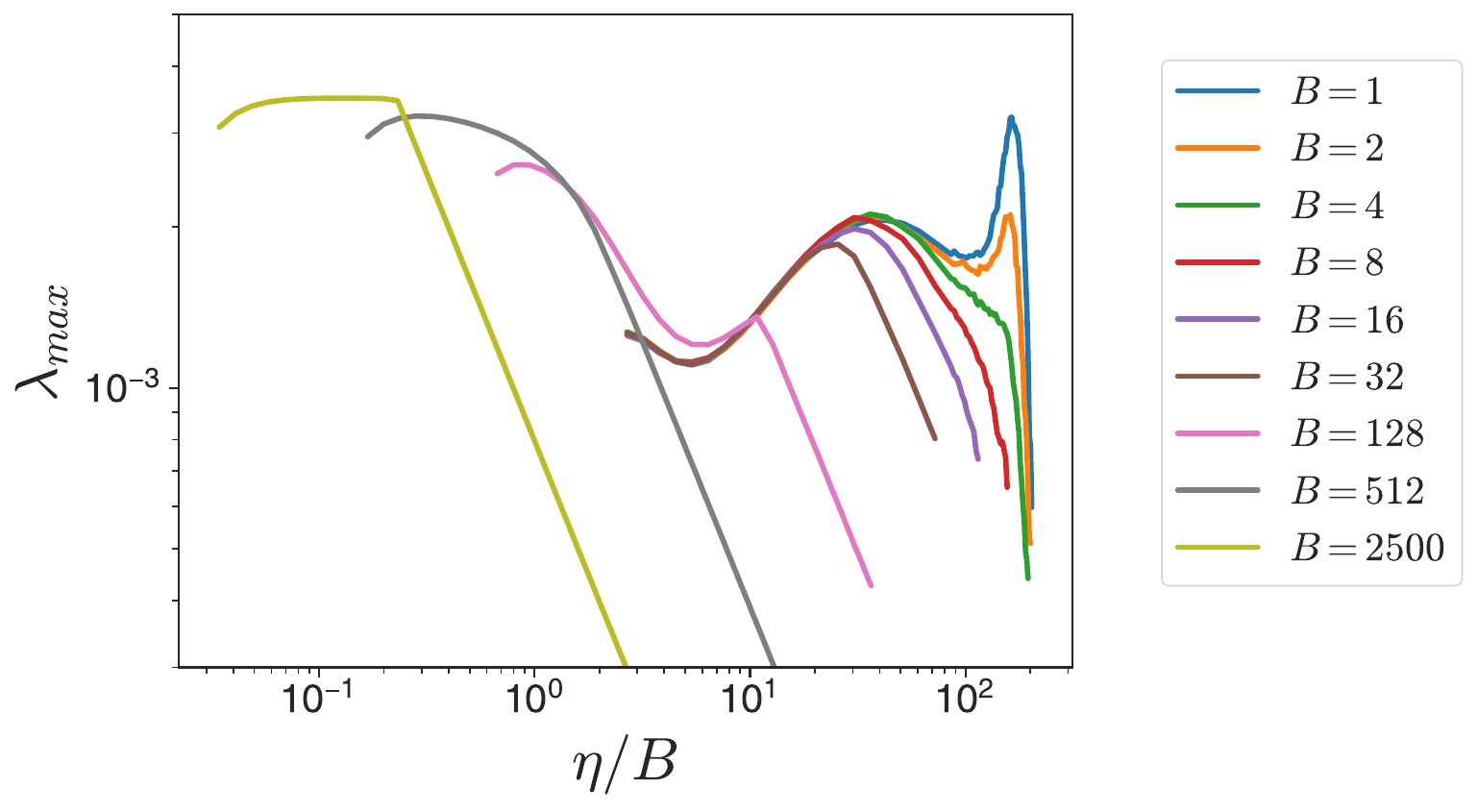}
    \end{tabular}
    \caption{$\lambda_{max}$ at convergence in MNIST experiment. Left: for large $\B$, final values of $\lam_{max}$ are similar for same $\lr$, especially when dynamics reaches EOS as $2/\lr$ (black dashed line); for small $\B$, $\lr$ is not predictive
    of $\lam_{max}$ and EOS is not reached. Right: quantities are similar
    for equal $\lr/\B$ for small $\B$ and small $\lr/\B$.
    }
    \label{fig:mnist_max}
\end{figure}

We also studied the dynamics of the largest NTK eigenvalue $\lam_{max}$ as a function of
batch size and learning rate. For larger batch sizes, the final value of
$\lam_{max}$ stabilizes at the deterministic EOS, $2/\lr$, over a wide range of learning
rates (Figure \ref{fig:mnist_max}, left). However, for small batch sizes such large learning
rates lead to divergent training.
In this regime, it is more informative to plot the dynamics
as a function of $\lr/\B$ (Figure \ref{fig:mnist_max}, right). All batch sizes
follow the $\B = 1$ curve for small and intermediate $\lr/\B$, but there
are batch-size dependent effects for larger learning rates.

For small $\B$, it is more informative to study the final value of the noise kernel norm
$\knorm_{f}$ after a fixed number of epochs of training (Figure \ref{fig:mnist_kern_phase_error}, left,
480 epochs).
For small values of $\lr/\B$, $\knorm_{f}$ is small, as expected,
and there is consistent behavior across $\B$ for constant $\lr/\B$. As $\lr/\B$ increases,
there is a regime where the kernel norm takes on values in the range $[0.7, 0.9]$ over
a large range of learning rates. In this regime, there is consistency across constant
$\lr/\B$, over a limited range in $\B$ - dynamics for larger $\B$ now diverge.

In the small batch regime, $\knorm_{f}$ is also highly informative
of the final training loss reached (Figure \ref{fig:mnist_kern_phase_error}, middle). If $\knorm_{f}$
is small, the dynamics has low noise but doesn't get as far in the given number of epochs -
the choice of stepsize is too conservative given the noise level.
If $\knorm_{f}$ is too close to $1$, convergence also seems to slow down - the steps are large
and generate too much noise. In this setting there
appears to a good range of $\knorm_{f}\in[0.6, 0.8]$ where the learning
rate is aggressive enough to drive the loss down considerably, but not enough to cause
noise-induced convergence issues. In contrast, the maximum eigenvalue is a poor
predictor of the final loss, even when scaled by the learning rate (Figure \ref{fig:mnist_kern_phase_error}, right).



\begin{figure}[ht]
    \centering
\begin{tabular}{ccc}
\includegraphics[width=0.33\linewidth]{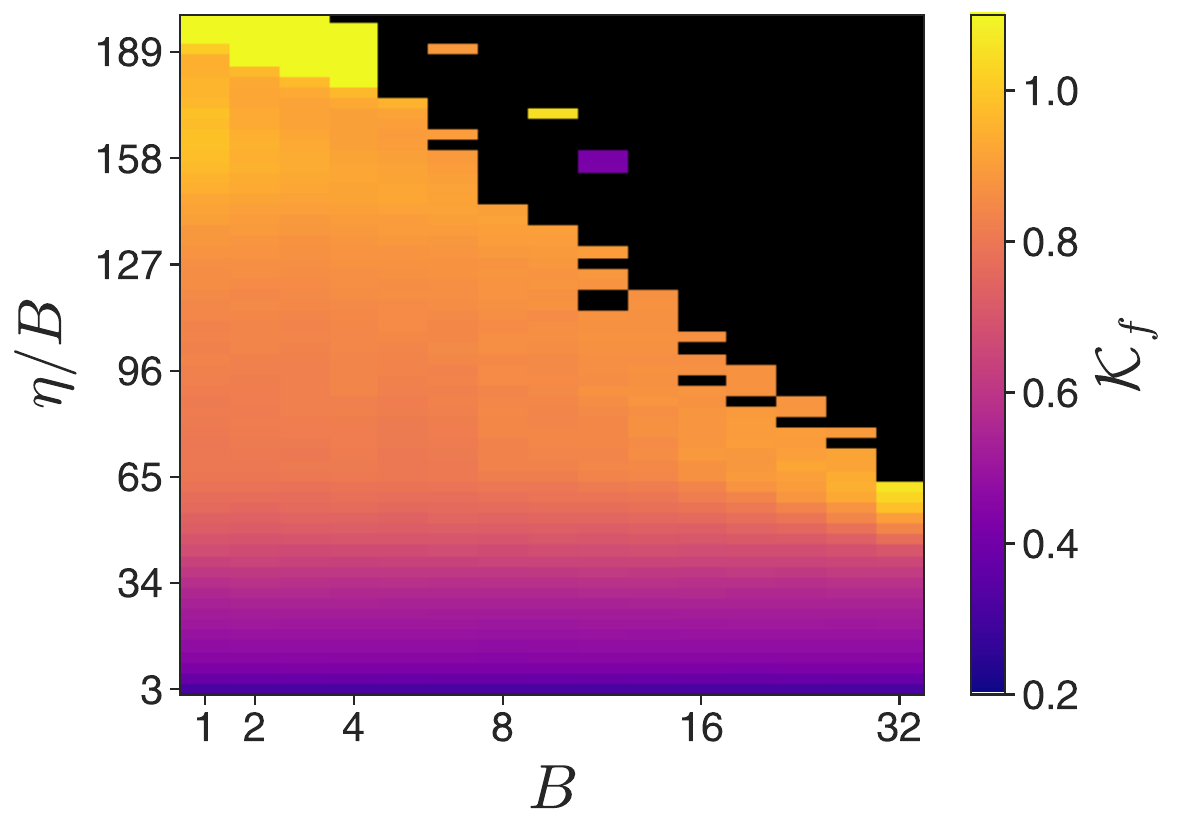} & 
    \includegraphics[width=0.33\linewidth]{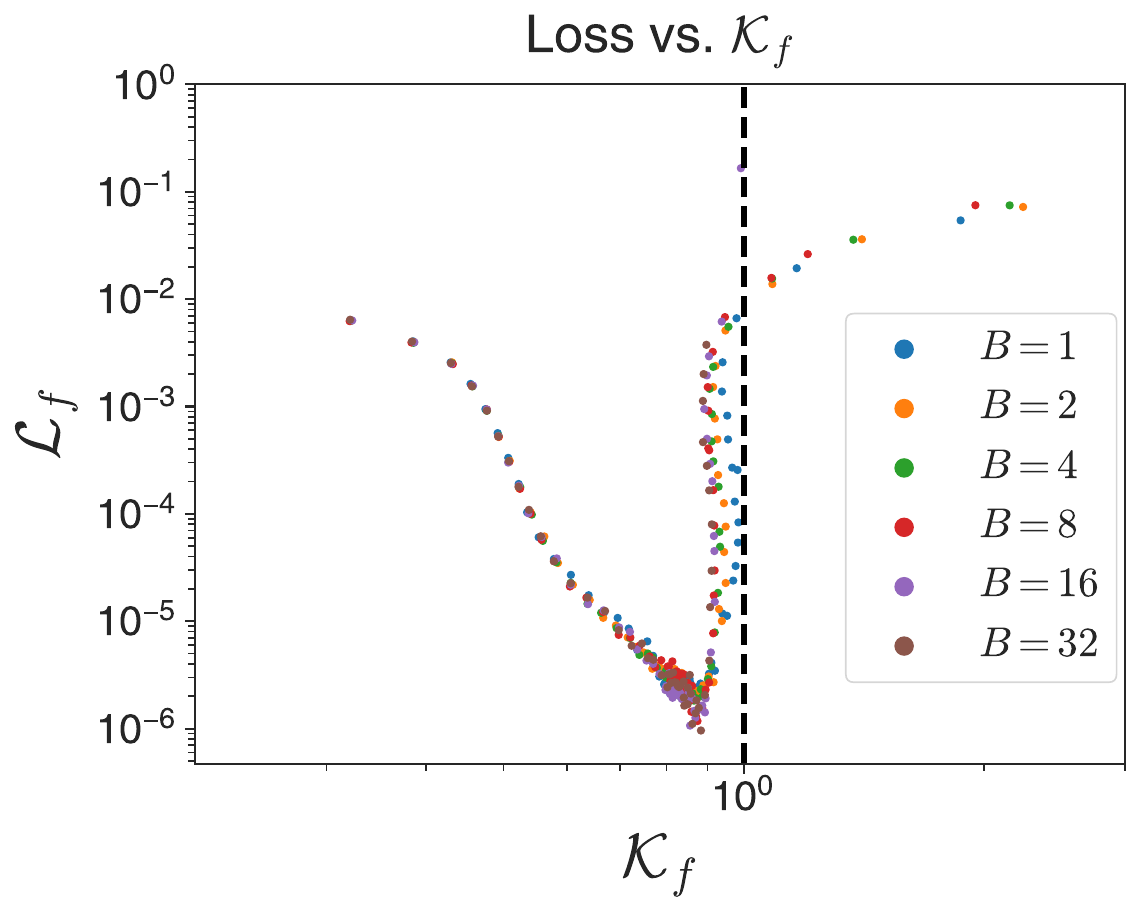} & \includegraphics[width=0.33\linewidth]{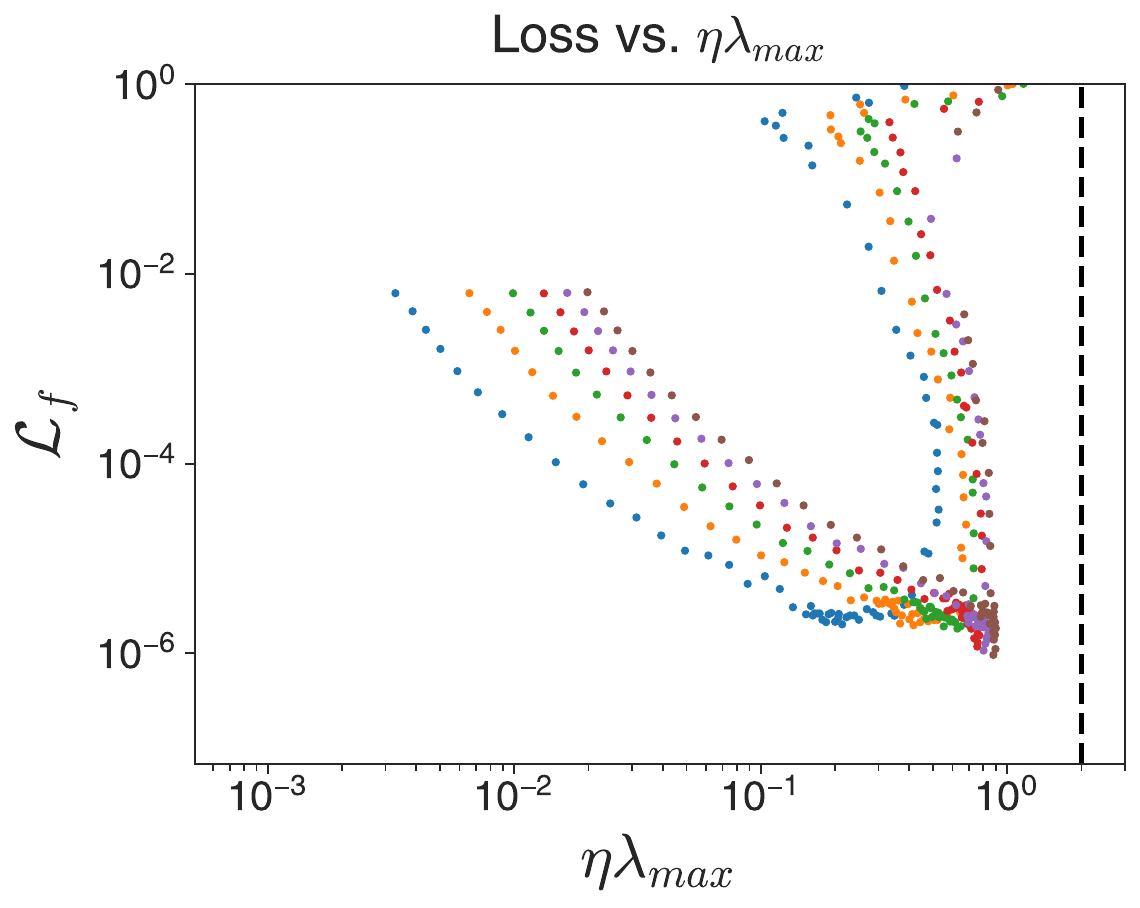}
    \end{tabular}

    \caption{Final noise kernel norm $\knorm_{f}$ is well predicted by $\lr/\B$ for
    fixed epoch training, and attains a value near $1$ over a large range of learning rates
    (left). Final loss is poor for $\knorm_{f}\ll 1$ (conservative steps) but also for $\knorm$
    too close to $1$ (aggressive steps) (middle). $\lam_{max}$ is not a good predictor
    of training loss (right).}
   \label{fig:mnist_kern_phase_error}
\end{figure}

\subsection{Momentum and learning rate schedule}

\label{sec:cifar_exp}

What does $\knorm$ look like in a bigger model where exact computation is intractable?
And what happens when common methods like momentum, learning rate schedule, and weight
norm are added?
In order to probe these questions, we ran experiments on ResNet-18 trained on CIFAR10, with
MSE loss, trained with momentum cosine learning rate schedule, and $L^{2}$ regularizer .
The experimental details can be found in Appendix \ref{app:cifar_details}.

Since the exact $\knorm$ requires analysis of a a $5\cdot 10^{5}\times5\cdot 10^{5}$
dimensional matrix, we used a trace estimator.
We computed additional corrections due to momentum and the
$L^{2}$ regularizer (see Appendices \ref{app:s_eos_mom} and \ref{app:s_eos_l2} for details).
We arrived at the estimator
\begin{equation}
\hat{\knorm}_{mom} \equiv \frac{\lr}{2\al\B}\tr\left[\ntk\right]
\end{equation}
where the momentum parameter $\mu = 1-\al$. In all our experiments,
$\al = 0.1$.


\begin{figure}[ht]
    \centering
    \begin{tabular}{ccc}
    \includegraphics[width=0.32\linewidth]{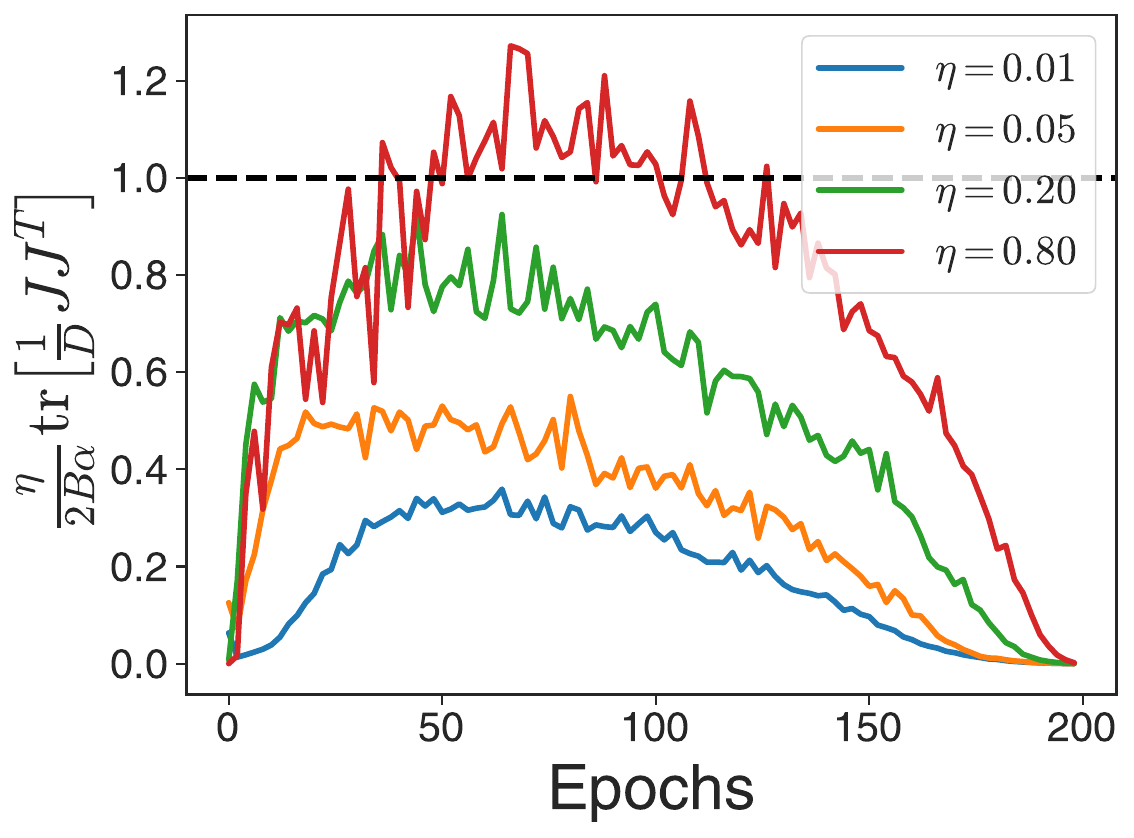} &
     \includegraphics[width=0.32\linewidth]{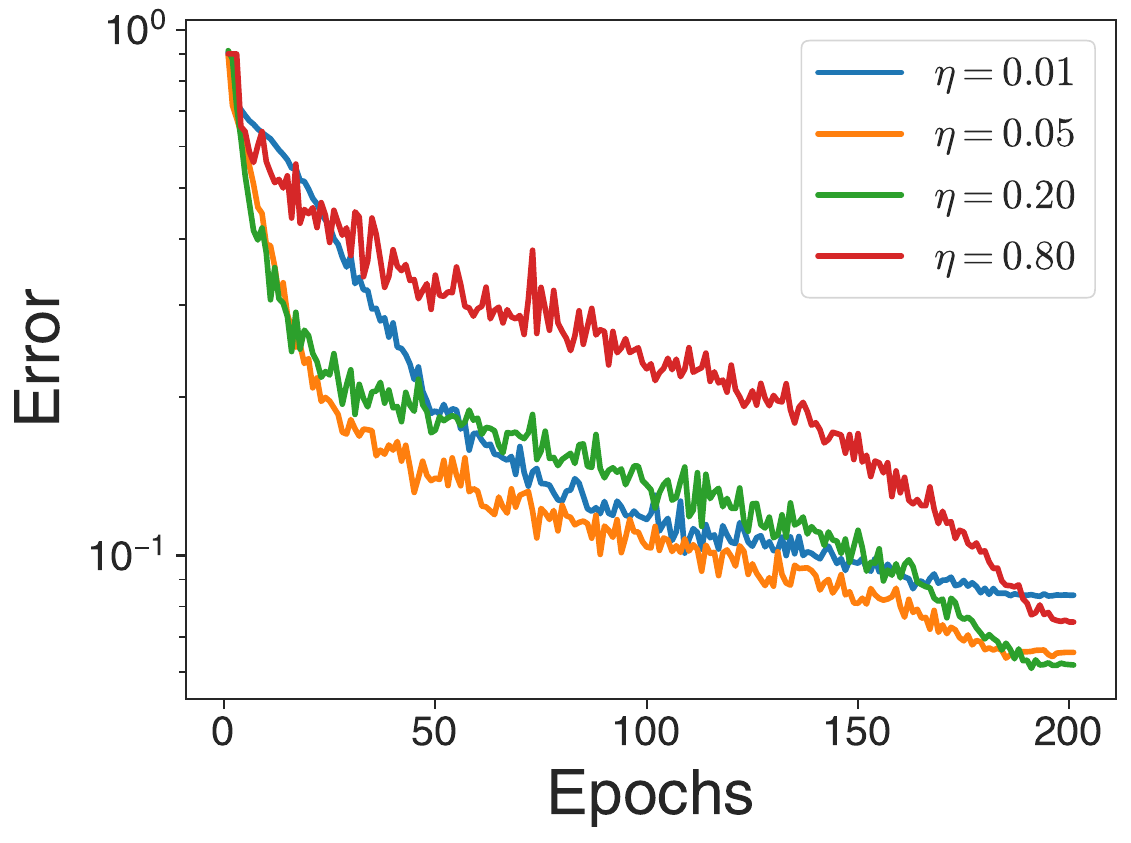} & \includegraphics[width=0.32\linewidth]{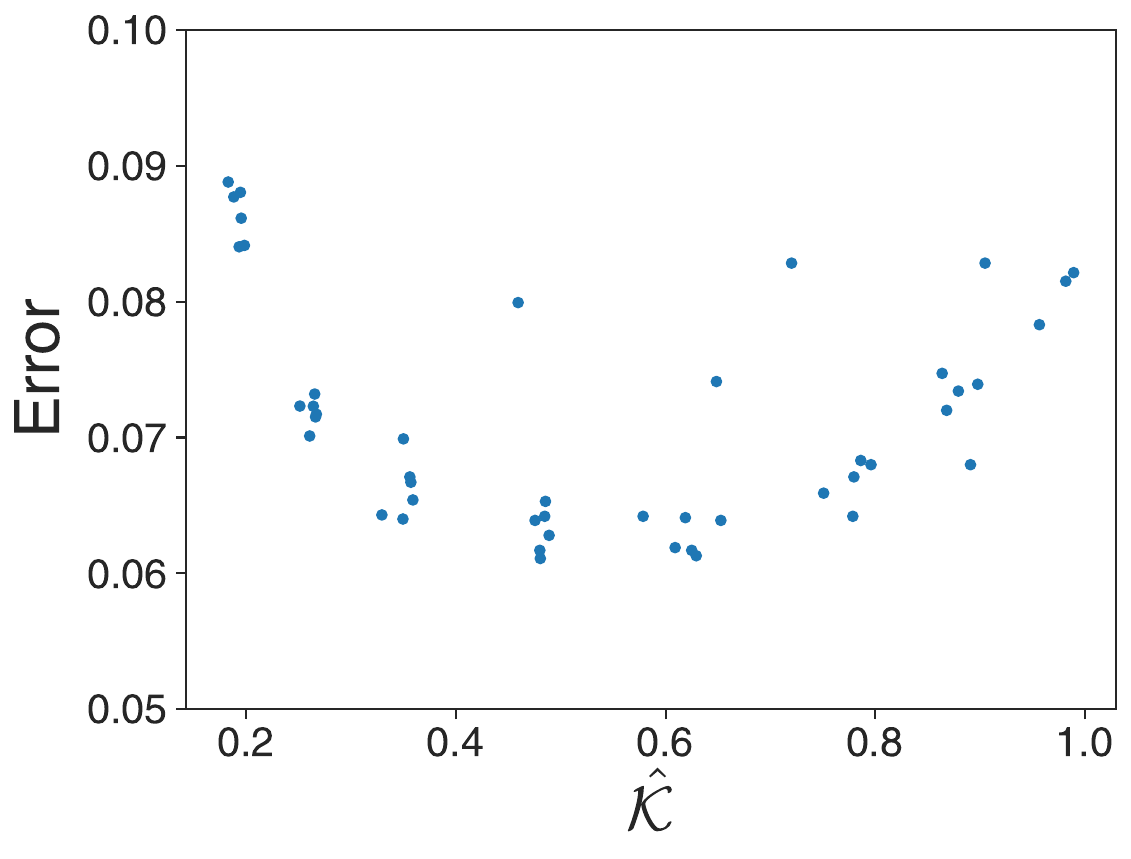}
    \end{tabular}
    \caption{$\knorm$ for ResNet-18 trained on CIFAR10 with momentum, cosine learning rate
    schedule, and $L^2$ regularization increases, remains flat at an $O(1)$ value, then
    decreases (left). Small base learning rate shows slower initial and late
    time error improvements, while large learning rate shows slow early time error improvements (middle). Best error is achieved for settings
    where median $\knorm$ remains within the interval $[0.4, 0.8]$ (right).
    }
    \label{fig:cifar_kern_norm}
\end{figure}


We trained over a variety of learning rates and batch sizes. We focus primarily on batch size
$128$ here; results for other batch sizes are similar (Appendix
\ref{app:cifar_details}). We found that the estimator $\hat{\knorm}_{mom}$ starts low, increases
dramatically at early times, levels off for much of training, and then decreases at late
times (Figure \ref{fig:cifar_kern_norm}, left). It remains $O(1)$ over a factor of $100$
variation of the base learning rate. The decrease at late times is primarily due to the
learning rate schedule; the unnormalized NTK trace is slowly increasing for most of
training (Appendix \ref{app:raw_ntk_tr}). The use of the NTK is key here;
the normalized Hessian trace has very different, non $O(1)$ dynamics (Appendix \ref{app:hess_trace}).

Even with learning rate schedules, $\knorm$ can be a useful tool to understand
aspects of learning dynamics.
At large learning rates where $\hat{\knorm}$ is near
$1$ for much of training, the test error decreases only slowly at that intermediate stage,
before dropping quickly at late times where the schedule pushes $\hat{\knorm}$ low
(Figure \ref{fig:cifar_kern_norm}, middle, red curve).
In contrast for a trajectory with low learning rate, the decrease is more smooth but still
slower overall (blue curve). The intermediate learning rates with lowest test error also
correspond to a median $\hat{\knorm}$ value in the range $[0.4, 0.8]$ (Figure \ref{fig:cifar_kern_norm}, right).
We repeated the experiments on an MLP-Mixer \texttt{S/16} architecture and found
similar results (Apppendix \ref{app:mlp-mixer-cifar10}).

\section{Discussion}

Our theoretical analysis and experiments suggest that there indeed is a stochastic edge of
stability, which can be derived simply at in the case of MSE loss. Non-linear models can
generate negative feedback to stabilize from above the S-EOS to below it; however, this
stabilization happens once on a long timescale, rather than the tight period $2$ quasi-stable
oscillations of the full batch EOS.

The approximate form of $\knorm$ in Equation \ref{eq:knorm_approx} scales as $\lr/\B$, which
is in accordance with both SDE-based analyses of SGD \citep{smith_don_2017, jastrzebski_three_2018},
as well as practical observations of the ``linear scaling rule'' regime where
scaling learning rate proportional to batch size achieves good performance
\citep{goyal_accurate_2018}. Our constant-epoch experiments on the MNIST example suggest that
there may be a link between the breakdown of the universal scaling regime of
$\knorm$ and the breakdown of the ``perfect scaling'' regime of steps to target
scaling as $\B^{-1}$ in constant epoch experiments \citep{shallue_measuring_2019}.

One advantage of the definition of $\knorm$ is the fact that it is scaled properly
independent of model and dataset size. Our experiments suggest that even in the
non-convex setting it is still meaningful. The full Hessian can suffer from sensitivity to
$L^{2}$ regularization and negative eigenvalues, and poor scaling with model size.
Our work naturally motivates the study of the NTK, which is often used to approximate
the loss Hessian in theoretical analyses \citep{wu_implicit_2023a}.

Another interesting result of our experiments is the observation that $\knorm$ can be a good 
predictor of training outcomes. Very small $\knorm$ ``wastes'' steps, while $\knorm$ close to
the S-EOS slows down \emph{all} eigenmodes and leads to poor optimization.
In a high dimensional convex setting this is the Malthusian exponent regime studied in \citet{paquette_sgd_2021}.
This is in contrast
to the full batch EOS where only one eigenmode converges slowly, leading to overall
good optimization. We hypothesize that these effects may be important in the compute
limited regimes where large models are often trained.

Both the definition of $\knorm$ and the analysis of conservative sharpening suggest that
in order to understand SGD dynamics, one must understand the \emph{distribution} of
NTK/Hessian eigenvalues. In fact our analysis of conservative sharpening suggests that the
distribution of \emph{model curvatures} is also crucial in understanding how the loss
landscape geometry evolves in SGD.



One key future direction is to extend some of the analyses to more general loss functions and
optimizers. Using local linearization of the loss function (Appendix \ref{app:non_mse_loss}) suggests that the Gauss-Newton trace may be a good estimator for non-MSE loss;
experiments on ResNet50 and ViT trained on Imagenet with cross-entropy loss show that this approximation captures some aspects of the dynamics but is quantitatively limited
(Appendix \ref{app:imagenet_experiments}).
A more sophisticated approach would be to adapt existing approaches to more general loss functions
to compute a better characterization of the EOS \citep{collins-woodfin_hitting_2023}.

Another extension is to develop algorithms that either control or use $\knorm$. Regularizing
the trace of the Gauss-Newton has been shown to have beneficial effects
\cite{dauphin_neglected_2024}, similar to the benefits of SAM at low batch size \cite{foret_sharpnessaware_2022, wen_how_2023}. A greater understanding of conservative
sharpening may lead to other ways to control SGD noise.

Maybe the most interesting direction is the prospect of using information about
$\knorm$ to dynamically choose step sizes. Traditional step size tuning methods often
fail dramatically in deep learning \citep{roulet_interplay_2023}, and some of that
failure may be due to not incorporating information relevant to SGD. This will require
further refining estimators of $\knorm$ or equivalents so the statistics can
be updated efficiently and frequently enough to be useful.


\bibliography{sgd_eos}
\bibliographystyle{unsrtnat}


\newpage

\newpage
\appendix
\onecolumn

\section{Stochastic edge of stability}

\subsection{Averaging lemma}

Here we prove a lemma which is used to take second moments with respect to SGD noise.
Recall that $\pmat_{t}$ is a sequence of i.i.d. random diagonal $\D\times\D$ matrices with
$\B$ $1$s and $\D-\B$ $0$s on the diagonal. We have the following lemma:

\begin{lemma}
Let $\m{M}$ be a matrix
independent of $\pmat_{t}$ for all $t$. Then we have the following moments:
\begin{equation}
\begin{split}
\expect[\pmat_{t}] & = \bfr\m{I},~\expect[\pmat_{t}\m{M}\pmat_{t+1}]  = \bfr^{2}\m{M} \\
\expect[\pmat_{t}\m{M}\pmat_{t}] & = \bfr\tl{\bfr}\m{M}+\bfr(1-\tl{\bfr})\diag(\m{M})
\end{split}
\end{equation}
where $\bfr\equiv \B/\D$ and $\tl{\bfr} \equiv (\B-1)/(\D-1)$.
\label{lem:ave_lemma}
\end{lemma}

\begin{proof}
The first moment of $\pmat_{t}$ is derived by averaging each diagonal term. Similarly,
$\expect[\pmat_{t}\m{M}\pmat_{t+1}] = \expect[\pmat_{t}]{\m{M}}\expect[\pmat_{t+1}]$
since $\pmat_{t}$ and $\pmat_{t+1}$ are independent.

Now consider $\expect[\pmat_{t}\m{M}\pmat_{t}]$. There are two cases to consider. First, consider the
diagonal of the output. For a coordinate $\al$ we have:
\begin{equation}
[\pmat_{t}\m{M}\pmat_{t}]_{\al\al} = [\pmat_{t}]_{\al\al}[\m{M}]_{\al\al}[\pmat_{t}]_{\al\al} = \begin{cases}
[\m{M}]_{\al\al} & {\rm with~probability~}\bfr\\
0 & {\rm with~probability~}(1-\bfr)
\end{cases}
\end{equation}
That is, the $\al\al$ diagonal element is non-zero precisely when the $\al\al$ diagonal element
of $\pmat$ is non-zero.

In the off-diagonal case, the $\al\bt$ element with $\al\neq \bt$ gives us:
\begin{equation}
[\pmat_{t}\m{M}\pmat_{t}]_{\al\bt} = [\pmat_{t}]_{\al\al}[\m{M}]_{\al\bt}[\pmat_{t}]_{\al\bt} = \begin{cases}
[\m{M}]_{\al\bt} & {\rm with~probability~}\binom{\B-2}{\D-2}/\binom{\B}{\D} \\
0 & {\rm with~probability~}1-\binom{\B-2}{\D-2}/\binom{\B}{\D}
\end{cases}
\end{equation}

Here the element is non-zero if and only if both $\al$ and $\bt$ are selected in the batch.

Taken
together, in coordinates we can write:
\begin{equation}
\expect[\pmat_{t}\m{M}\pmat_{t}]_{\al\bt} = \frac{\B}{\D}[\delta_{\al\bt}+(\B-1)/(\D-1)(1-\delta_{\al\bt})]\m{M}_{\al\bt}
\end{equation}
\begin{equation}
\expect[\pmat_{t}\m{M}\pmat_{t}]_{\al\bt} = \bfr[\delta_{\al\bt}+\tl{\bfr}(1-\delta_{\al\bt})]\m{M}_{\al\bt}
\end{equation}
Writing in matrix notation, we have the desired result.
\end{proof}

\subsection{Derivation of second moment dynamics}

\label{app:second_mom_dyn}

Here we derive the various dynamical equations for the second moment of $\z$ in the
linear model. We begin by noting that:
\begin{equation}
\expect_{\pmat}[\z_{t+1}\z_{t+1}^{\top}-\z_{t}\z_{t}^{\top}|\z_{t}] = \z_{t}\expect_{\pmat}[\z_{t+1}-\z_{t}|\z_{t}]^{\tpose}+\expect_{\pmat}[\z_{t+1}-\z_{t}|\z_{t}]\z_{t}^{\tpose}+\expect_{\pmat}[(\z_{t+1}-\z_{t})(\z_{t+1}-\z_{t})^{\tpose}|\z_{t}]
\end{equation}
Substitution gives us:
\begin{equation}
\begin{split}
\expect_{\pmat}[\z_{t+1}\z_{t+1}^{\top}-\z_{t}\z_{t}^{\top}|\z_{t}]  =  -\frac{\lr}{\B}\left(\z_{t}\expect_{\pmat}[\z_{t}^{\tpose}\pmat_{t}\J\J^{\tpose}]+\expect_{\pmat}[\J\J^{\tpose}\pmat_{t}\z_{t}]\z_{t}^{\tpose}\right)   +\frac{\lr^{2}}{\B^{2}}\expect_{\pmat}[\J\J^{\tpose}\pmat_{t}\z_{t}\z_{t}^{\top}\pmat_{t}\J\J^{\tpose}]
\end{split}
\end{equation}
Evaluation using Lemma \ref{lem:ave_lemma} gives us
\begin{equation}
\begin{split}
\expect_{\pmat}[\z_{t+1}\z_{t+1}^{\top}|\z_{t}]  = \z_{t}\z_{t}^{\top}-\lr \left(\ntk\z_{t}\z_{t}^{\top}+\z_{t}\z_{t}^{\top}\ntk\right)
+\tl{\bfr}\bfr^{-1}\lr^{2}\ntk\z_{t}\z_{t}^{\top}\ntk
+ \left(\bfr^{-1}-\tl{\bfr}\bfr^{-1}\right)\lr^{2} \ntk\diag\left[\z_{t}\z_{t}^{\top}\right]\ntk
\end{split}
\end{equation}
This means that $\expect_{\pmat}[\z_{t}\z_{t}^{\tpose}]$ evolves according to a linear dynamical
system. We denote the linear operator defining the dynamics as $\linop$.

We can rotate to the eigenbasis of the NTK. Given the eigendecomposition $\ntk = \V \lmat \V^{\tpose}$,
we define the matrix $\S_{t}$ as:
\begin{equation}
\S_{t} \equiv \V^{\tpose}\expect_{\pmat}[\z_{t}\z_{t}^{\top}]\V
\end{equation}
The diagonal elements of $\S_{t}$ correspond to the squared eigenprojections $\expect_{\pmat}[(\v_{\al}\cdot\z_{t})^{2}]$, while the off-diagonal elements correspond to
correlations $\expect_{\pmat}[(\v_{\al}\cdot\z_{t})(\v_{\bt}\cdot\z_{t})]$.

$\S_{t}$ also evolves linearly, according to the dynamical system:
\begin{equation}
\expect[\S_{t+1}|\z_{t}] = \S_{t}-\lr (\lmat\S_{t}+\S_{t}\lmat)+\tl{\bfr}\bfr^{-1}\lr^2 \lmat\S_{t}\lmat+(\bfr^{-1}-\tl{\bfr}\bfr^{-1})\lr^2\lmat \V^{\tpose}\left[\sum_{\al} (\V\S_{t}\V^{\tpose})_{\al\al} \m{e}_{\al}\m{e}_{\al}^{\tpose}\right]\V \lmat
\end{equation}
where $\m{e}_{\al}$ is the basis element for coordinate $\al$ in the original coordinate system.
The last term induces coupling in between the different elements of $\S_{t}$ - that is, between
the covariances of the different eigenmodes of $\ntk$. In coordinates we have:
\begin{equation}
[\lmat\V^{\tpose}\diag\left[\V\S_{t}\V^{\tpose}\right]\V\lmat]_{\mu\nu} = \lam_{\mu}\lam_{\nu}\left[\sum_{\al}\V_{\al\bt}\V_{\al\gm}\V_{\al\mu}\V_{\al\nu}\right](\S_{t})_{\bt\gm}
\end{equation}
That is, there is non-zero coupling between the residual dynamics in the eigendirections of
$\ntk$, and potentially non-trivial contributions from the covariances between different modes.
This is an effect entirely driven by SGD noise, as in the deterministic case the eigenmodes of
$\ntk$ evolve independently.

We can write the operator $\linop$ in the $\S$ basis, using a $4$-index notation:
\begin{equation}
\linop_{\mu\nu,\bt\gm} = \delta_{\mu\bt,\nu\gm}(1-\lr(\lam_{\mu}+\lam_{\nu})+\tl{\bfr}\bfr^{-1}\lr^{2}\lam_{\mu}\lam_{\nu})+(\bfr^{-1}-\tl{\bfr}\bfr^{-1})\lr^{2}\lam_{\mu}\lam_{\nu}\left[\sum_{\al}\V_{\al\bt}\V_{\al\gm}\V_{\al\mu}\V_{\al\nu}\right]
\end{equation}
In this notation, $(\S_{t+1})_{\mu\nu} = \sum_{\bt\gm}\linop_{\mu\nu,\bt\gm}(\S_{t})_{\bt\gm}$.

In the main text, we analyzed the dynamics restricted to the diagonal of $\S$.
Let $\pvec \equiv\diag(\S)$. The dynamical equation is, coordinate-wise:
\begin{equation}
(\pvec_{t+1})_{\mu} = \sum_{\bt} \linop_{\mu\mu,\bt\bt}(\pvec_{t})_{\bt}
\end{equation}
which becomes, in matrix notation
\begin{equation}
\pvec_{t+1} = \dmat\pvec_{t},~\dmat \equiv [(\Id-\lr\lmat)^2 +(\tl{\bfr}\bfr^{-1}-1)\lr^{2}\lmat^{2}+\lr^2(\bfr^{-1}-\tl{\bfr}\bfr^{-1})\lmat^2\cmat],~\cmat_{\bt\mu} \equiv \sum_{\al}\V_{\al\bt}^2\V_{\al\mu}^2
\end{equation}
Note that $\cmat$ is a PSD (and indeed, non-negative) matrix. If $\lmat$ is invertible,
$\dmat$ has all real non-negative eigenvalues, as seen via similarity transformation
(left multiply by $\lmat^{-1}$, right multiply by $\lmat$). In the general case,
if we define $\tpvec = \lmat^{+}\pvec$ (transformation by the Moore-Penrose pseudoinverse
of $\lmat$), we have:
\begin{equation}
\tpvec_{t+1} = [(\Id-\lr\lmat)^2 +(\tl{\bfr}\bfr^{-1}-1)\lr^{2}\lmat^{2}+\lr^2(\bfr^{-1}-1)\lmat\cmat\lmat]\tpvec_{t}
\end{equation}
This leads us directly to the decomposition in Equation \ref{eq:pvec_volt}.
\aga{This might be slightly wrong - correct in later versions. The nullspace of $\lmat$
can have some action here but I think it shouldn't affect stability conditions - I think it's
just sets a different equilibrium point.}

\subsection{Proof of Theorem \ref{thm:s_eos}}

\label{app:s_eos_def}

We will use the following lemmas:
\begin{lemma}
Let $a$ and $b$ be random variables with finite first and second moment. Then
$\expect[|ab|]\leq\expect[a^2]+\expect[b^2]$.
\label{lem:prod_lemma}
\end{lemma}
\begin{proof}
Given any fixed $a$ and $b$, $|ab|\leq a^2+b^2$. From the linearity of expectation
we have the desired result.
\end{proof}
\begin{lemma}
Let $\amat$ and $\bmat$ be two PSD matrices. Then
\begin{equation}
\max\lam[\amat]\leq\max\lam[\amat+\bmat]
\end{equation}
\label{lem:psd_inequality}
\end{lemma}
\begin{proof}
Let $\v$ be an eigenvector of $\amat$ associated with the largest eigenvalue, with
length $1$. Then we have:
\begin{equation}
\v^{\top}[\amat+\bmat]\v = \max\lam[\amat]+\v^{\top}\bmat\v\geq \max\lam[\amat]
\end{equation}
where the final inequality comes from the PSDness of $\bmat$. Note that $\amat+\bmat$ is PSD
since $\amat$ and $\bmat$ are individually. Therefore, we have
\begin{equation}
\v^{\top}[\amat+\bmat]\v = \sum_{k}(\v\cdot\w_{k})^{2}\lam_{k}
\end{equation}
where $\w_{k}$ is the eigenvector of $\amat+\bmat$ associated with the eigenvalue
$\lam_{k}$. Since the $\lam_{k}$ are non-negative, and the $(\v\cdot\w_{k})^{2}$ are
non-negative and sum to $1$, we have
\begin{equation}
\v^{\top}[\amat+\bmat]\v \leq \max\lam[\amat+\bmat]
\end{equation}
Combining all our inequalities, we have:
\begin{equation}
\max\lam[\amat]\leq\max\lam[\amat+\bmat]
\end{equation}
\end{proof}
\begin{lemma}
Let $\amat$ and $\bmat$ be PSD matrices. Then the product $\amat\bmat$ has non-negative
eigenvalues.
\label{lem:psd_prod}
\end{lemma}
\begin{proof}
Consider the symmetric matrix  $\m{M} = (\bmat)^{1/2}\amat\bmat^{1/2}$. This matrix is
PSD since
\begin{equation}
\w^{\tpose}(\bmat)^{1/2}\amat \bmat^{1/2}\w = [(\bmat)^{1/2}\w]^{\tpose}\amat [\bmat^{1/2}\w]\geq 0
\end{equation}
for any $\w$, by the PSDness of $\amat$. Let $\v$ be an eigenvector of
$\bmat^{1/2}\amat \bmat^{1/2}$ associated with eigenvalue $\lam$. We consider
two cases. The first is that $\bmat^{1/2}\v = 0$. In this case, $\amat\bmat\v = 0$,
and $\v$ is an eigenvector of eigenvalue $0$ for $\amat\bmat$ as well.

Now we consider non-zero eigenvalues of $\m{M}$. WLOG we choose a basis such that
the eigenvalue condition for positive $\lam$ can be written as
\begin{equation}
\m{M}\begin{pmatrix}
\v\\
0
\end{pmatrix} = \begin{pmatrix}
\m{L} & 0\\
0 & 0
\end{pmatrix}\begin{pmatrix}
\amat_{11} & \amat_{12}\\
\amat_{21} & \amat_{22}
\end{pmatrix}
\begin{pmatrix}
\m{L} & 0\\
0 & 0
\end{pmatrix}\begin{pmatrix}
\v\\
0
\end{pmatrix} = \lam \begin{pmatrix}
\v\\
0
\end{pmatrix}
\end{equation}
where $\m{L}$ is a positive diagonal matrix. Now consider the following product
involving $\amat\bmat$:
\begin{equation}
\amat\bmat \begin{pmatrix}
\m{L}^{-1}\v\\
\u
\end{pmatrix} =
\begin{pmatrix}
\amat_{11} & \amat_{12}\\
\amat_{21} & \amat_{22}
\end{pmatrix}
\begin{pmatrix}
\m{L}^{2} & 0\\
0 & 0
\end{pmatrix}\begin{pmatrix}
\m{L}^{-1}\v\\
\u
\end{pmatrix}
\end{equation}
We can rewrite this as
\begin{equation}
\amat\bmat \begin{pmatrix}
\m{L}^{-1}\v\\
\u
\end{pmatrix} = \begin{pmatrix}
\m{L}^{-1} & 0\\
0 & \Id
\end{pmatrix} \begin{pmatrix}
\m{L} & 0\\
0 & \Id
\end{pmatrix}\begin{pmatrix}
\amat_{11} & \amat_{12}\\
\amat_{21} & \amat_{22}
\end{pmatrix}
\begin{pmatrix}
\m{L}^{2} & 0\\
0 & 0
\end{pmatrix}\begin{pmatrix}
\v\\
\u
\end{pmatrix}
\end{equation}
Using the eigenvalue condition we have:
\begin{equation}
\amat\bmat \begin{pmatrix}
\m{L}^{-1}\v\\
\u
\end{pmatrix} = \begin{pmatrix}
\lam\m{L}^{-1}\v\\
\amat_{21}\v
\end{pmatrix}
\end{equation}
If we select $\u = \lam^{-1}\amat_{21}\v$, then we have
\begin{equation}
\amat\bmat \begin{pmatrix}
\m{L}^{-1}\v\\
\u
\end{pmatrix} = \lam\begin{pmatrix}
\m{L}^{-1}\v\\
\u
\end{pmatrix}
\end{equation}
Therefore $\lam$ is an eigenvalue of $\amat\bmat$. All eigenvalues of $\amat\bmat$ are
non-negative.
\end{proof}

Lemmas in hand, we can now prove the theorem. A key point is that the theorem would be
trivial if $\amat$ and $\bmat$ were scalars; in this case, it would be equivalent to
$A<1$, $A+B < 1$ if and only if $(1-A)^{-1}B<1$. We will use the PSD nature of
$\amat$ and $\bmat$ to extend the trivial manipulation of scalar inequalities to their
linear algebraic counterparts in terms of the largest eigenvalues of the corresponding
matrices.

\paragraph{Theorem \ref{thm:s_eos}}
Given the dynamics of Equation \ref{eq:pvec_volt}, $\lim_{t\to\infty}\expect_{\pmat}[\z_{t}\z_{t}^{\top}] = 0$ for any initialization $\z_{t}$
if and only if $||\amat||_{op} <1$ and $\knorm < 1$ where
\begin{equation}
\knorm \equiv \max\lam\left[(\Id-\amat)^{-1}\bmat\right]
\end{equation}
for the PSD matrices $\amat$ and $\bmat$ defined above. $\knorm$ is always non-negative.
\begin{proof}
We begin with Equation \ref{eq:pvec_volt}. This is a linear dynamical system which
determines the values of $\expect_{\pmat}[\diag(\z_{t}\z_{t}^{\top})]$. From
Lemma \ref{lem:prod_lemma}, $\lim_{t\to\infty}\expect_{\pmat}[\diag(\z_{t}\z_{t}^{\top})]$
implies $\lim_{t\to\infty}\expect_{\pmat}[\z_{t}\z_{t}^{\top}]$ for off-diagonal elements as
well.

The linear system converges to $0$ for all inputs if and only if $L_{max}$, the largest
eigenvalue of $\amat+\bmat$, has absolute value less than $1$. Since $\amat$ and $\bmat$
are both PSD, this condition is equivalent to $L_{max} <1$. From Lemma \ref{lem:psd_inequality}
we have:
\begin{equation}
||\amat||_{op} = \max\lam[\amat] \leq \max\lam[\amat+\bmat]
\end{equation}
Therefore, if $||\amat||_{op} \geq 1$, $L_{max} \geq 1$ and the dynamics does
not converge to $0$.

Now consider the case $||\amat||_{op} < 1$.
We first show that
$\max\lam[(\Id-\amat)^{-1}\bmat]\geq 1$ implies $\max\lam[\amat+\bmat] \geq 1$.
Since $||\amat||_{op}<1$, $\Id-\amat$ is invertible. Let $\w$ be an eigenvector
of $(\Id-\amat)^{-1}\bmat$ with eigenvalue $\omega\geq 1$. Then:
\begin{equation}
\w^{\tpose}\bmat\w = \w^{\tpose}(\Id-\amat)(\Id-\amat)^{-1}\bmat\w = \omega \w^{\tpose}(\Id-\amat)\w
\end{equation}
This implies that
\begin{equation}
\w^{\tpose}[\amat+\bmat]\w = \omega \w^{\tpose}\Id\w +(1-\omega)\w^{\tpose}\amat\w
\end{equation}
Since $||\amat||_{op}<1$, $(1-\omega)\w^{\tpose}\amat\w \geq 1-\omega$ and we have
\begin{equation}
\w^{\tpose}[\amat+\bmat]\w \geq \omega +(1-\omega) = 1
\end{equation}
Therefore, $\max\lam[\amat+\bmat] \geq 1$ and
$\lim_{t\to\infty}\expect_{t}[\z_{t}\z_{t}^{\tpose}]\neq 0$ for all initializations.

Now we show the converse. Suppose $\max\lam[\amat+\bmat] \geq 1$. Let $\u$ be an eigenvector
of $\amat+\bmat$ with eigenvalue $\nu > 1$. We note that the symmetric matrix $(\Id-\amat)^{-1/2}\bmat(\Id-\amat)^{-1/2}$ has the same spectrum as $(\Id-\amat)^{-1}\bmat$.
Let $\tl{\u} \equiv (\Id-\amat)^{1/2}\u$. We have:
\begin{equation}
\frac{\tl{\u}^{\tpose} (\Id-\amat)^{-1/2}\bmat(\Id-\amat)^{-1/2}\tl{\u}}{\tl{\u}^{\tpose}\tl{\u}} = \frac{\u^{\tpose}\bmat\u}{\u^{\tpose}(\Id-\amat)\u} = \frac{\u^{\tpose}(\nu\Id-\amat)\u}{\u^{\tpose}(\Id-\amat)\u} = 1+\frac{\nu-1}{\u^{\tpose}(\Id-\amat)\u}
\end{equation}
Since $\nu>1$ and $\Id-\amat$ is PSD and invertible, $\u^{\tpose}(\Id-\amat)\u > 0$.
Therefore, the expression is greater than $0$. This means that $\max\lam[(\Id-\amat)^{-1/2}\bmat(\Id-\amat)^{-1/2}]\geq 1$, and accordingly
$\max\lam[(\Id-\amat)^{-1}\bmat] \geq 1$

Note that $\max\lam[(\Id-\amat)^{-1}\bmat]$ is always non-negative by Lemma \ref{lem:psd_prod}. This concludes the
proof.
\end{proof}

\subsection{Validity of $\knorm$ and approximations}

\label{app:knorm_highd_approx}

The analysis of \citet{paquette_sgd_2021} established the following approximation for
$\knorm$:
\begin{equation}
\knorm\approx \hat{\knorm}_{HD} = \frac{\lr}{\B}\sum_{\al=1}^{\D} \frac{\lam_{\al}}{2-\lr\lam_{\al}}
\end{equation}
This approximation holds in the limit of large $\D$, with sufficiently smooth convergence of
the spectrum of $\frac{1}{\D}\J\J^{\tpose}$ to its limiting distribution, and a rotational
invariance assumption on the distribution of eigenvectors in the limit. For $\lr\lam\ll 2$,
there is an even simpler approximator:
\begin{equation}
\knorm\approx \hat{\knorm}_{tr} = \frac{\lr}{\B}\tr[\ntk]
\end{equation}

We can compare the approximations to $\knorm$ in different settings, and in turn compare
$\knorm$ to the exact $\max||\lam[\linop]||$. We performed numerical experiments in
$3$ settings ($\D = 100$, $\P = 120$, $\B = 5$):
\begin{itemize}
    \item \textbf{Flat spectrum.} Here $\J$ was chosen to have i.i.d. elements, and the resulting
    spectrum limits to Marchenko-Pastur in the high dimensional limit. This is the setting where
    $\knorm$ and its approximations best capture $\max||\lam[\linop]||$.
    \item \textbf{Dispersed spectrum.} Here we chose a spectrum $\lam_{\al} = 1/(\al^{2}+1)$ for
    the NTK, where the eigenvectors of $\ntk$ were chosen from a rotationally invariant distribution.
    This causes $\hat{\knorm}_{tr}$ to differ from $\hat{\knorm}_{HD}$ and $\hat{\knorm}_{HD}$
    differs from $\knorm$, but $\knorm$ still approximates $\max||\lam[\linop]||$.
    \item \textbf{Localized eigenvectors.} Here $\ntk = \diag(|\m{s}|)+\frac{1}{\D}\J_{0}\J_{0}^{\tpose}$ for a vector $\m{s}$ drawn i.i.d. from
    a Gaussian with $\sigma = 0.1$, and $\J_{0}$ from an i.i.d. Gaussian. This causes $\ntk$ to have additional weight
    on the diagonal, and causes the eigenvectors to delocalize in the coordinate basis. This is
    the most ``adversarial'' setup for the approximation scheme, and $\knorm$ no longer
    predicts $\max||\lam[\linop]||$ to high accuracy.
\end{itemize}

We can see the various stability measures as a function of $\lr$
in Figure \ref{fig:knorm_stability_measures}. As previously explained, $\max||\lam[\linop]||$ takes
a value close to $1$ for small learning rates, until the S-EOS is reached and it rises above $1$.
In contrast, $\knorm$ and its approximators start at $0$ for small learning rate and approach $1$
monotonically from below - by design. In all cases, the maximum eigenvalue is well below the
edge of stability value of $2/\lr$ (purple curve), so any instability is due to the S-EOS.

In the flat spectrum case (Figure \ref{fig:knorm_stability_measures}, left), $\knorm$ and its
approximators all give good predictions of the S-EOS - or equivalently, the region
of learning rates where $\max||\lam[\linop]||>1$. In the dispersed spectrum setting
(Figure \ref{fig:knorm_stability_measures}, middle), the differences between the approximations
are more apparent. However, $\knorm$ still predicts the S-EOS.

Finally, in the localized eigenvectors case, even $\knorm$ is a bad approximator of the S-EOS
(Figure \ref{fig:knorm_stability_measures}, right). The dynamics becomes unstable for values of
$\knorm$ well below $1$. It is not surprising that $\knorm$ does not capture the behavior of
$\max||\lam[\linop]||$ here. From the high dimensional analysis, we know that the effect of the noise
term is to evenly couple the different eigenmodes of $\ntk$; this is possible because the
eigenbasis of $\ntk$ has no correlation with the coordinate eigenbasis. Having eigenvectors
correlated with the coordinate basis breaks this property and leads to the approximations
leading to $\knorm$ to become bad.

\begin{figure*}[h]
    \centering
    \includegraphics[width=0.42\linewidth]{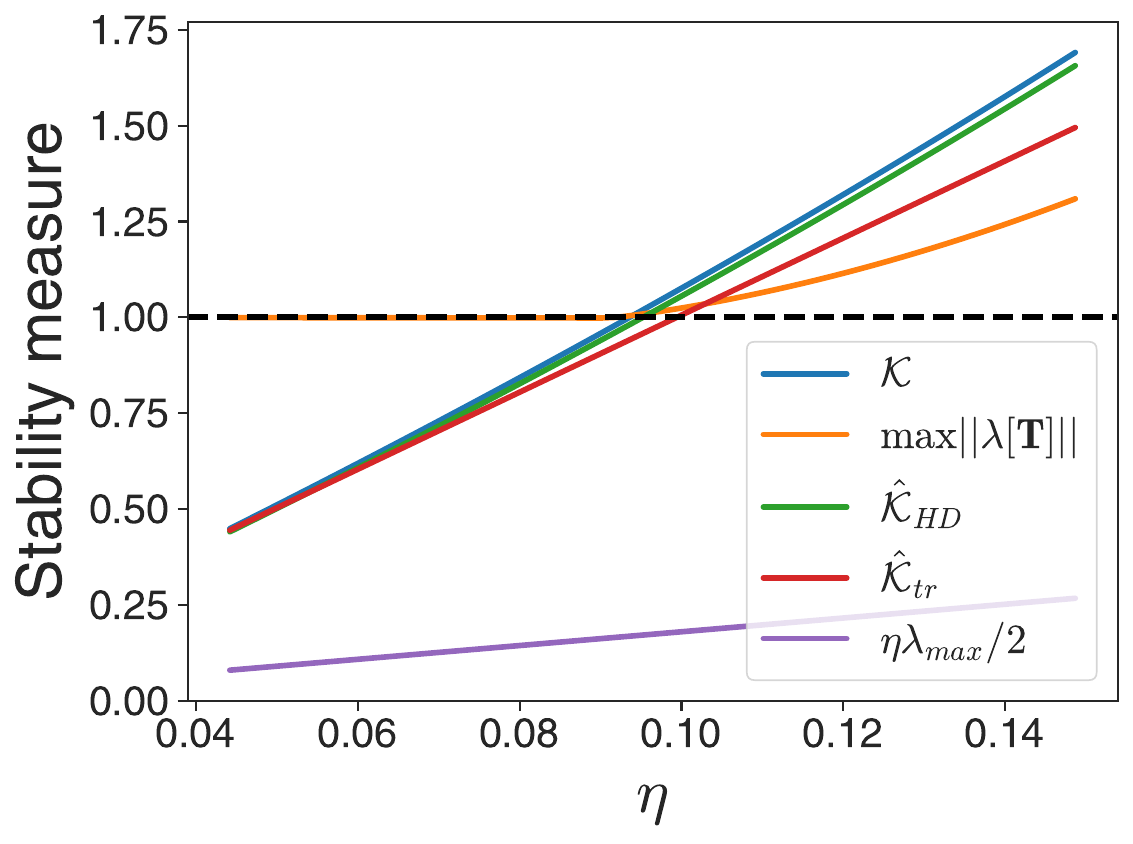}
    \begin{tabular}{cc}
 \includegraphics[width=0.42\linewidth]{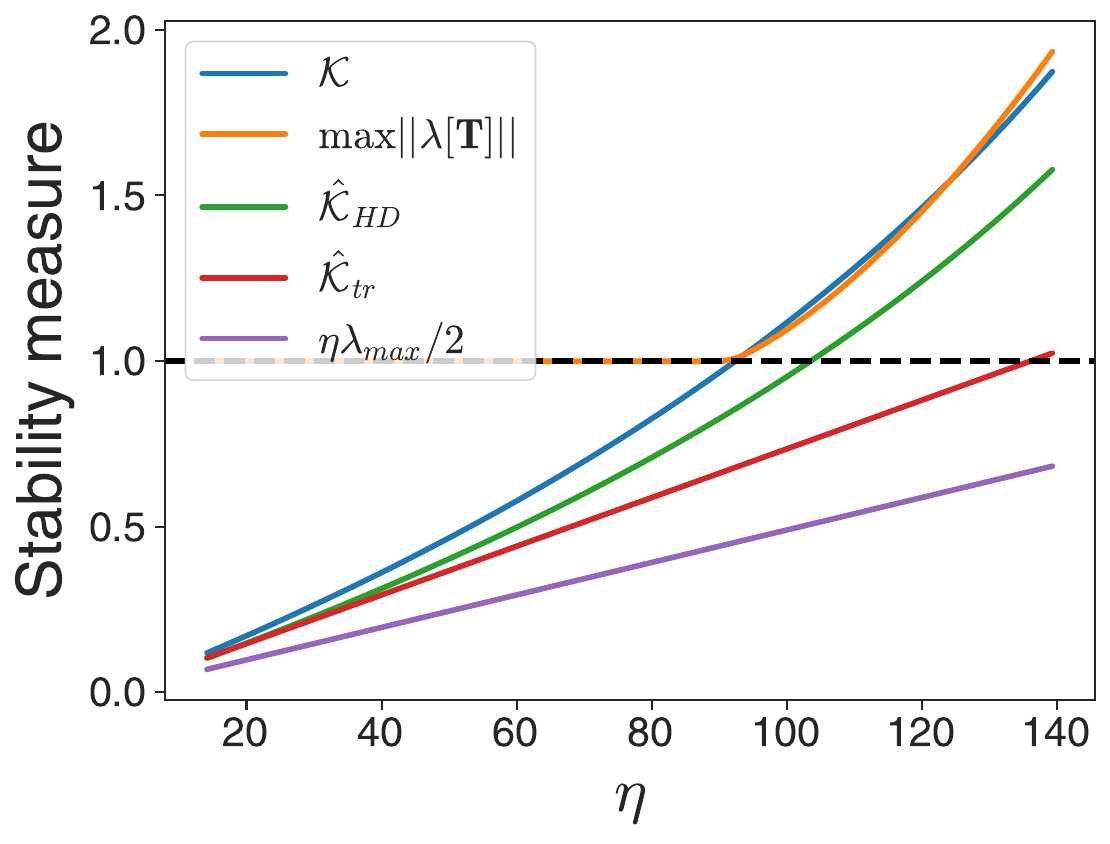} &
\includegraphics[width=0.42\linewidth]{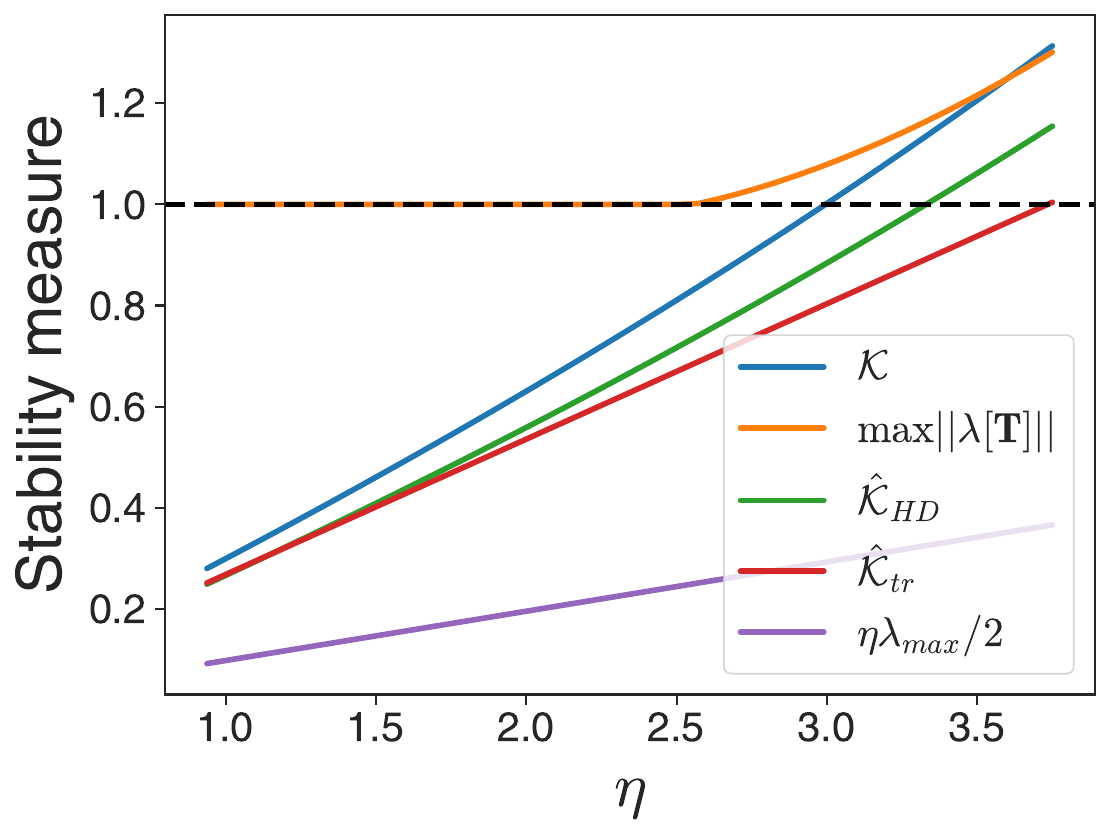}
\end{tabular}
    \caption{Stability measures for SGD in linear model with $\D = 100$, $\P = 120$, and $\B = 5$. For i.i.d. initialization of $\J$, NTK spectrum is not very varied and approximations are close to $\knorm$ (top). However, in the case of dispersed spectra ($\lam_{\al} = 1/(\al^{2}+1)$, bottom left), and localized eigenvectors (NTK $\diag(|\m{s}|)+\frac{1}{\D}\J_{0}\J_{0}^{\tpose}$, $\m{s}$ and $\J_{0}$ i.i.d, bottom right)
    approximations are less accurate. In all cases, maximum eigenvalue is well below stability
    threshold (red curves).}
    \label{fig:knorm_stability_measures}
\end{figure*}

The differences between the setups can be made even more clear by looking at the loss at late
times, as a function of the stability measures (Figure \ref{fig:knorm_loss_detailed}, for $10^{4}$
steps). We see that $\max||\lam[\linop]|| = 1$ predicts the transition from convergent to
divergent well in all settings, $\knorm = 1$ predicts it well in all but the localized eigenvectors
setting, and and $\hat{\knorm}_{HD}$ alread starts to become inaccurate in the dispersed setting.

\begin{figure*}[h]
    \centering
    \begin{tabular}{ccc}
\includegraphics[width=0.27\linewidth]{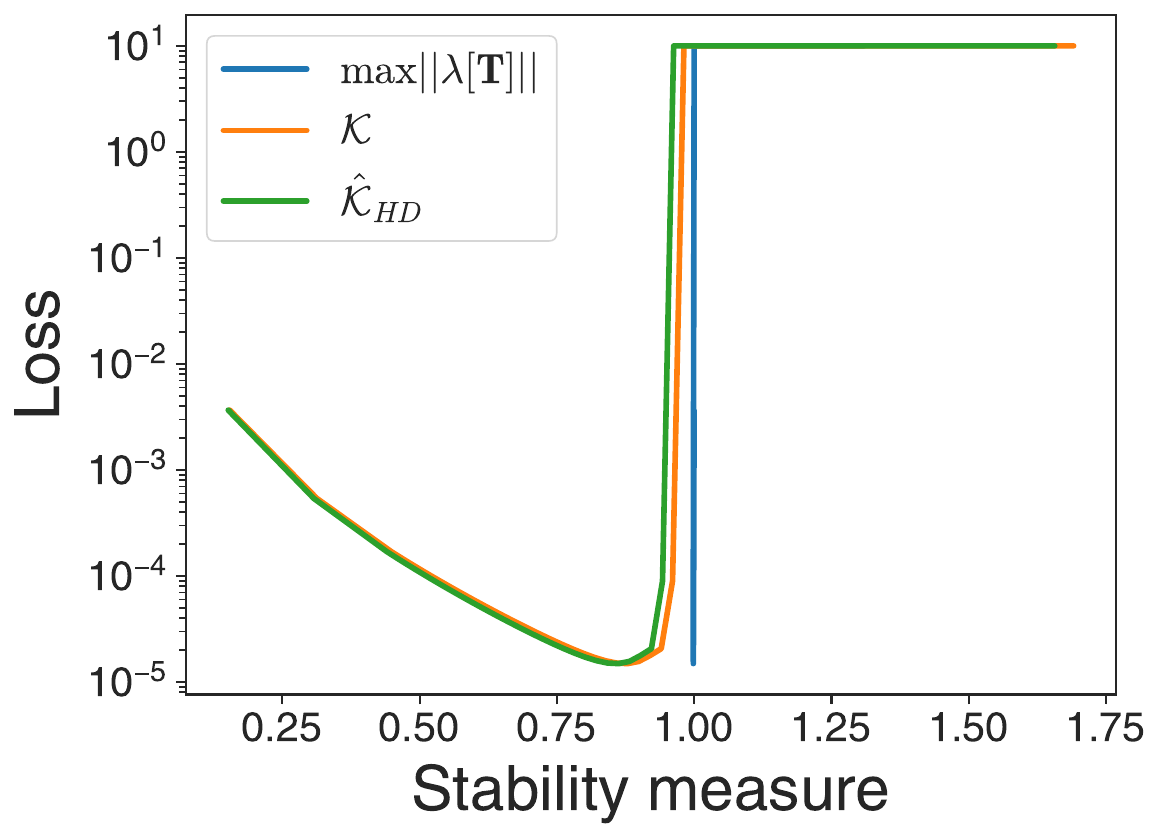} & \includegraphics[width=0.27\linewidth]{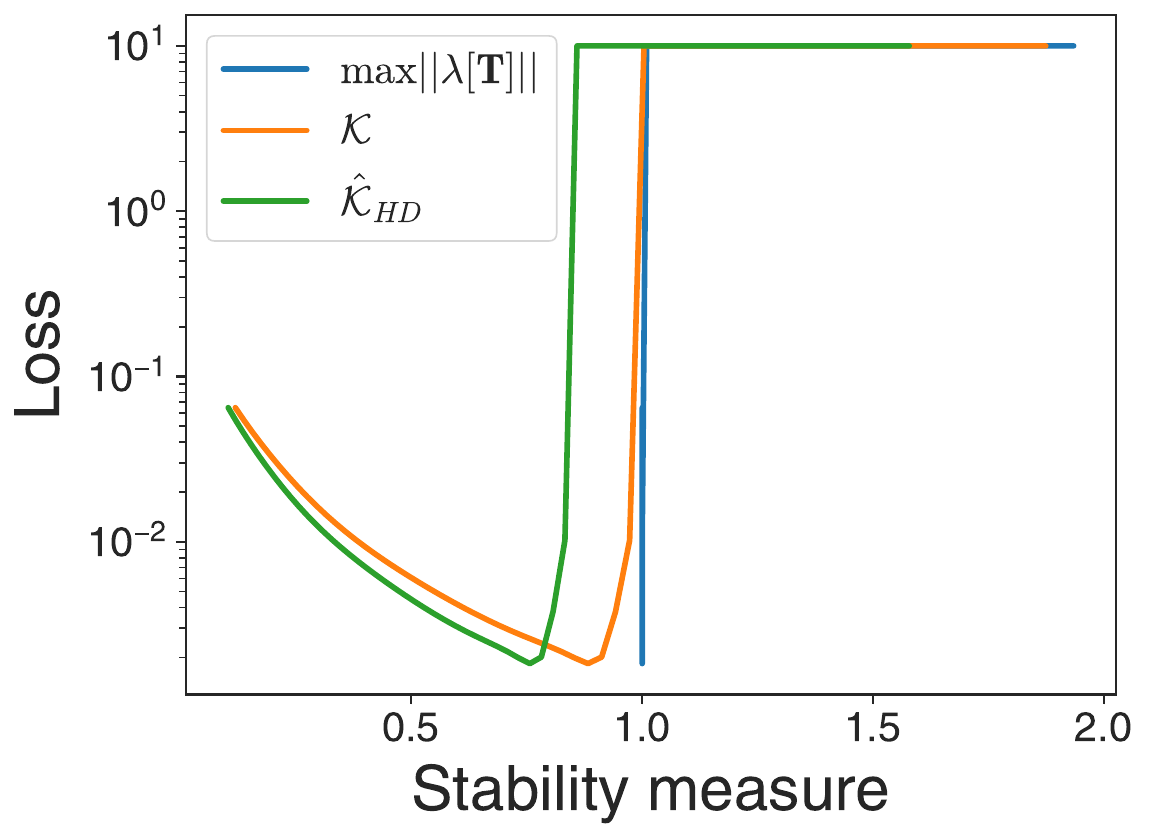} &
\includegraphics[width=0.27\linewidth]{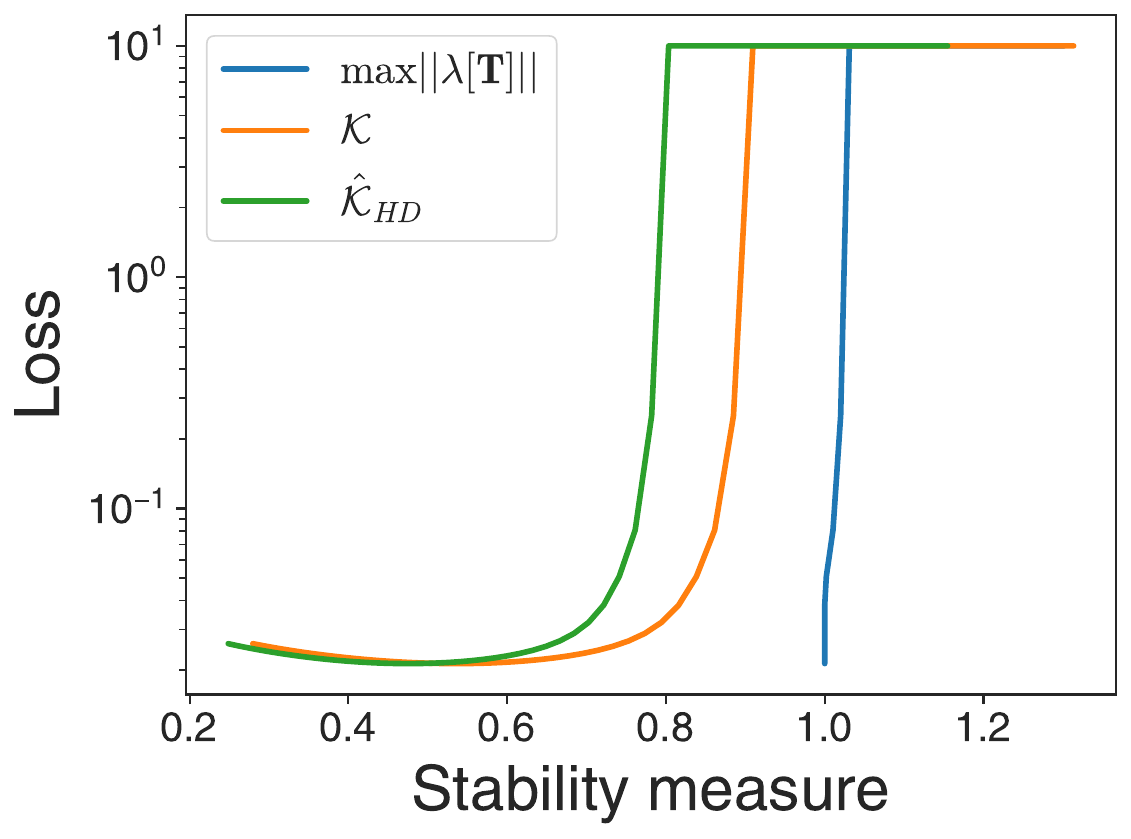}
\end{tabular}
    \caption{Loss versus stability measures after $10^{4}$ steps, for flat spectrum, dispersed
    spectrum, and localized eigenvector settings. All curves saturated at loss value $10$
    for ease of plotting. With a flat spectrum (left), all three of
    $\max||\lam[\linop]||$, $\knorm$, and $\hat{\knorm}_{HD}$ predict divergence of loss
    at the critical value of $1$. For dispersed spectrum, $\knorm$ is still a good approximator
    of the convergent regime but $\hat{\knorm}_{HD}$ is less so. For localized eigenvector setting,
    only $\max||\lam[\linop]||$ predicts the transition.
    }
    \label{fig:knorm_loss_detailed}
\end{figure*}

This analysis suggests that $\hat{\knorm}_{HD}$ and $\hat{\knorm}_{tr}$ are conservative
estimators of the noise level, but that $\knorm$ itself is a good estimator of the
S-EOS as long as the eigenvectors of $\ntk$ remain delocalized.
The approximations tend to get better in high dimensions,
but even in low dimensions they still provide valuable information on parameter ranges where
the optimization enters the noise-dominated regime.

\subsection{$\knorm$ and momentum}

\label{app:s_eos_mom}

In this section, we analyze the noise kernel norm with momentum.

In the high-dimensional isotropic case, we can compute $\knorm$ for SGD with
momentum. Consider
momentum with parameter $\mom$, where the updates evolve as:
\begin{equation}
\v_{t+1} = \mom\v_{t}+\g_{t}
\end{equation}
\begin{equation}
\th_{t+1} = \th_{t}+\lr \v_{t}
\end{equation}
for gradient $\g$. In a linear model,
\begin{equation}
\g_{t} = -\J\J^{\tpose}\pmat_{t}\z_{t}
\end{equation}

As per the analysis of \citet{paquette_sgd_2021}, the second moment dynamics
of $\z$ close once again.
In the high dimensional limit, where $\cmat = \frac{1}{\D}\m{1}\m{1}^{\tpose}$, we get the
noise kernel norm given by:
\begin{equation}
\knorm = \bfr(1-\bfr)\frac{1}{\D}\sum_{t=0}^{\infty}\sum_{\al=1}^{\D}\frac{2\lr^{2}\lam_{\al}^{2}}{\Om^{2}_{\al}-4\mom}\left(\mom^{t+1}+\frac{1}{2}\nu_{+,\al}^{t+1}+\frac{1}{2}\nu_{-,\al}^{t+1}\right)
\end{equation}
where $\Om_{\al} = 1-\bfr\lr\lam_{\al}+\mom$ and
\begin{equation}
\nu_{\al,\pm} = \frac{-2\mom+\Om_{\al}^{2}\pm\sqrt{\Om_{\al}^{2}(\Om_{\al}^{2}-4\mom)}}{2}
\end{equation}
We can simplify this expression considerably. Carrying out the sum over $t$ we have:
\begin{equation}
\knorm = \bfr(1-\bfr)\frac{1}{\D}\sum_{\al=1}^{\D}\frac{2\lr^{2}\lam_{\al}^{2}}{\Om^{2}_{\al}-4\mom}\left(\frac{\mom}{1-\mom}+\frac{1}{2}\left[\frac{\nu_{+,\al}}{1-\nu_{+, \al}}+\frac{\nu_{-,\al}}{1-\nu_{-, \al}}\right]\right)
\end{equation}
If we write $\nu_{\al, \pm} = \frac{1}{2}(a\pm b)$, we have:
\begin{equation}
\frac{\nu_{+,\al}}{1-\nu_{+, \al}}+\frac{\nu_{-,\al}}{1-\nu_{-, \al}} = \frac{a+b}{2-(a+b)}+\frac{a-b}{2-(a-b)} = \frac{(a-b)(2-(a+b))+(a+b)(2-(a-b))}{4-4a+(a^{2}-b^{2})}
\end{equation}
\begin{equation}
\frac{\nu_{+,\al}}{1-\nu_{+, \al}}+\frac{\nu_{-,\al}}{1-\nu_{-, \al}} = \frac{4a -2(a^{2}-b^{2})}{4-4a+(a^{2}-b^{2})}
\end{equation}
We have:
\begin{equation}
a^{2}-b^{2} = (-2\mom+\Om_{\al}^{2})^{2}-(\Om_{\al}^{2}(\Om_{\al}^{2}-4\mom)) = 4\mom^{2}
\end{equation}
Simplification gives us
\begin{equation}
\frac{\nu_{+,\al}}{1-\nu_{+, \al}}+\frac{\nu_{-,\al}}{1-\nu_{-, \al}} = \frac{4(\Om^{2}_{\al}-2\mom) -8\mom^{2}}{4-4(\Om^{2}_{\al}-2\mom)+4\mom^{2}} = \frac{(\Om^{2}_{\al}-2\mom) -2\mom^{2}}{1-(\Om^{2}_{\al}-2\mom)+\mom^{2}}
\end{equation}
\begin{equation}
\frac{\nu_{+,\al}}{1-\nu_{+, \al}}+\frac{\nu_{-,\al}}{1-\nu_{-, \al}} =  \frac{\Om^{2}_{\al}-2\mom -2\mom^{2}}{-\Om^{2}_{\al}+(1+\mom)^{2}}
\end{equation}
Substituting $\Om_{\al} = 1-\bfr\lr\lam_{\al}+\mom$, we have:
\begin{equation}
\frac{\nu_{+,\al}}{1-\nu_{+, \al}}+\frac{\nu_{-,\al}}{1-\nu_{-, \al}} =  \frac{(1-\bfr\lr\lam_{\al}+\mom)^{2}-2\mom -2\mom^{2}}{-(1-\bfr\lr\lam_{\al}+\mom)^{2}+(1+\mom)^{2}} = \frac{(1-\bfr\lr\lam_{\al})^{2}-2\mom\bfr\lr\lam_{\al}-\mom^{2}}{2(1+\mom)\bfr\lr\lam_{\al}-(\bfr\lr\lam_{\al})^{2}}
\end{equation}
The denominator of $\knorm$ can be written as
\begin{equation}
\Om_{\al}^{2}-4\mom = (1-\bfr\lr\lam_{\al}+\mom)^{2}-4\mom = (1-\mom)^{2}-2(1+\mom)\bfr\lr\lam_{\al}+(\bfr\lr\lam_{\al})^{2}
\end{equation}
Therefore we can re-write the noise kernel norm as:
\begin{equation}
\knorm = \frac{\bfr(1-\bfr)}{2\D}\sum_{\al=1}^{\D}\frac{2\lr^{2}\lam_{\al}^{2}}{(1-\mom)^{2}-2(1+\mom)\bfr\lr\lam_{\al}+(\bfr\lr\lam_{\al})^{2}}\left(-\frac{2\mom}{1-\mom}+\frac{(1-\bfr\lr\lam_{\al})^{2}-2\mom\bfr\lr\lam_{\al}-\mom^{2}}{2(1+\mom)\bfr\lr\lam_{\al}-(\bfr\lr\lam_{\al})^{2}}\right)
\end{equation}

As a sanity check, for $\mom = 0$ we have $\Om_{\al} = 1-\bfr\lr\lam_{\al}$ and
\begin{equation}
\frac{\nu_{+,\al}}{1-\nu_{+, \al}}+\frac{\nu_{-,\al}}{1-\nu_{-, \al}} = \frac{(1-\bfr\lr\lam_{\al})^{2}}{2\bfr\lr\lam_{\al}-(\bfr\lr\lam_{\al})^{2}}
\end{equation}
which leads to
\begin{equation}
\knorm = \frac{\bfr(1-\bfr)}{2\D}\sum_{\al=1}^{\D}\frac{2\lr^{2}\lam_{\al}^{2}}{(1-\bfr\lr\lam_{\al})^{2}}\frac{(1-\bfr\lr\lam_{\al})^{2}}{2\bfr\lr\lam_{\al}-(\bfr\lr\lam_{\al})^{2}} = (1-\bfr)\frac{1}{\D}\sum_{\al=1}^{\D}\frac{\lr\lam_{\al}}{2-\bfr\lr\lam_{\al}}
\end{equation}
as before.

If we re-write the momentum as $\mom = 1-\momd$, we have:
\begin{equation}
\frac{\nu_{+,\al}}{1-\nu_{+, \al}}+\frac{\nu_{-,\al}}{1-\nu_{-, \al}} =   \frac{(1-\bfr\lr\lam_{\al})^{2}-2\bfr\lr\lam_{\al}+2\momd\bfr\lr\lam_{\al}-1+2\momd-\momd^{2}}{{2(2-\momd)\bfr\lr\lam_{\al}-(\bfr\lr\lam_{\al})^{2}}}
\end{equation}
\begin{equation}
\frac{\nu_{+,\al}}{1-\nu_{+, \al}}+\frac{\nu_{-,\al}}{1-\nu_{-, \al}} =   \frac{-4\bfr\lr\lam_{\al}+(\bfr\lr\lam_{\al})^{2}+2\momd(1+\bfr\lr\lam_{\al})-\momd^{2}}{{2(2-\momd)\bfr\lr\lam_{\al}-(\bfr\lr\lam_{\al})^{2}}}
\end{equation}
This gives us the noise kernel norm:
\begin{equation}
\knorm = \frac{\bfr(1-\bfr)}{\D}\sum_{\al=1}^{\D}\frac{\lr^{2}\lam_{\al}^{2}}{-2\bfr\lr\lam_{\al}+2\momd\bfr\lr\lam_{\al}+\momd^{2}+(\bfr\lr\lam_{\al})^{2}}\left(-\frac{2(1-\momd)}{\momd}+\frac{-4\bfr\lr\lam_{\al}+(\bfr\lr\lam_{\al})^{2}+2\momd(1+\bfr\lr\lam_{\al})-\momd^{2}}{{2(2-\momd)\bfr\lr\lam_{\al}-(\bfr\lr\lam_{\al})^{2}}}\right)
\end{equation}
We have already simplified for no momentum ($\momd = 1$). Now we consider the opposite limit of $\momd \ll 1$.
We are also interested in $\bfr\lr\lam_{\al}\ll 1$.
In order for the denominator (of each term in the $t$ sum) to be non-negative, we take: $\momd\gg \sqrt{\bfr\lr\lam_{\al}}$. To lowest order we have:
\begin{equation}
\knorm  \approx \bfr(1-\bfr)\frac{1}{\D}\sum_{\al=1}^{\D}\frac{\lr^{2}\lam_{\al}^{2}}{\momd^{2}}\left(
-\frac{2}{\momd}+\frac{\momd}{2\bfr\lr\lam_{\al}}\right)
\end{equation}
which gives us
\begin{equation}
\knorm  \approx \bfr(1-\bfr)\left[-\frac{2}{\D}\sum_{\al=1}^{\D}\frac{\lr^{2}\lam_{\al}^{2}}{\momd^{3}}  +  \frac{1}{\D}\sum_{\al=1}^{\D}\frac{\lr\lam_{\al}}{2\momd \bfr} \right]
\end{equation}
If we perform the familiar conversions $\lam_{\al} = \D\lam_{\al}$ and $\lr = \lr_{0}/\B$, we have:
\begin{equation}
\knorm  \approx \left(\frac{1}{\B}-\frac{1}{\D}\right)\left[\frac{1}{2\momd}\sum_{\al=1}^{\D}\lr_{0}\lam_{\al}- \frac{2}{{\momd^{3}}}\sum_{\al=1}^{\D}\lr_{0}^{2}\lam_{\al}^{2}\right]
\end{equation}
Note that $\momd^{2}\gg \bfr\lr\lam_{\al} = \B\lr\lam_{\al}$. Therefore,
\begin{equation}
\frac{1}{{\momd^{3}}}\sum_{\al=1}^{\D}\lr_{0}^{2}\lam_{\al}^{2} \ll \frac{1}{\momd} \frac{1}{\B} \sum_{\al=1}^{\D}\lr_{0}\lam_{\al}
\end{equation}
At large batch size $\B$ the first term dominates and we have
\begin{equation}
\knorm  \approx \frac{1}{2}\left(\frac{1}{\B}-\frac{1}{\D}\right)\sum_{\al=1}^{\D}(\lr_{0}/\momd)\lam_{\al}
\end{equation}
The lowest order correction is evidently to replace $\lr_{0}$ with $\lr_{0}/\momd$. The form of the corrections
suggest that as $\momd$ increases ($\mom$ decreasing from $1$), the net effect is some extra stabilization
relative to the effective learning rate $\lr_{0}/\momd$.

\subsection{$\knorm$ and $L^{2}$ regularization}

\label{app:s_eos_l2}

Consider $L^{2}$ regularization in a linear model, with strength parameter $\rad$. The dynamical
equation for $\z$ becomes:
\begin{equation}
\z_{t+1}-\z_{t} = -\lr\left(\J \J^{\top}\pmat_{t}\z_{t} + \rad\J^{\tpose}\th_{t}\right)
\end{equation}
This gives us
\begin{equation}
\z_{t+1}-\z_{t} = -\lr\left(\J \J^{\top}\pmat_{t} +\rad\Id\right)\z_{t}
\end{equation}
The covariance evolves as
\begin{equation}
\z_{t+1}\z_{t+1}^{\tpose}-\z_{t}\z_{t} = -\lr\left(\J\J^{\tpose}\pmat_{t} +\rad\Id\right)\z_{t}\z_{t}^{\tpose}-\lr\z_{t}\z_{t}^{\tpose}\left(\J\J^{\tpose}\pmat_{t} +\rad\Id\right)+\lr^{2}\left(\J\J^{\tpose}\pmat_{t} +\rad\Id\right)\z_{t}\z_{t}^{\tpose}\left(\J\J^{\tpose}\pmat_{t} +\rad\Id\right)
\end{equation}
Averaging over $\pmat$ once again, we have:
\begin{equation}
\begin{split}
\expect_{\pmat}[\z_{t+1}\z_{t+1}^{\tpose}-\z_{t}\z_{t}] & = -\lr\left(\bfr\J\J^{\tpose} +\rad\Id\right)\z_{t}\z_{t}^{\tpose}-\lr\z_{t}\z_{t}^{\tpose}\left(\bfr\J\J^{\tpose} +\rad\Id\right)+\\
& \lr^{2}\left(\bfr\J\J^{\tpose} +\rad\Id\right)\z_{t}\z_{t}^{\tpose}\left(\bfr\J\J^{\tpose} +\rad\Id\right) + \bfr(1-\bfr)\lr^{2} \J\J^{\top}\diag\left[\z_{t}\z_{t}^{\top}\right]\J\J^{\top}
\end{split}
\end{equation}
If we once again define $\pvec_{t}$ to be the vector with elements $\expect_{\pmat}[(\v_{\al}\cdot\z_{t})^{2}]$, where
the $\v_{\al}$ are the eigenvectors of $\J\J^{\top}$, we have
\begin{equation}
\pvec_{t} = \dmat^{t}\pvec_{0},~\dmat\equiv (\Id-\bfr\lr\lmat-\lr\rad\Id)^2 +\bfr(1-\bfr)\lr^2\lmat^2\cmat
\end{equation}
where $\cmat$ is defined as before. In the high-dimensional limit, the noise kernel norm
becomes
\begin{equation}
||\mathcal{K}|| = \bfr(1-\bfr)\sum_{\al} \frac{\lr^{2}\lam_{\al}^{2}}{1-(1-\bfr\lr\lam_{\al}-\lr\rad)^{2}}
\end{equation}
This is bounded from above by the $\rad = 0$ case:
\begin{equation}
\bfr(1-\bfr)\sum_{\al} \frac{\lr^{2}\lam_{\al}^{2}}{1-(1-\bfr\lr\lam_{\al}-\lr\rad)^{2}} \leq  \bfr(1-\bfr)\sum_{\al} \frac{\lr^{2}\lam_{\al}^{2}}{1-(1-\bfr\lr\lam_{\al})^{2}}
\end{equation}
Which suggests that the regularization decreases the noise kernel norm in this case.

Simplifying, we have:
\begin{equation}
||\mathcal{K}|| = (\bfr^{-1}-1)\sum_{\al} \frac{\bfr^{2}\lr^{2}\lam_{\al}^{2}}{2(\bfr\lr\lam_{\al}+\lr\rad)-(\bfr\lr\lam_{\al}+\lr\rad)^{2}}
\end{equation}
Dividing the numerator and denominator by $\bfr\lr\lam_{\al}$, we have
\begin{equation}
||\mathcal{K}|| = (\bfr^{-1}-1)\sum_{\al} \frac{\bfr\lr\lam_{\al}}{2-\bfr\lr\lam_{\al}+(\rad/\lam_{\al})-2\lr\rad-\lr\rad^{2}/\lam_{\al}}
\end{equation}
We can look at limiting behaviors to see two different types of contributions. Assume
$\lr\rad\ll1$. We have:
\begin{equation}
\frac{\bfr\lr\lam_{\al}}{2-\bfr\lr\lam_{\al}+(\rad/\lam_{\al})-2\lr\rad-\lr\rad^{2}/\lam_{\al}} \approx
\begin{cases}
\frac{\bfr\lr\lam_{\al}}{2-\bfr\lr\lam_{\al}} & {\rm~if~}\bfr\lr\lam_{\al}\gg \lr\rad\\
\frac{\bfr\lam_{\al}}{\rad}\frac{\bfr\lr\lam_{\al}}{2} & {\rm~if~}\bfr\lr\lam_{\al}\ll \lr\rad
\end{cases}
\end{equation}
Evidently the effect of the normalization is to decrease the contribution of eigenvalues such that
$\bfr\lam_{\al} < \rad$.

\aga{

\begin{itemize}
    \item Non-MSE analysis?
\end{itemize}

}

\subsection{Beyond MSE loss}

\label{app:non_mse_loss}

Here we consider the stability of SGD under more general convex losses. We will derive a stability condition
by expanding around a minimum. The upshot is that under certain assumptions, we can derive a noise kernel norm
$\knorm$ for non-MSE losses, and there is a regime where we have the estimator
\begin{equation}
\knorm \approx \hat{\knorm}_{tr}\equiv \frac{\lr}{2\B} \tr\left(\frac{1}{\D}\J^{\tpose}\H_{\z}^*\J\right)
\end{equation}
where $\H^*_{\z}$ is the Hessian of the loss with respect to the logits at the minimum.
We note that $\J^{\tpose}\H_{\z}^*\J$ is the Gauss-Newton part of the Hessian.

\subsubsection{Expansion around a fixed point}

Consider a linear model $\z_{t} = \J\th_{t}$. Here $\theta_{t}$ is the $\P$-dimensional parameter vector,
and $\z_{t}$ is the output. If each data point has $\C$ outputs, then we flatten them so that $\z_{t}$ has
dimension $\C\D$. $\J$ is the (flattened) Jacobian with dimension $\C\D\times\P$.

Consider the loss function
\begin{equation}
\Lo(\th) = \frac{1}{\D}\sum_{\al=1}^{\D}\Lo_{z}(\z_{\al}(\theta_{t}))
\end{equation}
Here $\Lo_{z}$ is the per-example loss, convex in the inputs. The update equation for $\th$ under SGD with batch
size $\B$ is
\begin{equation}
\th_{t+1}-\th_{t} = -\frac{\lr}{\B}\J^{\tpose}\pmat_{t}\nabla_{\z}\Lo(\z_{t})
\end{equation}
where $\pmat_{t}$ is the projection matrix with exactly $\B$ $1$s on the diagonal, drawn i.i.d. at each step. The
update equation for $\z_{t}$ is
\begin{equation}
\z_{t+1}-\z_{t} = -\frac{\lr}{\B}\J\J^{\tpose}\pmat_{t}\nabla_{\z}\Lo(\z_{t})
\end{equation}

In general this is a non-linear stochastic system in $\z_{t}$, whose moments don't close at any finite order.
However, we can make progress by expanding around a minimum.
Let $\z^*$ be a minimum of the loss. We have:
\begin{equation}
\nabla_{\z}\Lo(\z) = \H_{\z}(\z^*)(\z-\z^*)+O(||\z-\z^*||^{2})
\end{equation}
where $\H_{\z}$ is the PSD Hessian of $\Lo$ with respect to the logits $\z$. Therefore near $\z^*$ we can write:
\begin{equation}
\z_{t+1}-\z_{t} = -\frac{\lr}{\B}\J\J^{\tpose}\pmat_{t}\H_{\z}(\z^*)(\z_{t}-\z^*)+O(||\z-\z^*||^{2})
\end{equation}
Let $\tl{\z} \equiv \z-\z^*$. Neglecting terms of $O(||\tl{\z}^*||^{2})$ we have:
\begin{equation}
\tl{\z}_{t+1}-\tl{\z}_{t} = -\frac{\lr}{\B}\J\J^{\tpose}\pmat_{t}\H_{\z}^*\tl{\z}_{t}
\end{equation}
where we denote $\H_{\z}^*\equiv \H_{\z}(\z^*)$.

This is similar to the dynamical equation for the MSE case, but with an additional PSD matrix factor. The second
moment equations are:
\begin{equation}
\begin{split}
\expect_{\pmat}[\tl{\z}_{t+1}\tl{\z}_{t+1}^{\top}-\tl{\z}_{t}\tl{\z}_{t}^{\top}|\tl{\z}_{t}] & = -\frac{\lr}{\D} (\J\J^{\top}\H^{*}_{\z}\tl{\z}_{t}\tl{\z}_{t}^{\top}+\tl{\z}_{t}\tl{\z}_{t}^{\top}\H^{*}_{\z}\J\J^{\top}) +\frac{\lr^{2}}{\D^{2}} \J\J^{\top}\H^{*}_{\z}\tl{\z}_{t}\tl{\z}_{t}^{\top}\H^{*}_{\z}\J\J^{\top}\\
&+ (\bfr^{-1}-1)\frac{\lr^{2}}{\D^{2}} \J\J^{\top}\H^{*}_{\z}\diag\left[\tl{\z}_{t}\tl{\z}_{t}^{\top}\right]\H^{*}_{\z}\J\J^{\top}
\end{split}
\label{eq:cov_dyn_gen_loss_lin}
\end{equation}

Using the PSDness of $\H_{\z}^{*}$,
we can define the modified covariance matrix $\modcov_{t}\equiv \H_{\z}^{1/2}\tl{\z}_{t}\tl{\z}_{t}^{\tpose}\H_{\z}^{1/2}$.
The dynamics are given by:
\begin{equation}
\expect_{\pmat}[\modcov_{t+1}-\modcov_{t}|\modcov_{t}]  = -\lr(\modntk\modcov_{t}+\modcov_{t}\modntk) +\lr^{2}(\modntk\modcov_{t}\modntk+
(\bfr^{-1}-1) \modntk(\H_{\z}^{*})^{1/2}\diag\left[(\H^{*}_{\z})^{-1/2}\modcov_{t}(\H^{*}_{\z})^{-1/2}\right](\H_{\z}^{*})^{1/2}\modntk)
\label{eq:cov_dyn_gen_loss_lin}
\end{equation}
where we define $\modntk\equiv \frac{1}{\D} (\H_{\z}^*)^{1/2}\J\J^{\tpose}(\H_{\z}^*)^{1/2}$. Note that $\modntk$ is the Gram
matrix of the Gauss-Newton part of the Hessian, up to a normalizing constant - they have the same non-zero eigenvalues.

We can once again attempt to
work in a diagonal basis to reduce the complexity of the analysis. Consider the eigendecomposition
$\modntk = \V\lmat\V^{\tpose}$. If $\tl{\Svec}\equiv \V^{\tpose}\modcov\V$, we have:
\begin{equation}
\begin{split}
\expect_{\pmat}[\tl{\Svec}_{t+1}-\tl{\Svec}_{t}|\tl{\Svec}_{t}] & = -\lr(\lmat\tl{\Svec}_{t}+\tl{\Svec}_{t}\lmat) +\lr^{2}(\lmat\tl{\Svec}_{t}\lmat+
\\
& (\bfr^{-1}-1) \lmat\V^{\tpose}(\H_{\z}^{*})^{1/2}\diag\left[(\H^{*}_{\z})^{-1/2}\V\tl{\Svec}_{t}\V^{\tpose}(\H^{*}_{\z})^{-1/2}\right](\H_{\z}^{*})^{1/2}\V\lmat)
\end{split}
\label{eq:cov_dyn_gen_loss_svec}
\end{equation}
This equation defines a linear operator $\tl{\linop}$ whose maximum eigenvalue defines stability. We have
\begin{equation}
\expect_{\pmat}[(\tl{\Svec}_{t+1})_{\mu\nu}|\tl{\Svec}_{t}] =\sum_{\bt\gm}\tl{\linop}_{\mu\nu,\bt\gm}(\tl{\S}_{t})_{\bt\gm}
\end{equation}
where $\tl{\linop}$ is given by
\begin{equation}
\begin{split}
\tl{\linop}_{\mu\nu,\bt\gm} = & \delta_{\mu\bt,\nu\gm}(1-\lr(\lam_{\mu}+\lam_{\nu})+\lr^{2}\lam_{\mu}\lam_{\nu})\\
& +(\bfr^{-1}-1)\lr^{2}\lam_{\mu}\lam_{\nu}\left[\sum_{\al,\dl,\eps,\phi,\psi}\V_{\phi\mu}(\H_{\z}^{*})^{1/2}_{\al\phi}(\H_{\z}^{*})^{-1/2}_{\dl\al}\V_{\dl\bt}\V_{\eps\gm}   (\H_{\z}^{*})^{-1/2}_{\al\eps}(\H_{\z}^{*})^{1/2}_{\al\psi}\V_{\psi\nu}\right]
\end{split}
\end{equation}
where $\lam_{\mu}$ is the $\mu$th eigenvalue from $\lmat$.
If we reduce the $(\D\C)^{2}\times(\D\C)^{2}$ system by restricting to the diagonal $\pvec = \diag(\tl{\Svec})$
\begin{equation}
(\pvec_{t+1})_{\mu} = \sum_{\bt} \tl{\linop}_{\mu\mu,\bt\bt}(\pvec_{t})_{\bt}
\end{equation}
which becomes, in matrix notation
\begin{equation}
\pvec_{t+1} = \dmat\pvec_{t},~\dmat \equiv [(\Id-\lr\lmat)^2 +\lr^2(\bfr^{-1}-1)\lmat^2\tl{\cmat}]
\end{equation}
with
\begin{equation}
\tl{\cmat}_{\bt\mu} \equiv \sum_{\al,\dl,\phi}[\V_{\phi\mu}(\H_{\z}^{*})^{1/2}_{\al\phi}]^{2}[(\H_{\z}^{*})^{-1/2}_{\dl\al}\V_{\dl\bt}]^{2}
\end{equation}
Note: $\H_{\z}^*$ is block-diagonal with respect to the $\C\times\C$ blocks for the $\D$ datapoints. If $\H_{\z}^*$ is diagonal within each
block (no logit-logit interactions), then $\tl{\cmat} = \cmat$ from the MSE case. Otherwise, $\tl{\cmat}$ is a slightly
different positive matrix.

This means that we can derive a noise kernel norm $\knorm$ following the analysis in Section 3.2 of the main text,
using $\amat = (\Id-\lr\lmat)^2$, $\bmat = \lr^2(\bfr^{-1}-1)\lmat\tl{\cmat}\lmat$.

\subsubsection{Relationship to previous analysis}

\label{sec:non_mse_caveats}

This analysis is analogous to the MSE case, with the modified NTK $\modntk$ taking the role of the NTK - meaning the
Gauss-Newton eigenvalues are key. If
$\tl{\cmat}\approx \frac{1}{\D}\m{1}\m{1}^{\tpose}$, then we recover the estimators from Section 3.3, replacing
$\ntk$ with $\modntk$ - or alternatively,
\begin{equation}
\knorm \approx \hat{\knorm}_{tr}\equiv \frac{\lr}{2\B} \tr\left(\modntk\right) = \frac{\lr}{2\B} \tr\left(\frac{1}{\D}\J^{\tpose}\H_{\z}^*\J\right)
\end{equation}
The last expression is written in terms of the Gauss-Newton matrix at the minimum.

However, there are a few ways this quantity may suffer compare to the MSE one:
\begin{itemize}
    \item \textbf{Expansion around $\z^*$.} In order to derive a linear recurrence relation, we expanded around the minimum
    $\z^*$. If the dynamics is near but not at a minimum, an accurate computation would require finding $\z^*$, and computing
    $\modntk$ there. If the dynamics is not near a minimum, then the accuracy of the stability condition is unclear.
    \item \textbf{Restriction of $\tl{\linop}$ to the diagonal.} In order to derive $\knorm$ we reduce to the dynamics of the
    diagonal of the covariance only. For MSE loss previous work has justified this approximation in certain high dimensional
    limits; for more general loss functions this is not clear.
    \item \textbf{Nontrival structure of $\H^*_{\z}$.} In order to use efficient high-dimensional approximators of $\knorm$,
    it is useful for $\cmat$ to have a low-rank structure. In the MSE case this can be a good approximation because eigenvectors
    are delocalized in the coordinate basis; in the more general setting, this may no longer be the case. For example,
    cross-entropy could introduce additional correlations across members of the same class, different inputs, or the same inputs,
    different classes.
\end{itemize}

\section{Conservative sharpening in the quadratic regression model}

\label{app:cons_sharp_proofs}

\subsection{Quadratic regression model definition}

\label{app:quad_reg_model}

The quadratic regression model can be derived from a second order Taylor expansion
of a model $\f(\theta)$ on $\D$ outputs with $\P$-dimensional parameter vector
$\th$:
\begin{equation}
\f(\th) \approx \f(\th_{0})+\J_{0}[\th-\th_{0}] +\frac{1}{2}\Q[\th-\th_{0},\th-\th_{0}].
\end{equation}
Here
$\J_{0}\equiv \frac{\partial \f}{\partial\th}(\th_{0})$ is the $\D\times\P$-dimensional Jacobian at $\th_{0}$,
and $\Q \equiv \frac{\partial^{2}\f}{\partial\th\partial\th'}(\th_{0})$ is the $\D\times\P\times\P$-dimensional
\emph{model curvature}. For $\Q = 0$, we recover a linear regression model.We assume, WLOG, that $\th_{0} = 0$. This means we can write the model as
\begin{equation}
\f(\th) \approx \f(\th_{0})+\J_{0}[\th] +\frac{1}{2}\Q[\th,\th].
\end{equation}

For MSE loss with targets $\y_{tr}$, the full loss is
given by
\begin{equation}
\Lo(\th) = \frac{1}{2\D}||\z||^{2},~\z\equiv \f(\th)-\y_{tr}.
\end{equation}
while the loss with minibatch SGD, batch size $\B$ is
\begin{equation}
\Lo_{mb,t}(\th) = \frac{1}{2\B}\z^{\tpose}\pmat_{t}\z.
\end{equation}
where $\pmat_{t}$ is the sequence of random diagonal projection matrices of rank $\B$ as
before. The dynamics of $\th_{t}$ are given by:
\begin{equation}
\th_{t+1}-\th_{t} = -\lr\J_{t}^{\tpose}\z_{t}
\end{equation}
where $\J_{t}\equiv \left.\frac{d\f}{d\th}\right|_{\th_{t}}$. Following the analysis of \citet{agarwala_sam_2023}, in the quadratic regression model we have:
\begin{equation}
\z_{t} = \f(\th_{0})+\J_{0}[\th_{t}] +\frac{1}{2}\Q[\th_{t},\th_{t}]-\y_{tr} 
\end{equation}
\begin{equation}
\J_{t} = \J_{0}+\Q[\th_{t},\cdot]
\end{equation}
which gives us the differences:
\begin{equation}
\z_{t+1}-\z_{t} = \J_{t}[\th_{t+1}-\th_{t}]+\lr^{2}\Q[\th_{t+1}-\th_{t}, \th_{t+1}-\th_{t}]
\end{equation}
\begin{equation}
\J_{t+1}-\J_{t} = \Q(\th_{t+1}-\th_{t}, \cdot)
\end{equation}
Substitution gives us:
\begin{align}
\z_{t+1}-\z_{t} & = -\frac{\lr}{\B}\J_{t} \J_{t}^{\top}\pmat_{t}\z_{t}  +\frac{\lr^{2}}{2\B^{2}} \Q(\J_{t}^{\top}\pmat_{t}\z_{t},\J_{t}^{\top}\pmat_{t}\z_{t})\notag\\
\J_{t+1} -\J_{t} & = -\frac{\lr}{\B} \Q(\J_{t}^{\top}\pmat_{t}\z_{t}, \cdot)\,.
\end{align}
Therefore the dynamics close in $\z_{t}$ and $\J_{t}$ given the fixed model curvature
$\Q$.

In the remainder of this section, we prove Theorem \ref{thm:cons_sharp} in two parts, and provide
numerical evidence for its validity.
For ease of notation, we define $\tlr = \lr/\B$. This is equivalent to the scaling
in \citet{paquette_sgd_2021}, and allows us to keep the calculations
in terms of $\bfr$ rather than the raw $\B$.
The final theorem can be obtained with the substitution of $\tlr$.

\subsection{First discrete derivative of NTK}

By definition we have
\begin{equation}
\Delta_{1}\lhat_{\al,t} = \w_{\al}^{\tpose}[\J_{t+1}\J^{\tpose}_{t+1}-\J_{t}\J^{\tpose}_{t}]\w_{\al}
\end{equation}
Using Equation \ref{eq:SGD_general}, and averaging over $\pmat$ we have
\begin{equation}
\begin{split}
\expect_{\pmat}[\J_{t+1}\J^{\tpose}_{t+1}-\J_{t}\J^{\tpose}_{t}|\z_{t}, \J_{t}] & = -\bfr\tlr\left[\Q(\J_{t}^{\top}\z_{t}, \J_{t}^{\tpose}\cdot)+\Q(\J_{t}^{\top}\z_{t}, \J_{t}^{\tpose}\cdot)^{\tpose}\right]+\bfr^{2}\tlr^{2}\Q(\J_{t}^{\top}\z_{t}, \cdot)\Q(\J_{t}^{\top}\z_{t}, \cdot)^{\tpose}\\
& +\bfr(1-\bfr)\tlr^2\expect_{\Q}\left[\sum_{\mu}z_{t, \mu}^{2} \left[\Q(\J_{t}^{\top}\m{e}_{\mu}, \cdot)\Q(\J_{t}^{\top}\m{e}_{\mu}, \cdot)^{\tpose}\right]\right]
\end{split}
\end{equation}
where the $\m{e}_{\mu}$ are the coordinate basis vectors.

Recall that we define $\Q$ via the equation
\begin{equation}
\Q = \sum_{\gm} \w_{\gm}\otimes\m{M}_{\gm}
\end{equation}
where the $\m{M}_{\gm}$ are i.i.d. symmetric matrices with variances $V(\sigma_{\gm})$.
Therefore, averaging over $\Q$, the first two terms vanish and we have:
\begin{equation}
\begin{split}
\w_{\al}^{\tpose} \expect_{\pmat, \Q}[\J_{t+1}\J^{\tpose}_{t+1}-\J_{t}\J^{\tpose}_{t}|\z_{t}, \J_{t}]\w_{\al} & = \bfr^{2}\tlr^{2}\expect_{\m{M}_{\al}}\left[\z_{t}^{\tpose}\J_{t}^{\tpose}\m{M}_{\al}^{\tpose}\m{M}_{\al}\J_{t}\z_{t}\right]\\
&+\bfr(1-\bfr)\tlr^2 \expect_{\m{M}_{\al}}\left[\sum_{\mu}z_{t, \mu}^{2} \m{e}_{\mu}^{\tpose}\J_{t}^{\tpose}\m{M}_{\al}^{\tpose}\m{M}_{\al}\J_{t}\m{e}_{\mu}\right]+O(\D^{-1})
\end{split}
\end{equation}
Conducting the average over $\z_{t}$ gives us, as desired:
\begin{equation}
\expect_{\pmat,\Q,\z}[\Delta_{1}\lhat_{\al,t}] =
\P\D^{2} V_{z}\tr\left[\frac{1}{\D}\J_{t}^{\tpose}\J_{t}\right]\tlr^{2} \bfr V(\sigma_{\al})+O(\D^{-1})
\end{equation}

\subsection{Second discrete derivative of $\J$}

Now we consider
\begin{equation}
\Delta_{2}\shat_{\al,t} = \w_{\al}^{\tpose}[\J_{t+2}-2\J_{t+1}+\J_{t}]\v_{\al}
\end{equation}

We can re-write this as:
\begin{equation}
\begin{split}
\J_{t+2}-2\J_{t+1}+\J_{t} & = -\tlr [\Q((\J_{t+1}-\J_{t})^{\top}\pmat_{t+1}\z_{t}, \cdot)+\Q(\J_{t}^{\top}\pmat_{t+1}(\z_{t+1}-\z_{t}), \cdot)\\
&+\Q((\J_{t+1}-\J_{t})^{\top}\pmat_{t+1}(\z_{t+1}-\z_{t}), \cdot)]
 -\tlr \Q(\J_{t}^{\top}(\pmat_{t+1}-\pmat_{t})\z_{t}, \cdot)
\end{split}
\end{equation}

Consider $\expect_{\pmat}[\J_{t+2}-2\J_{t+1}+\J_{t}]$.
Most of the terms contain only one copy of 
$\pmat_{t}$ or $\pmat_{t+1}$, so averaging gives a quantity that is 
``deterministic'' - identical for fixed values of the product $\bfr\tlr$. The one
non-trivial average is, to lowest order in $\tlr$:
\begin{equation}
\expect_{\pmat}[\Q((\J_{t+1}-\J_{t})^{\top}\pmat_{t+1}(\z_{t+1}-\z_{t}), \cdot)] = \tlr^{2} \expect_{\pmat}[\Q(\Q(\J_{t}^{\top}\pmat_{t}\z_{t}, \cdot)^{\top}\pmat_{t+1}\J_{t} \J_{t}^{\top}\pmat_{t}\z_{t}, \cdot)]+O(\tlr^3)
\end{equation}
Evaluating the average we have:
\begin{equation}
\begin{split}
\expect_{\pmat}[\Q((\J_{t+1}-\J_{t})^{\top}\pmat_{t+1}(\z_{t+1}-\z_{t}), \cdot)] & = \bfr^{2}\tlr^{2} \Q(\Q(\J_{t}^{\top}\z_{t}, \cdot)^{\top}\J_{t} \J_{t}^{\top}\z_{t}, \cdot)\\
& +\bfr^2(1-\bfr)\tlr^2\Q(\m{N}(\z_{t},\J_{t})\z_{t}, \cdot)+O(\tlr^3)
\end{split}
\end{equation}
Where the matrix valued function $\m{N}(\z, \J)$ is given by:
\begin{equation}
\m{N}(\z, \J)_{i\gm} = \sum_{\bt, j}\Q_{\bt i j}\J_{\gm j}\z_{\gm}(\J\J^{\tpose})_{\bt \gm}
\end{equation}

We can write $\expect_{\pmat}[\Delta_{2}\lhat_{\al,t}]$ as
\begin{equation}
\expect_{\pmat}[\Delta_{2}\lhat_{\al,t}] = d_{2}(\z_{t}, \J_{t}, \bfr\tlr)-\bfr^2\tlr^{3}\w_{\al}^{\tpose}\Q(\m{N}(\z_{t}, \J_{t})\z_{t}, \v_{\al})+O(\tlr^4)
\end{equation}
where the deterministic part $d_{2}$ is given by
\begin{equation}
d_{2}(\z, \J, \tlr) = \tlr^2\w_{\al}^{\tpose}\left[\Q(\z\cdot\Q(\J^{\tpose}\z, \cdot), \v_{\al})+\Q(\J^{\tpose}\J\J^{\tpose}
\z,\v_{\al})- \tlr\Q(\Q(\J^{\top}\z, \cdot)^{\top}\J \J^{\top}\z, \v_{\al})\right]+\tlr^3\w_{\al}^{\tpose}\Q(\m{N}(\z_{t},\J_{t})\z_{t}, \v_{\al})
\label{eq:d2_pre_ave}
\end{equation}
The $d_{2}$ term is the same for constant $\bfr\tlr$. In the batch-averaged setting, it
has no dependence on batch size.

It remains to average the stochastic term over $\Q$ and $\z$. Averaging over $\Q$ first, we
have
\begin{equation}
\expect_{\Q}[\w_{\al}\cdot\Q(\m{N}(\z, \J)\z, \v_{\al})] = \expect_{\Q}[[\w_{\al}\cdot\Q]_{ij}(\v_{\al})_{j}\Q_{\bt ik}\J_{\gm k}\z_{\gm}(\J\J^{\tpose})_{\bt\gm}\z_{\gm}]
\end{equation}
Expanding $\J\J^{\tpose} = \sum_{\bt} (\sigma_{\bt})^2 \w_{\bt}\w_{\bt}^{\tpose}$, we
have
\begin{equation}
\expect_{\Q}[\w_{\al}\cdot\Q(\m{N}(\z, \J)\z, \v_{\al})] = \sum_{\bt}\expect_{\Q}[(\sigma_{\bt})^{2}[\w_{\al}\cdot\Q]_{ij}(\v_{\al})_{j}[\w_{\bt}\cdot\Q]_{ik}\J_{\gm k}\z^{2}_{\gm}(\w_{\bt})_{\gm}]
\end{equation}
If $\z$ is independent of $\J$ we have
\begin{equation}
\expect_{\Q, \z}[\w_{\al}\cdot\Q(\m{N}(\z, \J)\z, \v_{\al})] =
\sigma^{3}_{\al}V(\sigma_{\al}) \P \expect_{\z}[(\w_{\al})^{\tpose}\diag(\z^{2})\w_{\al}]
\end{equation}
This is a non-negative number. The magnitude depends on the correlation
between $\w_{\sigma}$ and $\z$, the singular values $\sigma$, and the magnitude
of the projection of $\Q$ in the appropriate eigenspace.

Finally, making the i.i.d. assumption on $\z$ we have
\begin{equation}
\expect_{\Q, \z}[\w_{\al}\cdot\Q(\m{N}(\z, \J)\z, \v_{\sigma})] =
\shat_{\al, t}^{3}V(\sigma_{\al}) \P V_{z} +O(\tlr^4)
\end{equation}
In total, we have:
\begin{equation}
\expect_{\pmat, \Q, \z}[\Delta_{2}\shat_{\al,t}]  = d_{2}(\bfr\tlr) -\bfr^{2}\tlr^{3}\shat_{\al, t}^{3}V(\sigma_{\al}) \P V_{z} +O(\tlr^4)
\end{equation}
where $d_{2}(\tlr) = \expect_{\z,\Q}[d_{2}(\z, \J, \tlr)]$ from Equation \ref{eq:d2_pre_ave}.
This concludes the proof of the theorem.

\subsection{Numerical results}

In order to support the theory, we simulated a quadratic regression model with
$\D = 400$, $\P = 600$, with various $\Q$ spectra $V(\sigma)$, and plotted
the dynamics of $\Delta_{1}\hat{\lam}$ and $\Delta_{2}\hat{\sigma}$ for the
largest eigenvalues (Figure \ref{fig:cons_scaling},
averaged over $30$ seeds). We compare the ``flat'' spectrum $V(\sigma) = 1$ with the
``shaped'' spectrum $V(\sigma)\propto \sigma$. As predicted by the theory, the first derivative
increases with $\B^{-1}$ for fixed $\lr$ , while second derivative decreases. Theoretical fit is
better for flat $\Q$. Both the increase and the decrease are more extreme for
the shaped $\Q$.

\begin{figure}[h]
\centering
    \begin{tabular}{cc}
    \includegraphics[width=0.4\linewidth]{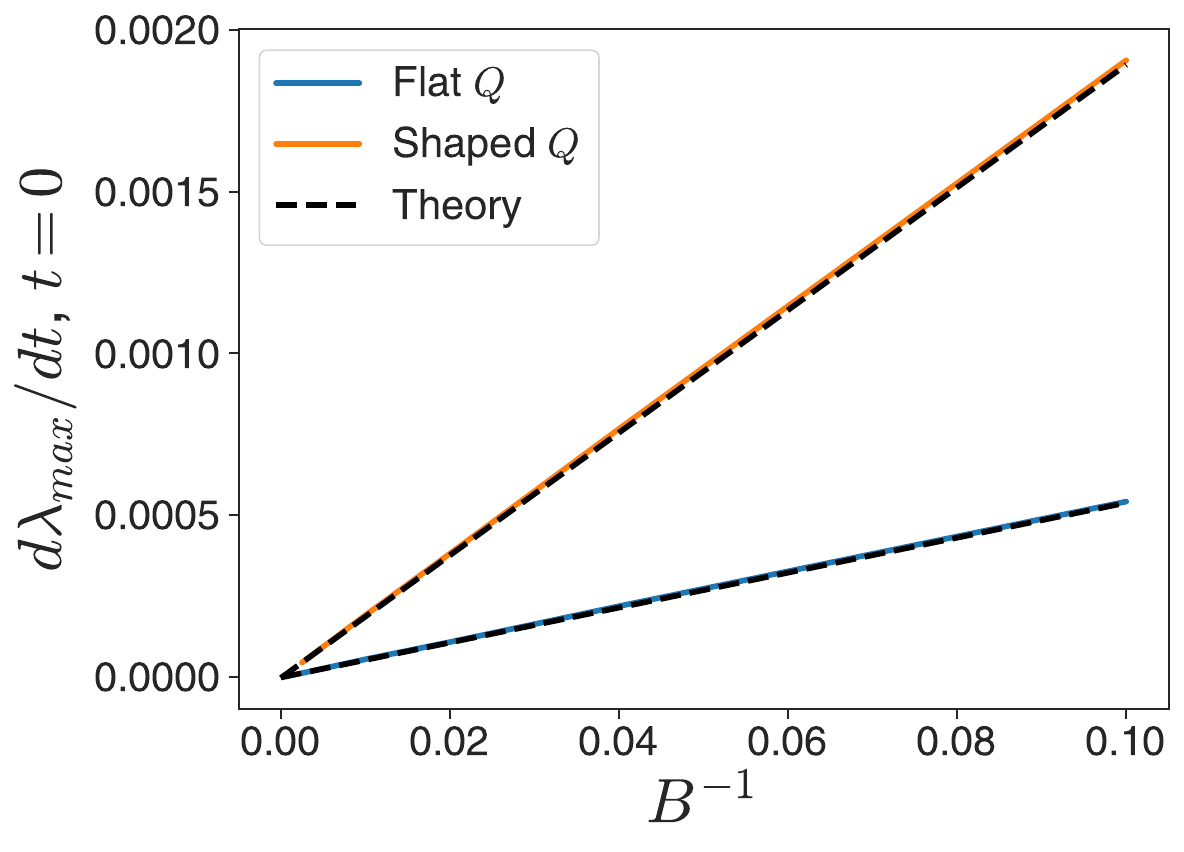} & \includegraphics[width=0.4\linewidth]{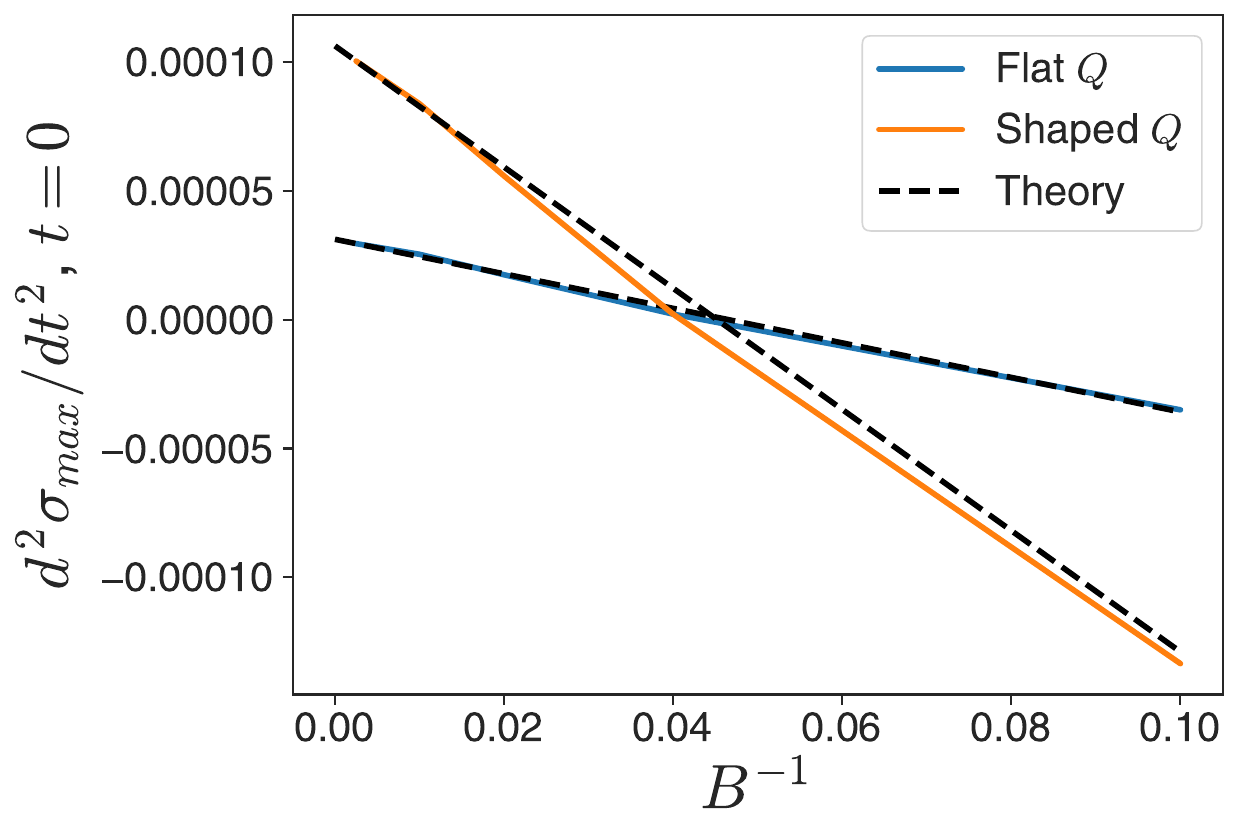}
    \end{tabular}
    \caption{Eigenvalue discrete derivatives $\Delta_{1}\hat{\lam}$ (left) and
    $\Delta_{2}\hat{\sigma}$ (right) for quadratic regression model, $\D = 400$,
    $\P = 600$, averaged over $30$ seeds. The Jacobian $\J$ is initialized with random
    elements, and $\Q$ has either a ``flat'' spectrum of $V(\sigma) = 1$ (blue) or a ``shaped''
    spectrum of $V(\sigma) = \sigma$ (orange). First derivative increases as batch size $\B$
    decreases, while second derivative decreases. Shaped $\Q$ show stronger trends for both.
    }
    \label{fig:cons_scaling}
\end{figure}

\section{MNIST experiments}

\label{app:mnist_details}

\aga{

\begin{itemize}
    \item Cross-entropy analysis?
\end{itemize}

}

\subsection{Experimental setup}

The experiments in Section \ref{sec:fcn_experiments} were all conducted using the first $2500$
examples from MNIST. The labels were converted to $1$ (odd digits) or $-1$ (even digits), and
the models were trained with MSE loss. The networks architecture was
two fully connected hidden layers of width $256$, with ${\rm erf}$ activation function.
Inputs were pre-processed with ZCA.

For small batch sizes, networks were trained with a constant number of epochs. We trained for
$1.2\cdot10^{6}$ total samples ($480$ epochs) up to and including batch size $32$. This was motivated by the observation
that for small $\lr$ and small $\B$, dynamics was roughly universal for a fixed number
of epochs for constant $\lr/\B$ (as is the case in the convex setting of \citep{paquette_sgd_2021}).
However, for larger batch sizes the dynamics is most similar for similar values of $\lr$,
keeping the number of \emph{steps} fixed. For batch size $32$ and larger, models were trained
for $3.75\cdot10^{4}$ steps. Models were trained on A100 GPUs; Figure \ref{fig:cifar_phase_plane}
took $\sim500$ GPU hours to generate due to the large number of steps. There is much room for
efficiency improvement by using just-in-time compilation for sets of steps rather than
individual ones.

We also changed the learning rate sweep range in a batch-size dependent way. For small
batch size $\B\leq 32$ we swept over a constant range in $\lr/\B$, since this was the
parameter which predicts divergence in the small batch setting. For larger batch sizes
$\B\geq 32$ we swept over a constant range in $\lr$ - chosen once again using the
same $\lr$ range as for $\B = 32$. This let us efficiently explore both the small batch and
large batch regimes in fine detail over $\lr$ and $\B$.

\subsection{Approximate vs exact $\knorm$}

This setup was chosen to allow for exact computation of $\knorm$ as per Equation
\ref{eq:pvec_eqn}. We computed the empirical NTK exactly, took its eigendecomposition,
and used that to construct the matrix $\m{M} = (\Id-\amat)^{-1/2}\bmat(\Id-\amat)^{-1/2}$.
This is similar to $(\Id-\amat)^{-1}\bmat$ but is symmetric. We then computed
the maximum eigenvalue of $\m{M}$ to obtain the instantaneous value of $\knorm$.

As in the convex case, the trace estimator of $\knorm$ systematically underestimates
the true value of $\knorm$, especially near $\knorm = 1$ (Figure \ref{fig:fcn_knorm_vs_trace}).
Both quantities are still $O(1)$ over a similar regime but quantitative prediction of
largest stable learning rate is easier with exact value.

\begin{figure}[h]
\centering
    \begin{tabular}{cc}
    \includegraphics[width=0.4\linewidth]{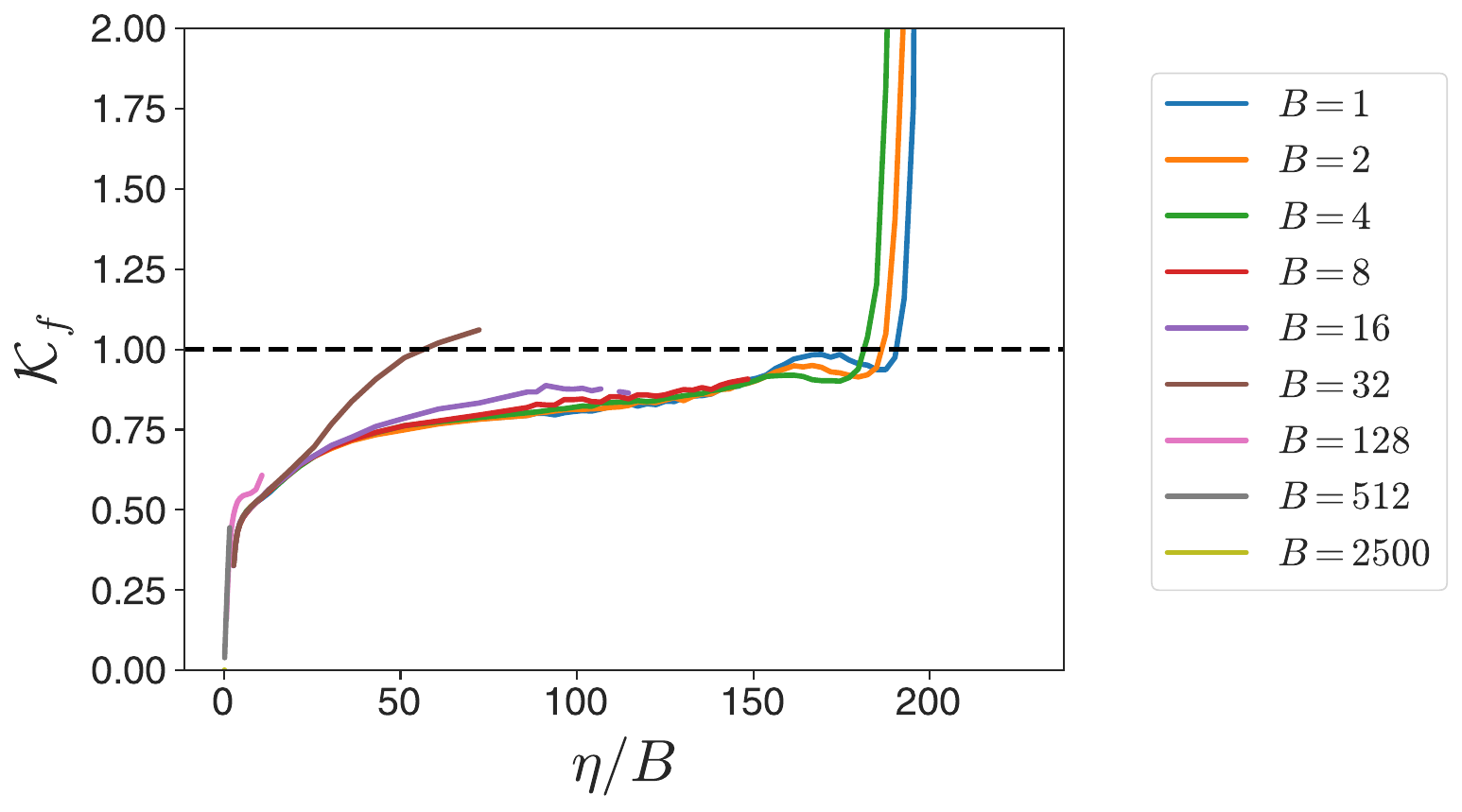} & \includegraphics[width=0.4\linewidth]{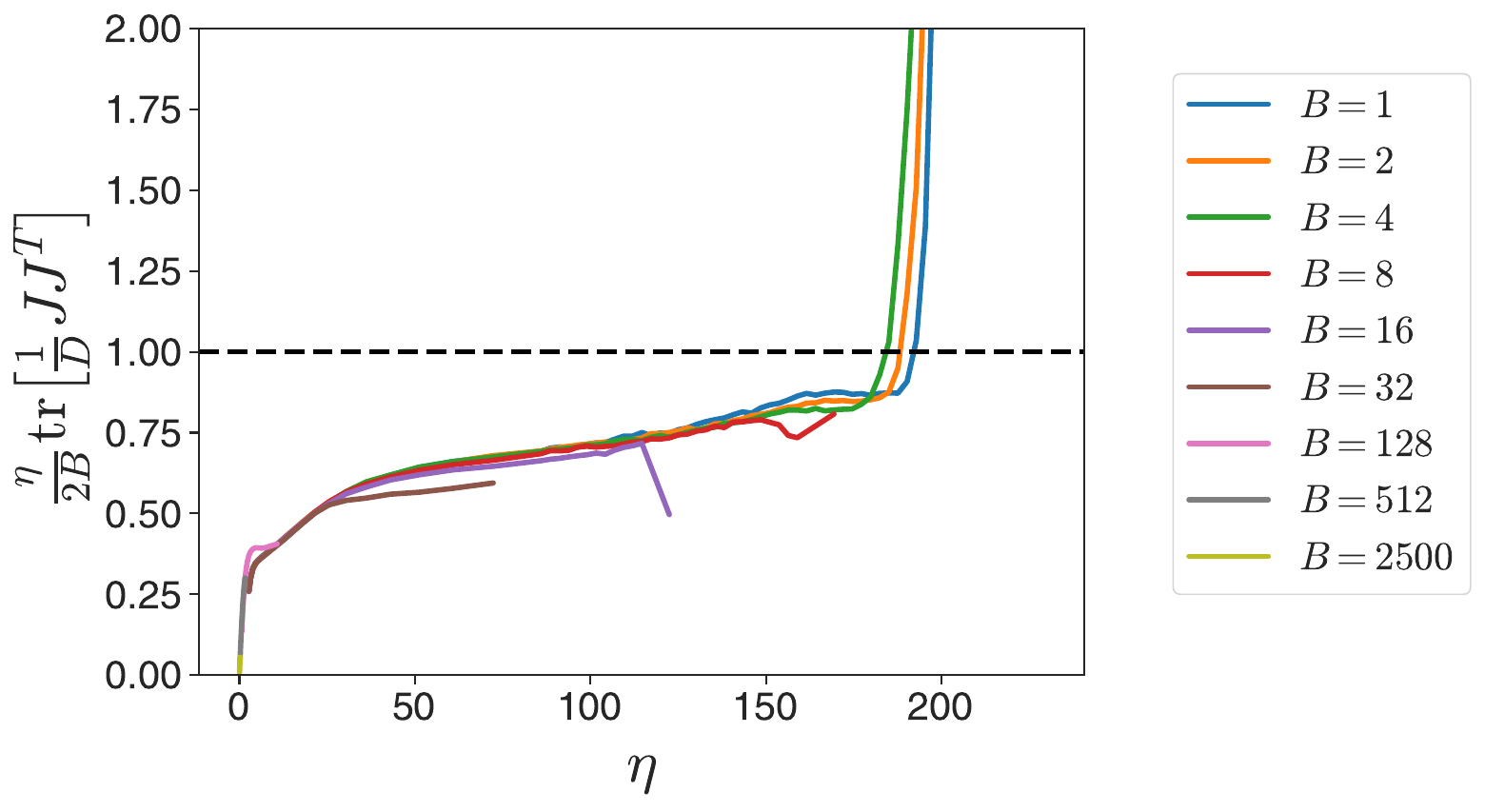}
    \end{tabular}
    \caption{Exact computation of $\knorm$ (left) vs. trace estimator (right) for FCN trained
    on MNIST. Trace estimate underestimates the true $\knorm$, especially as $\knorm$ goes
    to a value near $1$.
    }
    \label{fig:fcn_knorm_vs_trace}
\end{figure}

\section{CIFAR experiments}

\label{app:cifar_details}

\aga{

\begin{itemize}
    \item Hessian trace without L2 regularizer
\end{itemize}

}

\subsection{Experimental setup}

The experiments in Section \ref{sec:cifar_exp} were conducted on CIFAR10 using ResNet18
\citep{he_deep_2016}, with layer normalization and ${\rm GeLU}$ activation function.
The models were trained with MSE loss and $L^{2}$ regularization with
$\lambda = 5\cdot 10^{-4}$ using momentum with
parameter $0.9$ and a cosine learning rate schedule. We trained with batch sizes
$2^{k}$ for $k\in \{3, 4, \dots, 8\}$. All models were trained
for $200$ epochs on 8 V100 GPUs ($20$ hours per training run, most time
spent on full batch eigenvalue estimation).

For each batch size $\B$, we swept over constant normalized base learning rate $\lr/\B$
in the range $[10^{-4}, 0.0125]$, interpolating evenly in log space by powers of $2$.
For batch size $128$, this corresponds to a range $[0.0125, 1.6]$ in base learning
rate $\lr$.

The measurements of largest eigenvalues were made with a Lanczos method as in
\citep{ghorbani_investigation_2019}, from which we also obtained estimates of the
trace of the full Hessian. The NTK trace was computed exactly using autodifferentiation.

\subsection{Phase plane plots}

We can use the sweep over $\B$ and $\lr/\B$ to construct phase plane plots
for the CIFAR
experiments similar to those for MNIST in Section \ref{sec:fcn_experiments}. Once again we
see that the median estimated noise kernel norm (Figure \ref{fig:cifar_phase_plane}, left)
and the final error (Figure \ref{fig:cifar_phase_plane}, right) are similar for constant
$\lr/\B$ across batch sizes. We also see evidence that the universality is broken
for both large $\lr/\B$ as well as large batch size $\B$.

\begin{figure}[h]
\centering
    \begin{tabular}{cc}
    \includegraphics[width=0.4\linewidth]{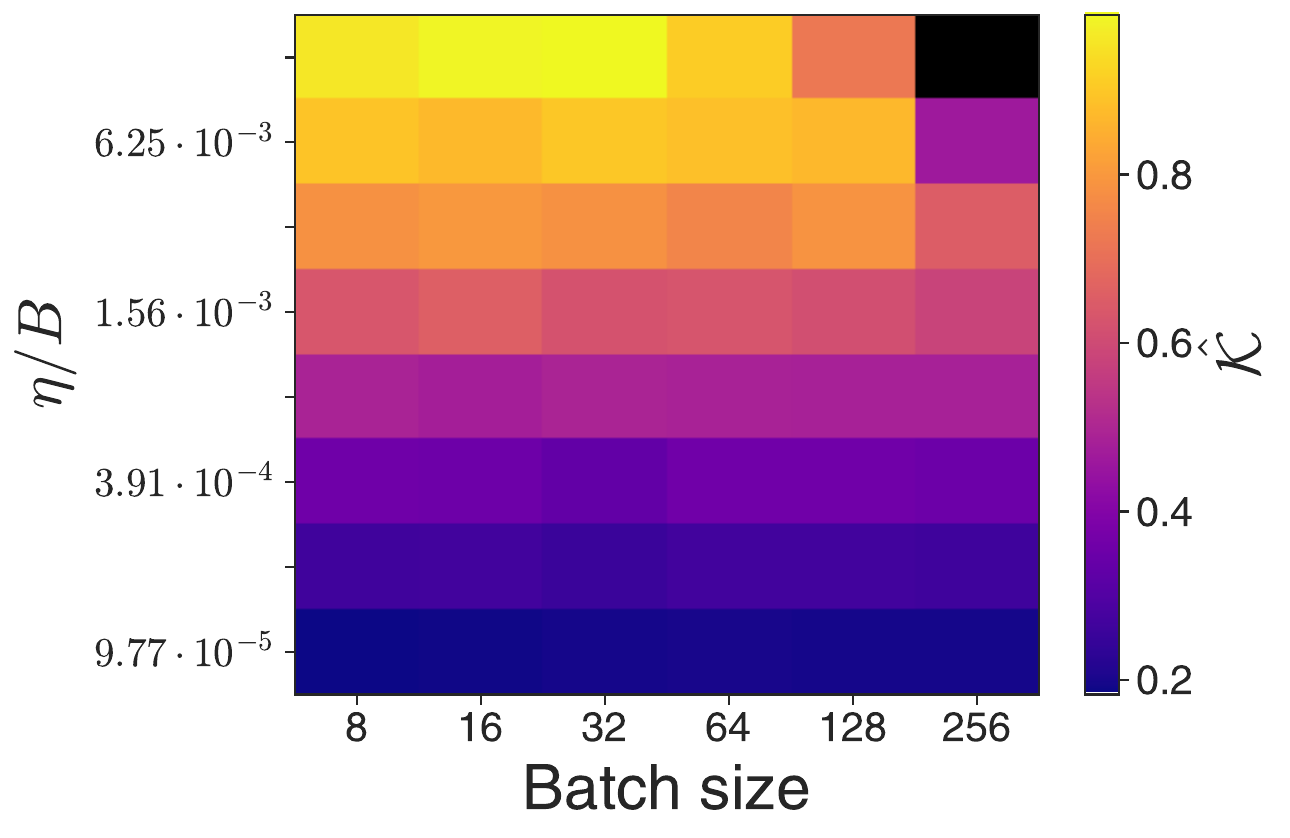} & \includegraphics[width=0.4\linewidth]{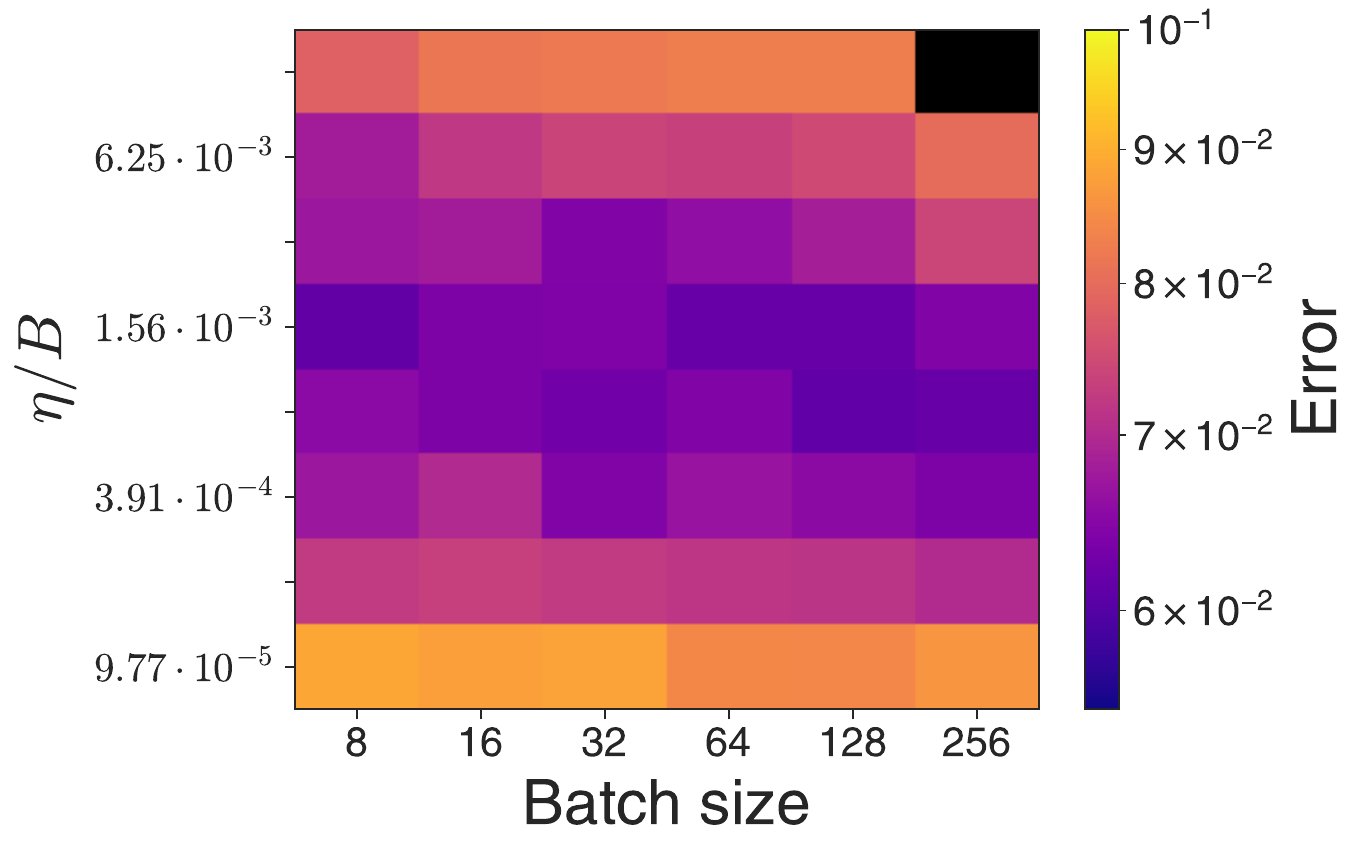}
    \end{tabular}
    \caption{Phase planes for median $\knorm$ (left) and final test error (right) for
    ResNet18 trained on CIFAR10. $\knorm$ increases with increasing $\lr/\B$. Statistics
    are consistent for a range of batch sizes for fixed $\lr/\B$. Consistency breaks down
    at large $\lr/\B$ corresponding to values of $\knorm$ close to $1$, as well as for
    larger batch size.
    }
    \label{fig:cifar_phase_plane}
\end{figure}

\subsection{Raw NTK trace}

\label{app:raw_ntk_tr}

The raw values of the NTK trace are plotted in Figure \ref{fig:cifar_ntk_tr}. The raw
eigenvalues actually increase slightly as the learning rate drops, except at late times
where they decrease.

\begin{figure}[h]
    \centering
     \includegraphics[width=0.45\linewidth]{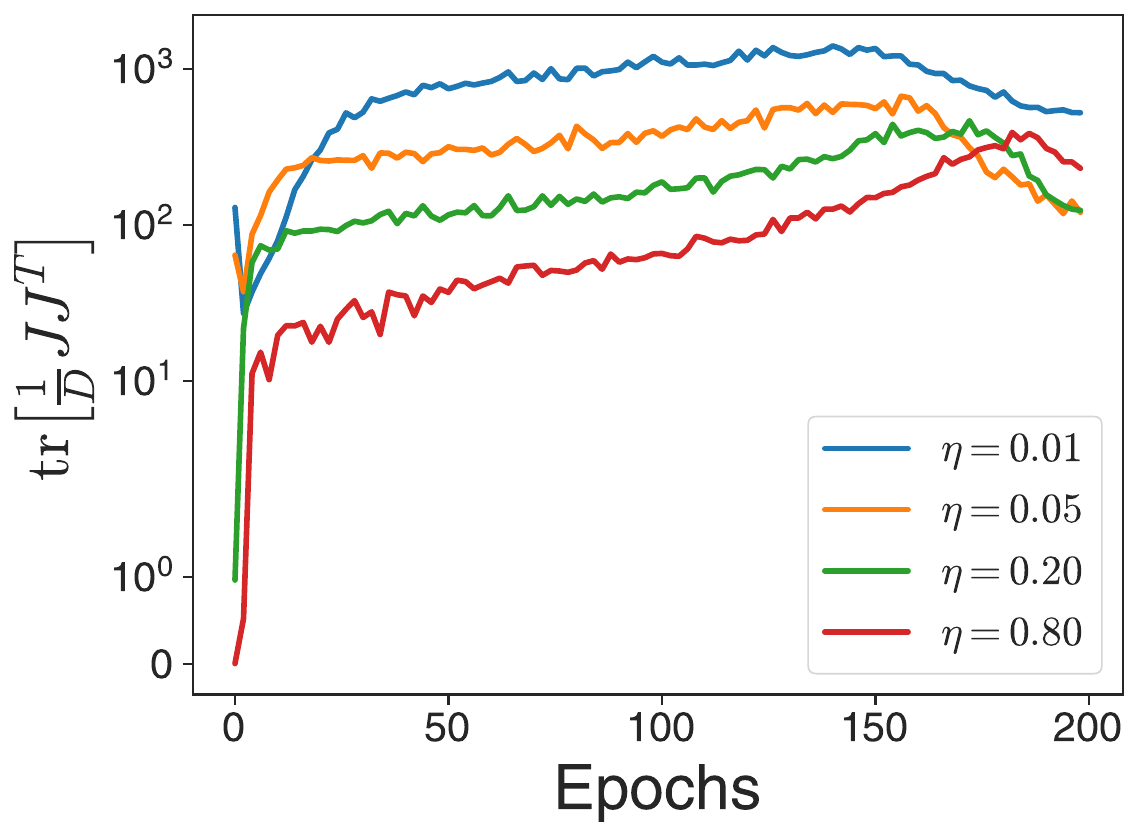}
    \caption{Un-normalized $\tr[\ntk]$ quantity is increasing for much of learning but decreases at the end of training.
    }
    \label{fig:cifar_ntk_tr}
\end{figure}

\subsection{Largest eigenvalue dynamics}

\label{app:cifar_large_eig}

The learning rate and batch size ranges were chosen, in part, because they lead to dynamics
which is well below the (deterministic) edge of stability. The dynamics of the largest
eigenvalue $\lam_{max}$ of the full-dataset Hessian can be seen in Figure
\ref{fig:cifar_lam_max}. The raw eigenvalue has an initial increase, a later decrease,
a plateau, and finally a decrease (left). However, the normalized eigenvalue
$\lr_{t}\lam_{max}$ increases and then decreases, and stays well below the edge of stability
value of $2$ (right). This suggests that the results of Section \ref{sec:cifar_exp} can't
be explained by the deterministic edge of stability. Note that the normalized values
are computed using the instantaneous step size.

\begin{figure}[h]
\centering
    \begin{tabular}{cc}
    \includegraphics[width=0.4\linewidth]{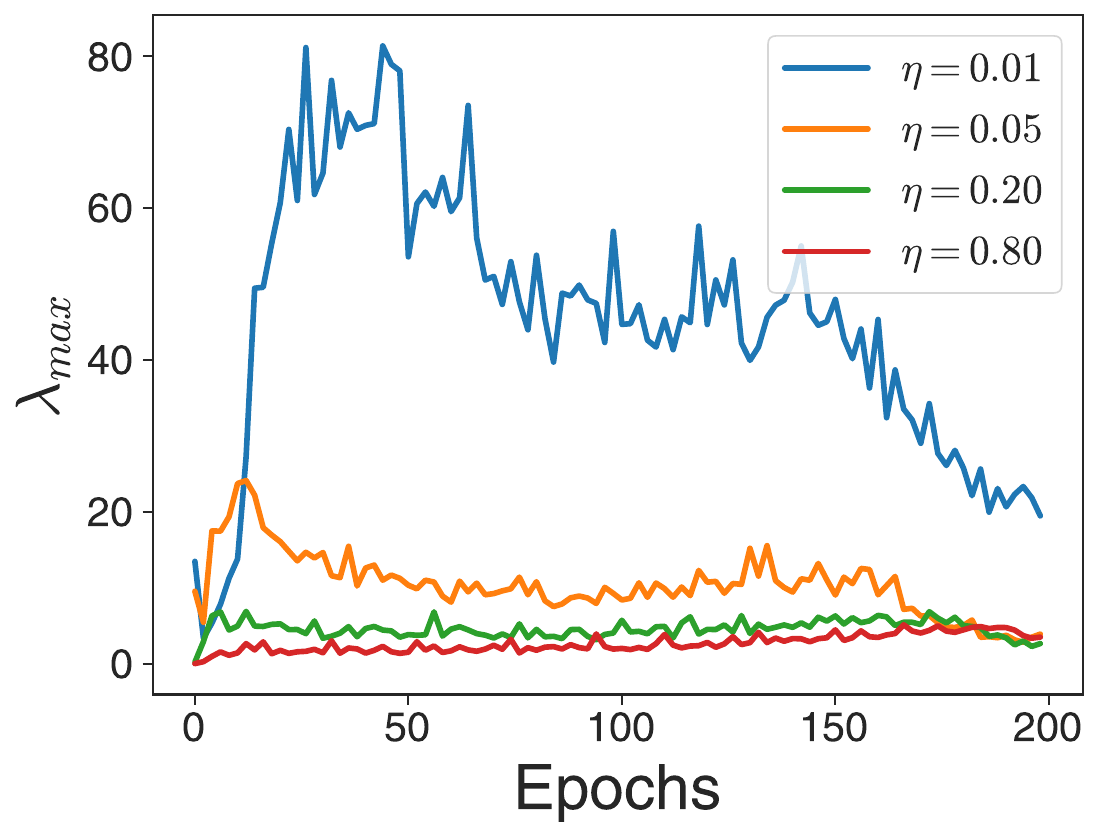} & \includegraphics[width=0.4\linewidth]{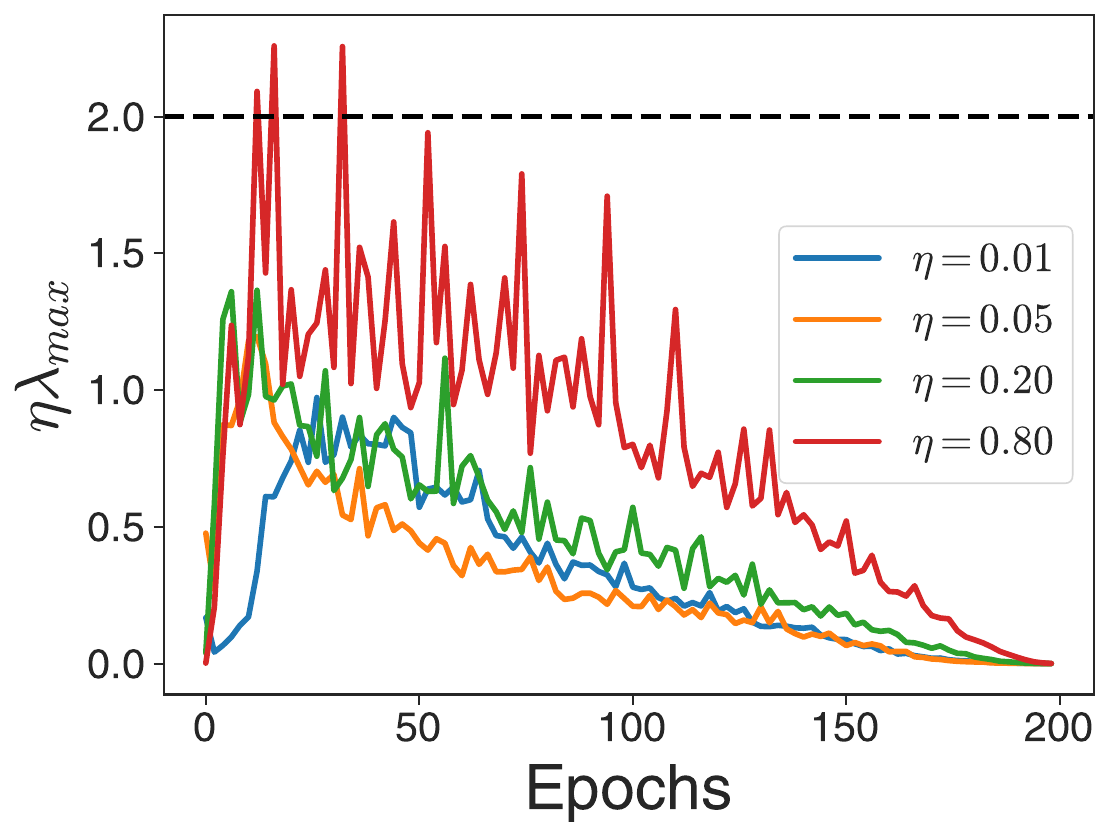}
    \end{tabular}
    \caption{Maximum eigenvalue for ResNet18 on CIFAR. Raw eigenvalue increases at
    early times, then decreases to a steady value at intermediate times, and finally
    decreases at late times (left). Normalized eigenvalue is below the edge of stability
    ($\lr_{t}\lam_{max} < 2$) for all but the largest learning rate (right).
    }
    \label{fig:cifar_lam_max}
\end{figure}

\subsection{Hessian trace}

\label{app:hess_trace}

The full Hessian trace is dominated by the $L^{2}$ regularizer coefficient, and is therefore
a poor estimator of $\knorm$ (Figure \ref{fig:hess_tr_no_l2}, left).
We can confirm that even removing the $L^{2}$ regularizer during the computation of the
Hessian trace does not fix the issue (Figure \ref{fig:hess_tr_no_l2}, right). Indeed the Hessian
trace varies wildly over the course of learning, due to the contributions from the
non-Gauss Newton part \citep{dauphin_neglected_2024}.

\begin{figure}[h]
\centering
\begin{tabular}{cc}
    \includegraphics[width=0.32\linewidth]{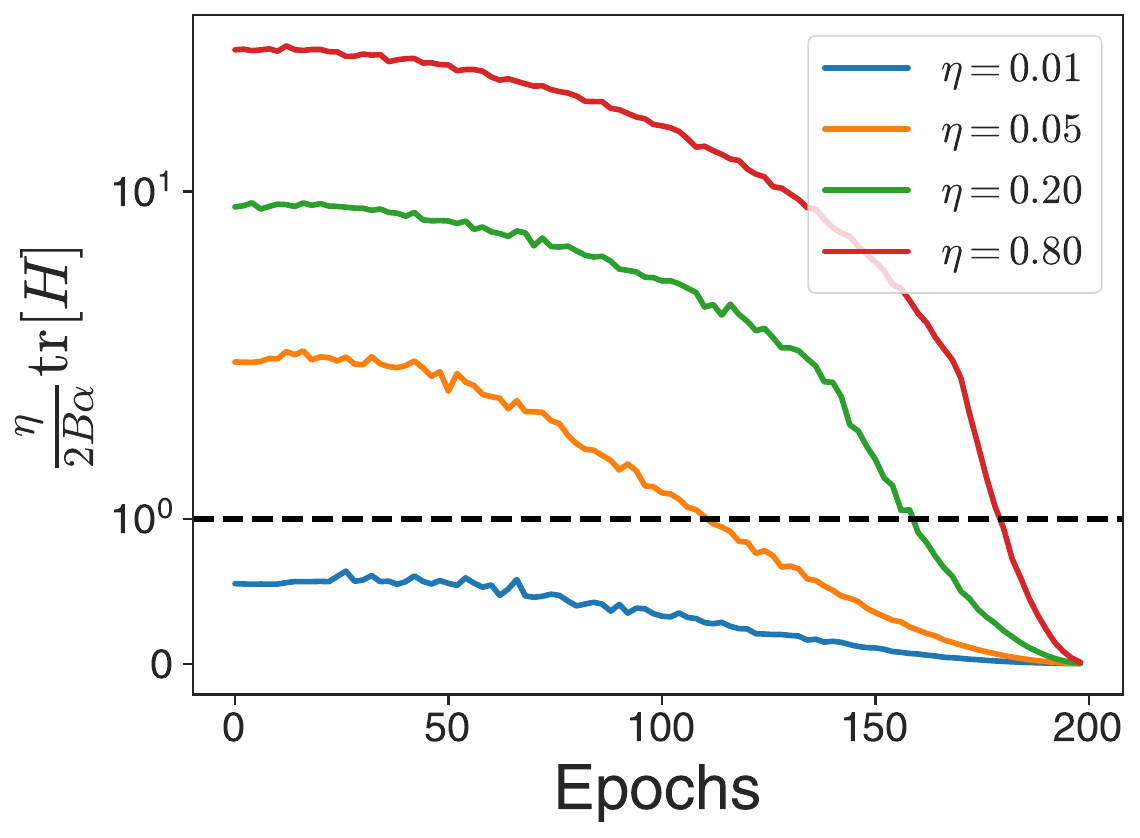} &
    \includegraphics[width=0.4\linewidth]{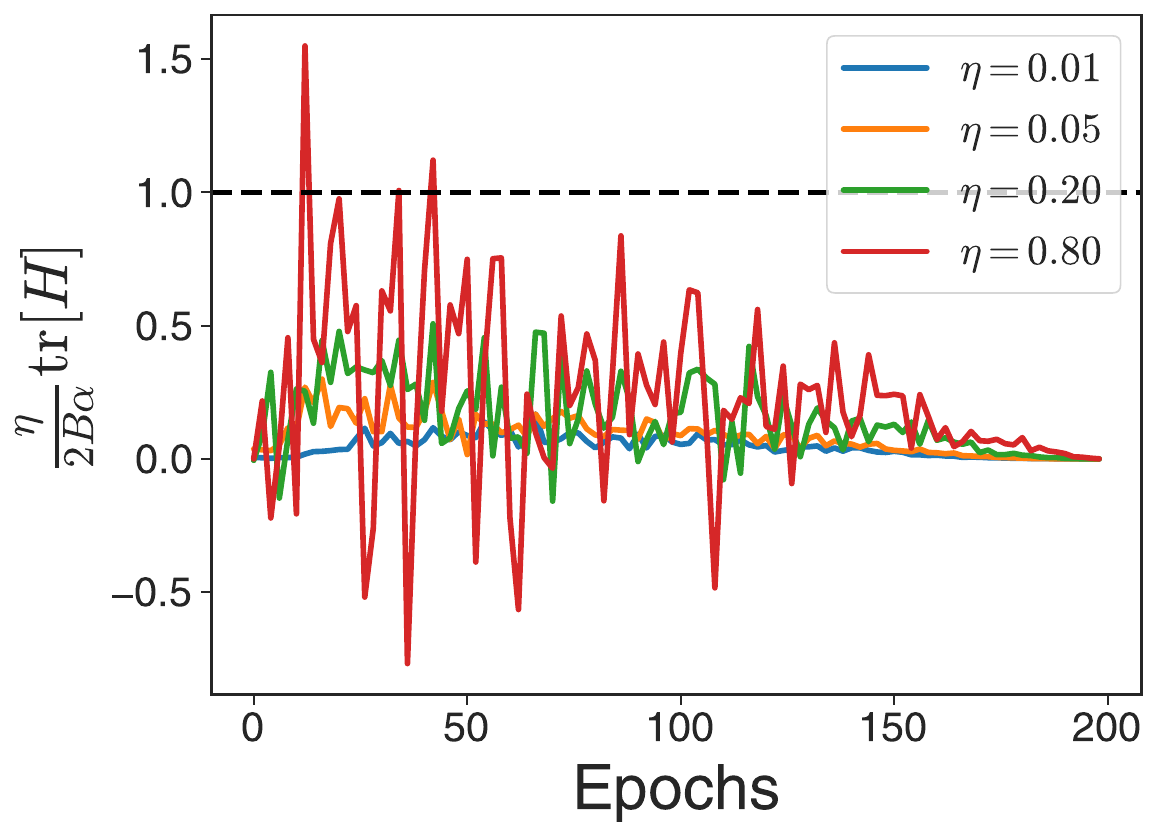}
    
    \end{tabular}
    \caption{Hessian trace for CIFAR model is dominated by $L^{2}$ regularizer (left). Ignoring $L^{2}$ regularization parameter, full Hessian trace is not a
    good approximator of $\knorm$ and does not spend most of its time near $1$ (right).
    }
    \label{fig:hess_tr_no_l2}
\end{figure}

\subsection{MLP-Mixer on CIFAR10}

\label{app:mlp-mixer-cifar10}

To provide additional evidence for the importance of the S-EOS, we also
trained the MLP-Mixer model from \citep{tolstikhin_mlpmixer_2021}, size \texttt{S/16}, on CIFAR10. We trained using SGD with momentum,
MSE loss, batch
size $128$, and a cosine learning rate schedule with $1$ epoch of linear warmup. The base learning rate
was varied by factors of $2$ from $0.00625$ to $1.6$.
We find similar trends to ResNet:

\begin{itemize}
    \item \textbf{$\knorm$ stays in range $[0.3, 1.0]$ over a wide range of learning rates.} We vary learning rates by a factor
    of $128$ and the typical $\knorm$ value only varies by a factor of $3$ (Figure \ref{fig:cifar_mixer}, top left). This suggests there is some effect
    stabilizing its growth.
    \item \textbf{$\lam_{\max}$ remains far from the edge of stability.} Even for the largest learning rates, $\lr\lam_{max}\approx 1$, far from the critical value of $2$ (Figure \ref{fig:cifar_mixer}, top right).
    \item \textbf{$\knorm$ close to $1$ impedes training.} Larger learning rates spend more time with $\knorm$ close to $1$,
    which leads to slower improvements in loss and error rate (Figure \ref{fig:cifar_mixer}, bottom row).
\end{itemize}

\begin{figure}[h]
    \centering
    \begin{tabular}{cc}
     \includegraphics[width=0.45\linewidth]{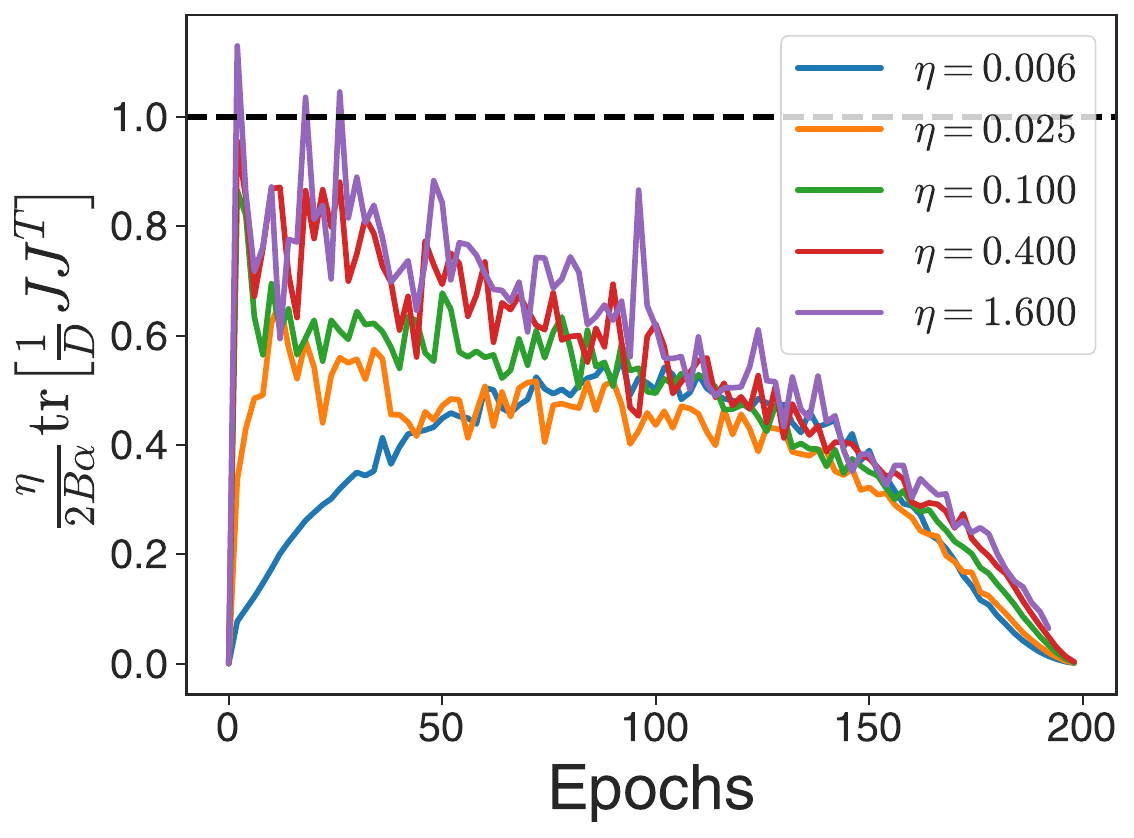} & \includegraphics[width=0.45\linewidth]{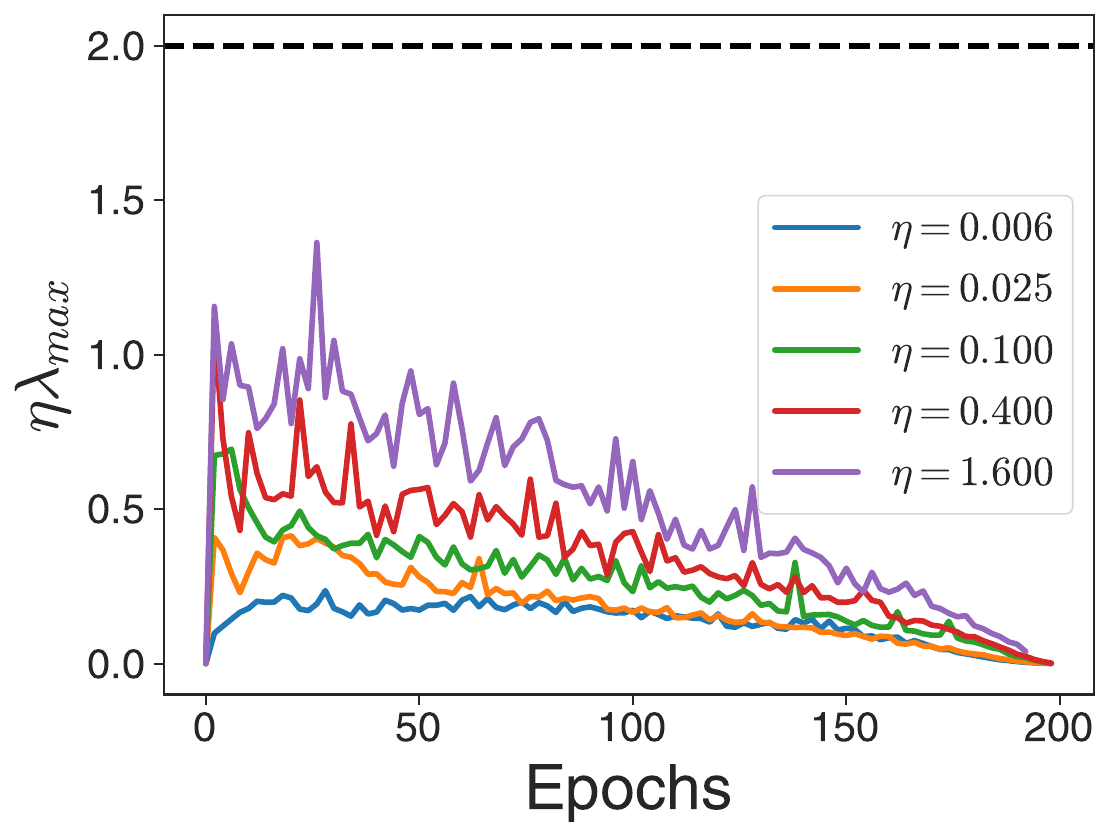}\\
    \includegraphics[width=0.45\linewidth]{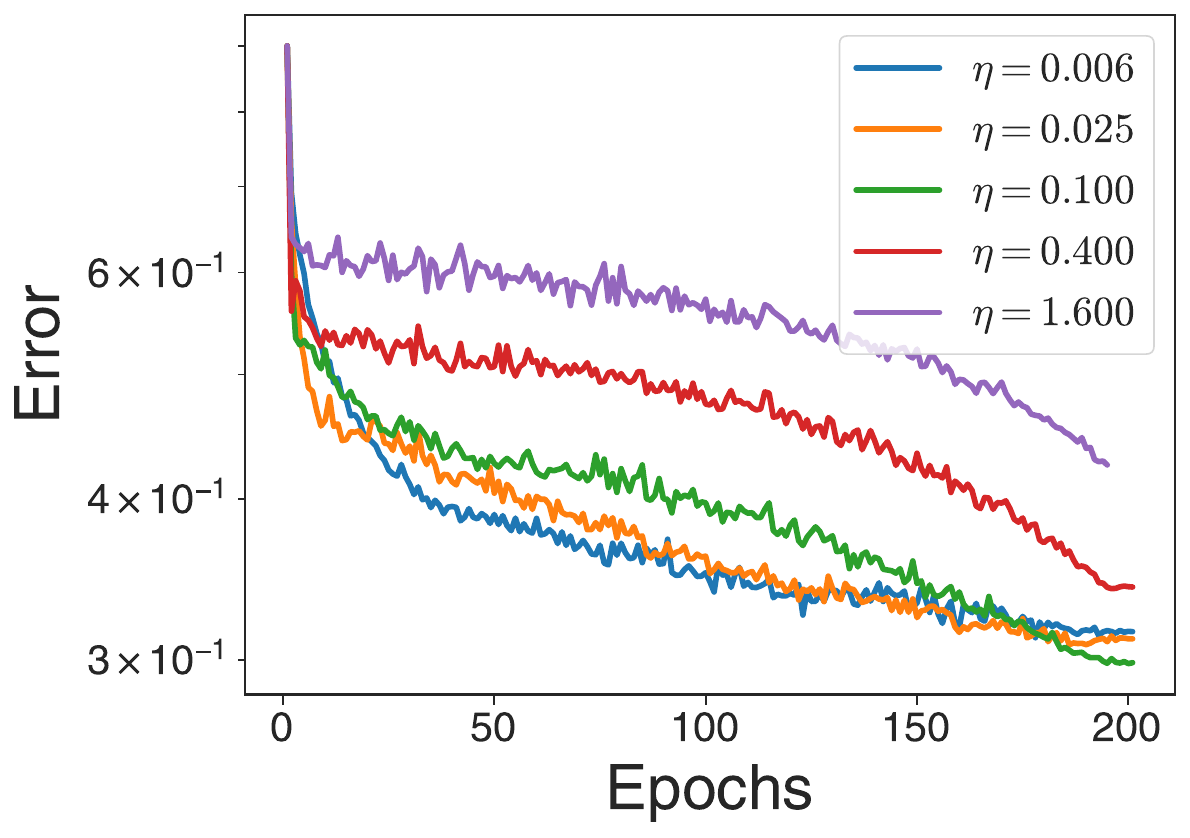} & \includegraphics[width=0.45\linewidth]{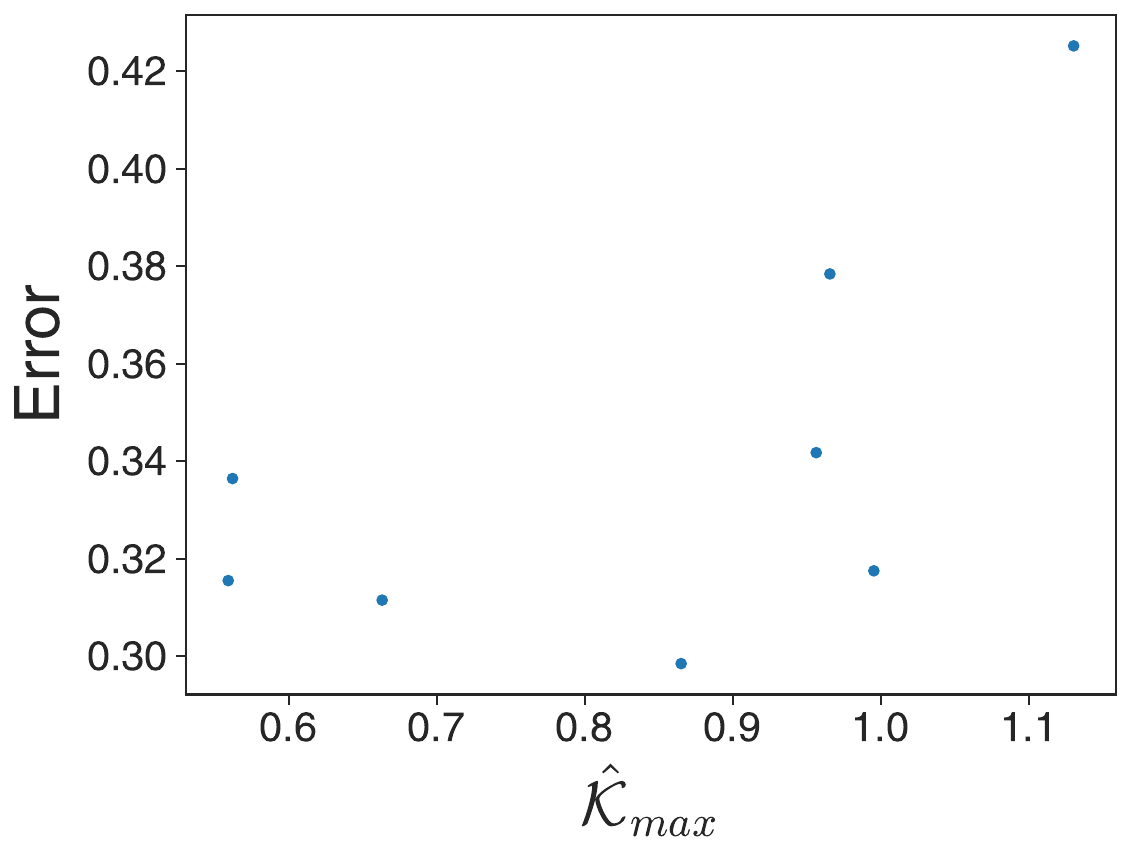}
    \end{tabular}
    \caption{MLP-Mixer trained on CIFAR10. At large learning rates $\hat{\knorm}$ is near
    $1$ at early times, and at intermediate times values cluster over a large range of learning
    rates (top left). Maximum eigenvalue remains below edge of stability (top right). Learning is slow when
    $\hat{\knorm}$ is near $1$ (bottom left), and best performance is for intermediate values of $\hat{\knorm}$
    (bottom right).}
    \label{fig:cifar_mixer}
\end{figure}

\subsection{ResNet50 and ViT on Imagenet - cross entropy loss}

\label{app:imagenet_experiments}

We conducted experiments to test the strengths and limitations of the analysis extending $\knorm$ to non-MSE loss
(Appendix \ref{app:non_mse_loss}).
We trained ResNet50 and ViT on Imagenet. The ViT implementation was the \texttt{S/16} size
from \citet{dosovitskiy2021imageworth16x16words}. Both models were trained using SGD with momentum, batch size $1024$,
on cross-entropy loss. We used a linear warmup for $5$ epochs followed by cosine learning rate schedule for both models.

We used the analysis in Appendix \ref{app:non_mse_loss} to compute an estimator of the noise kernel norm given by:
\begin{equation}
\hat{\knorm}_{mom} \equiv \frac{\lr}{2\al\B}\tr\left[\frac{1}{\D}\H_{GN}\right]
\end{equation}
where the Gauss-Newton component of the Hessian $\H_{GN}\equiv \J^{\tpose}\H_{\z}\J$, where $\H_{\z}$ is the loss Hessian
with respect to the logits. In order to compute the trace of $\H_{GN}$ efficiently over all of Imagenet,
we used the Bartlett Gauss-Newton estimator. This let us estimate
$\hat{\knorm}_{mom}$ with an epoch's worth of backwards passes.
The results are found in Figure \ref{fig:imagenet_vit_resnet}, with ResNet50 in the left column,
and ViT in the right column.

We found qualitative similarities with the experiments studying $\knorm$ in the MSE setting:
\begin{itemize}
    \item \textbf{$\knorm$ remains in a small range over a wide range of learning rates.} Over a range of learning rates of
    factor $100$, $\knorm$ only changed by a factor of $\sim 5$ (Figure \ref{fig:imagenet_vit_resnet}, top row).
    \item \textbf{There is an $O(1)$ threshold of $\knorm$ corresponding to stable training.} The stability threshold was higher than $\knorm = 1$ in both examples. For ResNet50 it appears to be slightly below $2$, for MLP-Mixer slightly above $2$.
    \item \textbf{$\knorm$ is predictive of training success.} In both cases $\knorm<0.5$ and $\knorm>2.0$ lead to either
    inefficient or unstable training respectively (Figure \ref{fig:imagenet_vit_resnet}, middle and bottom rows).
\end{itemize}

\begin{figure}[h]
    \centering
    \begin{tabular}{cc}
     \includegraphics[width=0.45\linewidth]{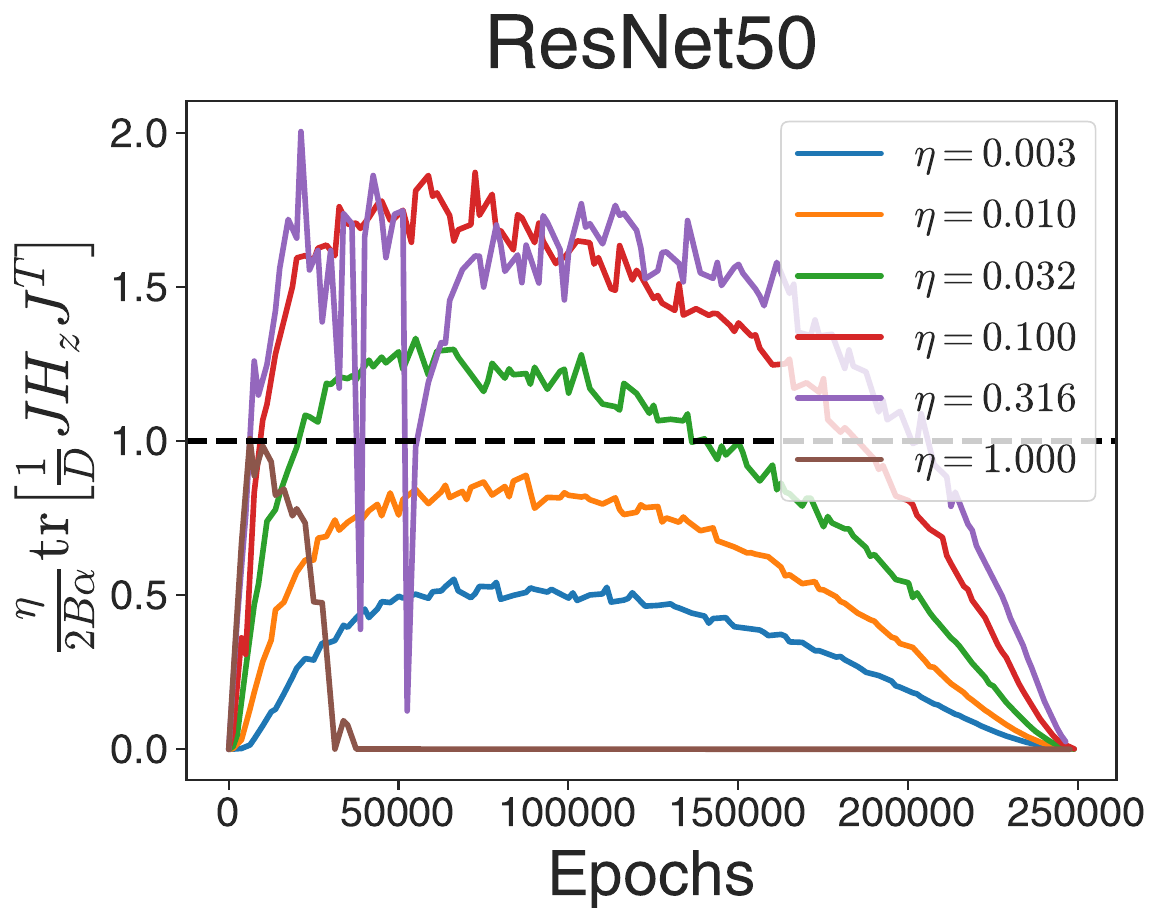} & \includegraphics[width=0.45\linewidth]{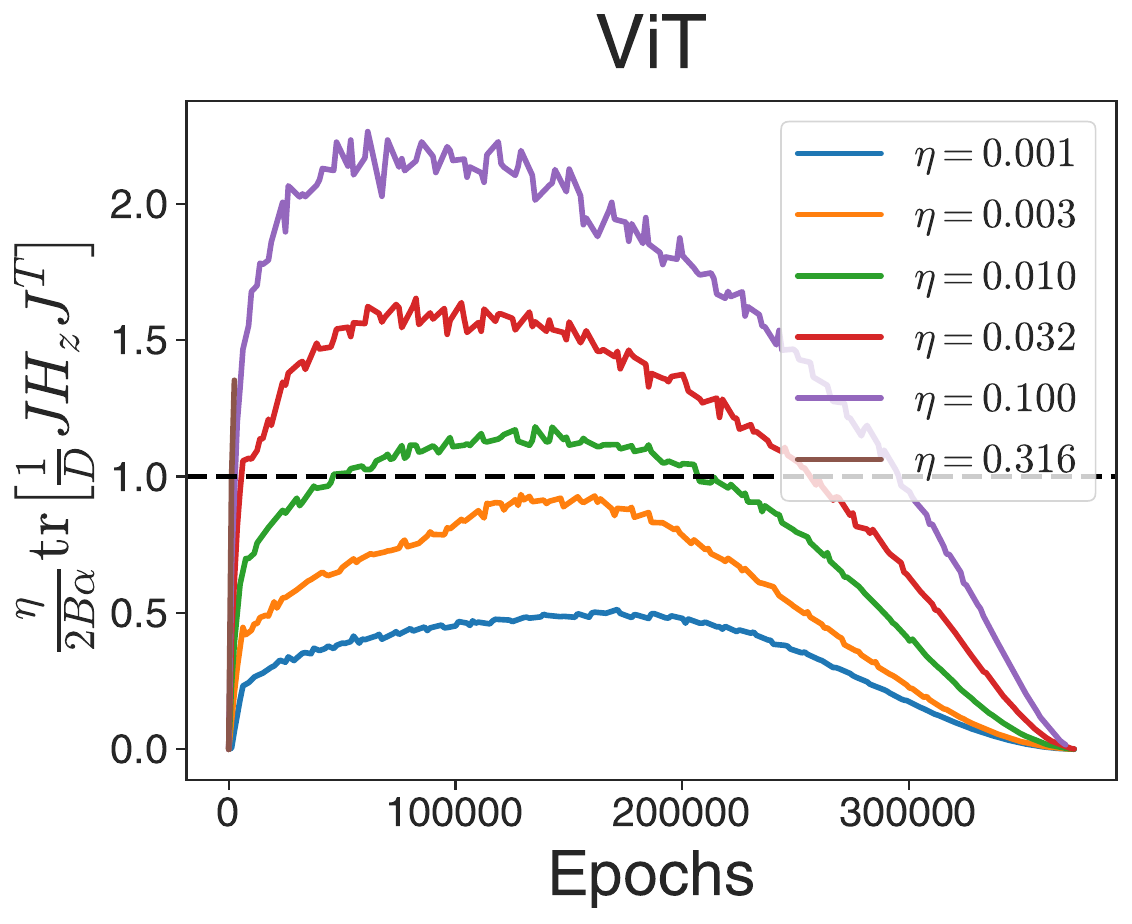}\\
    \includegraphics[width=0.45\linewidth]{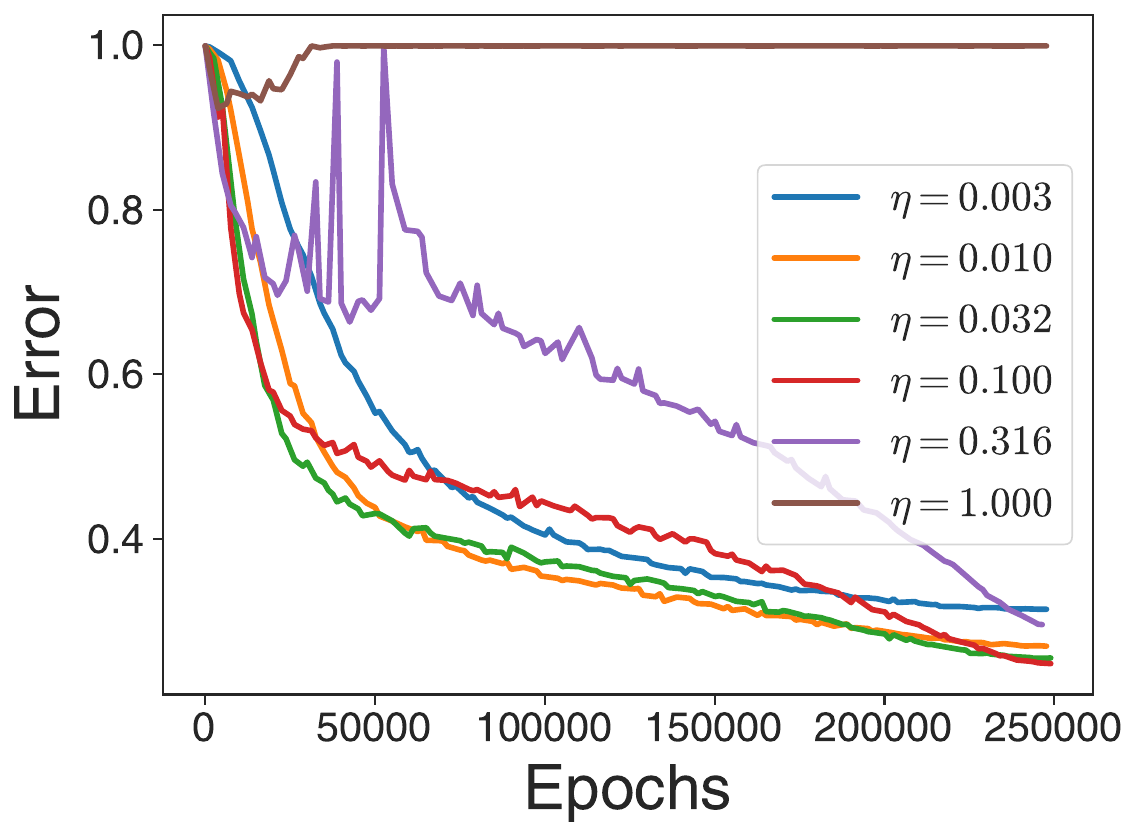} & \includegraphics[width=0.45\linewidth]{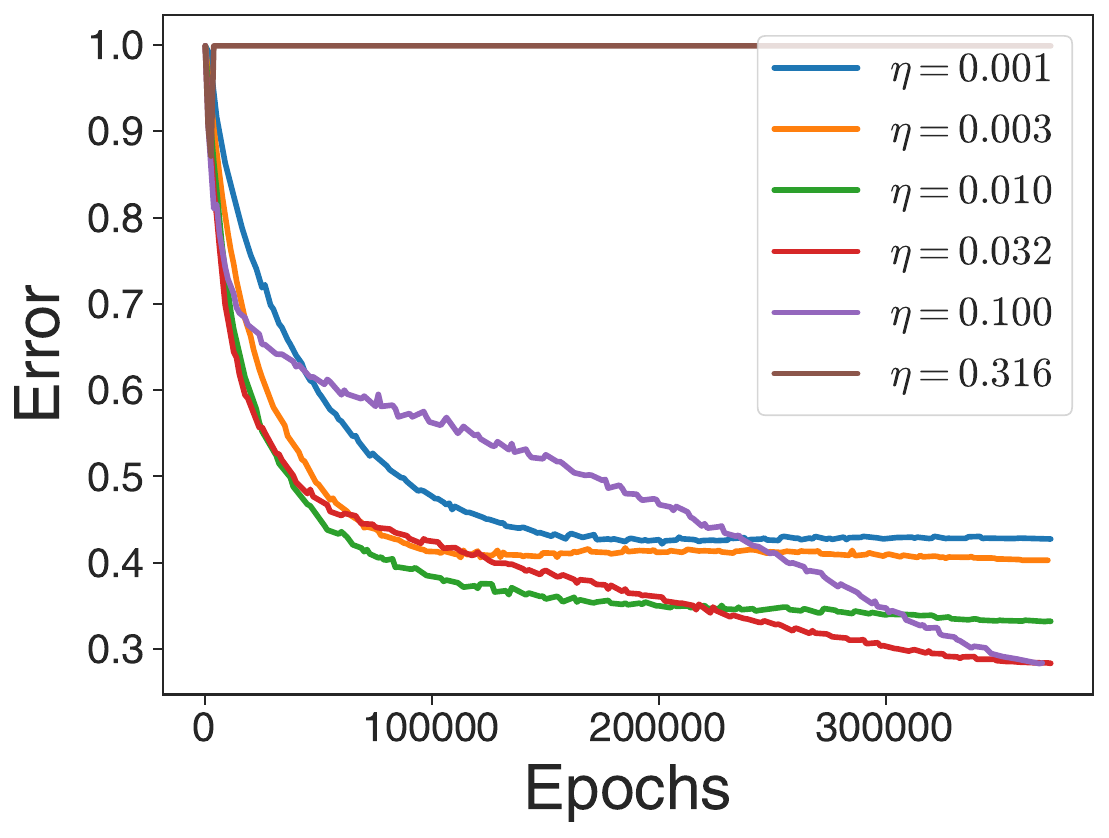}\\
     \includegraphics[width=0.45\linewidth]{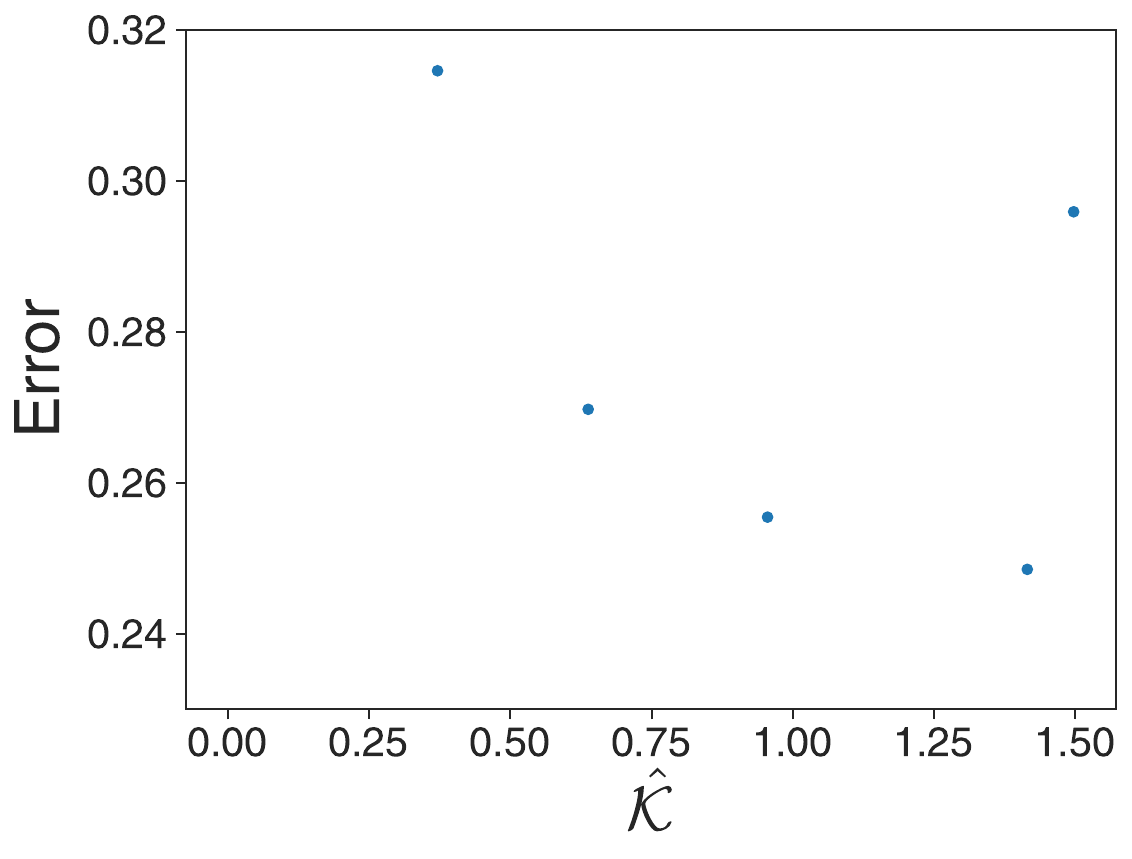} & \includegraphics[width=0.45\linewidth]{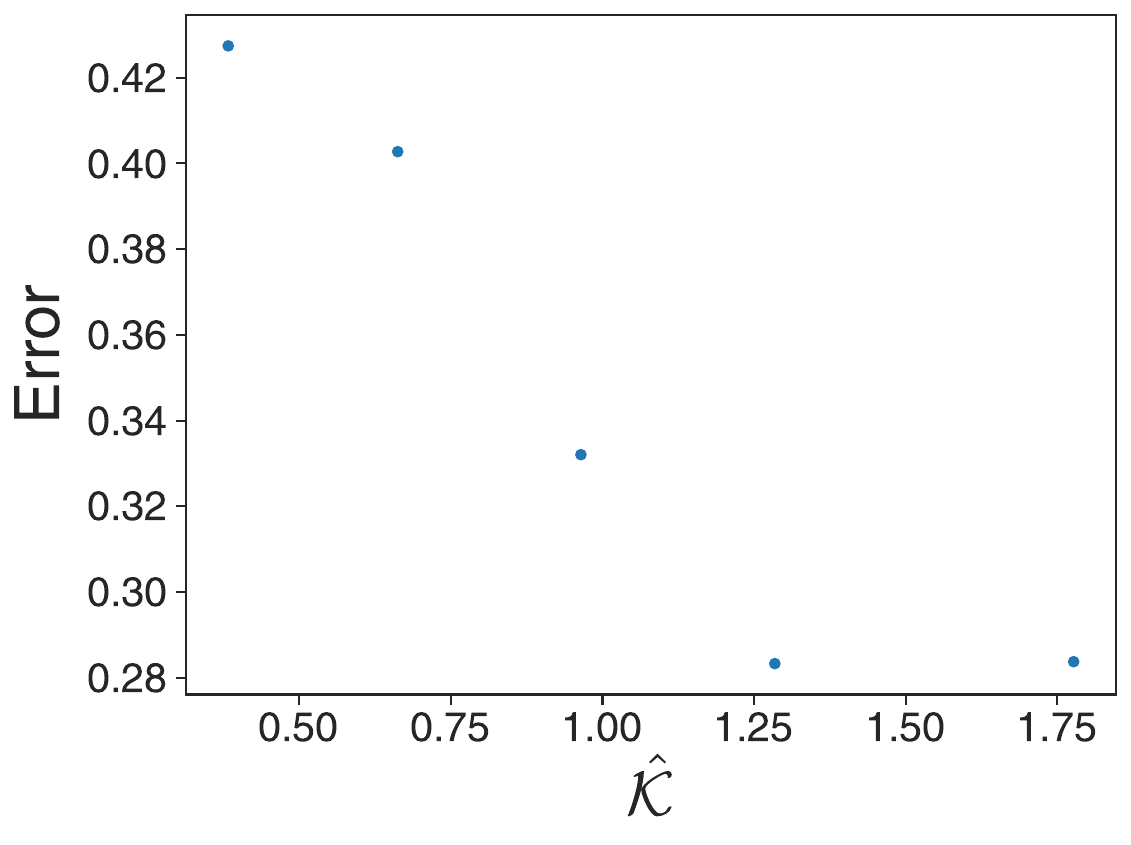}\\
    \end{tabular}
    \caption{ResNet50 (left column) and ViT (right column) trained on Imagenet with cross-entropy loss.
    $\knorm$ was approximated using the Gauss-Newton trace, estimated using the Bartlett-Gauss-Newton
    estimator. Learning rate variation of $1000$ leads to $\hat{\knorm}$ variation of
    a factor of $\sim 5$. $\hat{\knorm}$ seems to have a critical value around $2$ (top and middle row).
    There appears to be an $O(1)$ value of $\hat{\knorm}$ predictive of low error (bottom row),
    but more work is needed to refine the measurement.}
    \label{fig:imagenet_vit_resnet}
\end{figure}

These experiments suggest that extending the analysis of the MSE case to cross-entropy via the Gauss-Newton matrix
is promising, but still requires work. In particular, a better estimator is needed to bring the stability threshhold to the
predictable value $\knorm = 1$. We discuss some of the issues with the approximation in Section \ref{sec:non_mse_caveats}.

\clearpage

\end{document}